\documentclass{article}

\usepackage{microtype}
\usepackage{graphicx}
\usepackage{subfigure}
\usepackage{booktabs} 

\usepackage{hyperref}

\usepackage[accepted]{icml2024}

\usepackage{makecell}
\usepackage{amsmath}
\usepackage{amssymb}
\usepackage{mathtools}
\usepackage{amsthm}
\usepackage{pifont}
\makeatletter

\usepackage[capitalize,noabbrev]{cleveref}

\theoremstyle{plain}
\newtheorem{theorem}{Theorem}[section]

\newtheorem{lemma}[theorem]{Lemma}

\newtheorem{conjecture}[theorem]{Conjecture}
\theoremstyle{definition}
\newtheorem{definition}[theorem]{Definition}

\theoremstyle{remark}
\newtheorem{remark}[theorem]{Remark}

\newcommand{\R}{\mathbb{R}}
\newcommand{\abs}[1]{\left\lvert#1\right\rvert }
\newcommand{\norm}[1]{\left\lVert#1\right\rVert}
\newcommand{\parenth}[1]{\left(#1\right)}

\usepackage[textsize=tiny]{todonotes}

\icmltitlerunning{On the Emergence of Cross-Task Linearity in the Pretraining-Finetuning Paradigm}
\begin{document}

\twocolumn[
\icmltitle{On the Emergence of Cross-Task Linearity in Pretraining-Finetuning Paradigm}


\icmlsetsymbol{equal}{*}

\begin{icmlauthorlist}
\icmlauthor{Zhanpeng Zhou}{equal,cs}
\icmlauthor{Zijun Chen}{equal,cs,pjlab}
\icmlauthor{Yilan Chen}{ucsd}
\icmlauthor{Bo Zhang}{pjlab}
\icmlauthor{Junchi Yan}{cs}
\end{icmlauthorlist}

\icmlaffiliation{pjlab}{Shanghai Artificial Intelligence Laboratory}
\icmlaffiliation{ucsd}{Computer Science and Engineering, University of California San Diego}
\icmlaffiliation{cs}{School of Artificial Intelligence \&  Department of Computer Science and Engineering \& MoE Lab of AI, Shanghai Jiao Tong University, Shanghai, China}

\icmlcorrespondingauthor{Bo Zhang}{bo.zhangzx@gmail.com}
\icmlcorrespondingauthor{Junchi Yan}{yanjunchi.sjtu.edu.cn}

\icmlkeywords{Machine Learning, ICML}

\vskip 0.3in
]



\printAffiliationsAndNotice{\icmlEqualContribution} 

\begin{abstract}
The pretraining-finetuning paradigm has become the prevailing trend in modern deep learning.
In this work, we discover an intriguing linear phenomenon in models that are initialized from a common pretrained checkpoint and finetuned on different tasks\footnote{In this work, task refers to dataset used for finetuning unless otherwise stated.}, termed as \emph{Cross-Task Linearity (CTL)}. 
Specifically, we show that if we linearly interpolate the weights of two finetuned models, the features in the weight-interpolated model are often approximately equal to the linear interpolation of features in two finetuned models at each layer.
We provide comprehensive empirical evidence supporting that CTL consistently occurs for finetuned models that start from the same pretrained checkpoint. 
We conjecture that in the pretraining-finetuning paradigm, neural networks approximately function as linear maps, mapping from the parameter space to the feature space.
Based on this viewpoint, our study unveils novel insights into explaining model merging/editing, particularly by translating operations from the parameter space to the feature space.
Furthermore, we delve deeper into the root cause for the emergence of CTL, highlighting the role of pretraining.
We released our source code at \url{https://github.com/zzp1012/Cross-Task-Linearity}.
\end{abstract}

\section{Introduction}
Pretrained models have become the fundamental infrastructure of modern machine learning systems, and finetuning has evolved as a predominant way for adapting the pretrained model to various downstream tasks. 
Despite the prominent success, our understanding of the pretraining-finetuning paradigm still lags behind. 
There is a growing interest in unraveling the hidden mechanisms of pretraining and finetuning, particularly in human preference alignment~\cite{ouyang2022training}, interpretability~\cite{elhage2021mathematical,olsson2022context}, and AI ethics~\cite{weidinger2021ethical} etc.

\begin{figure}[tb!]
  \begin{center}
    \includegraphics[width=0.48\textwidth]{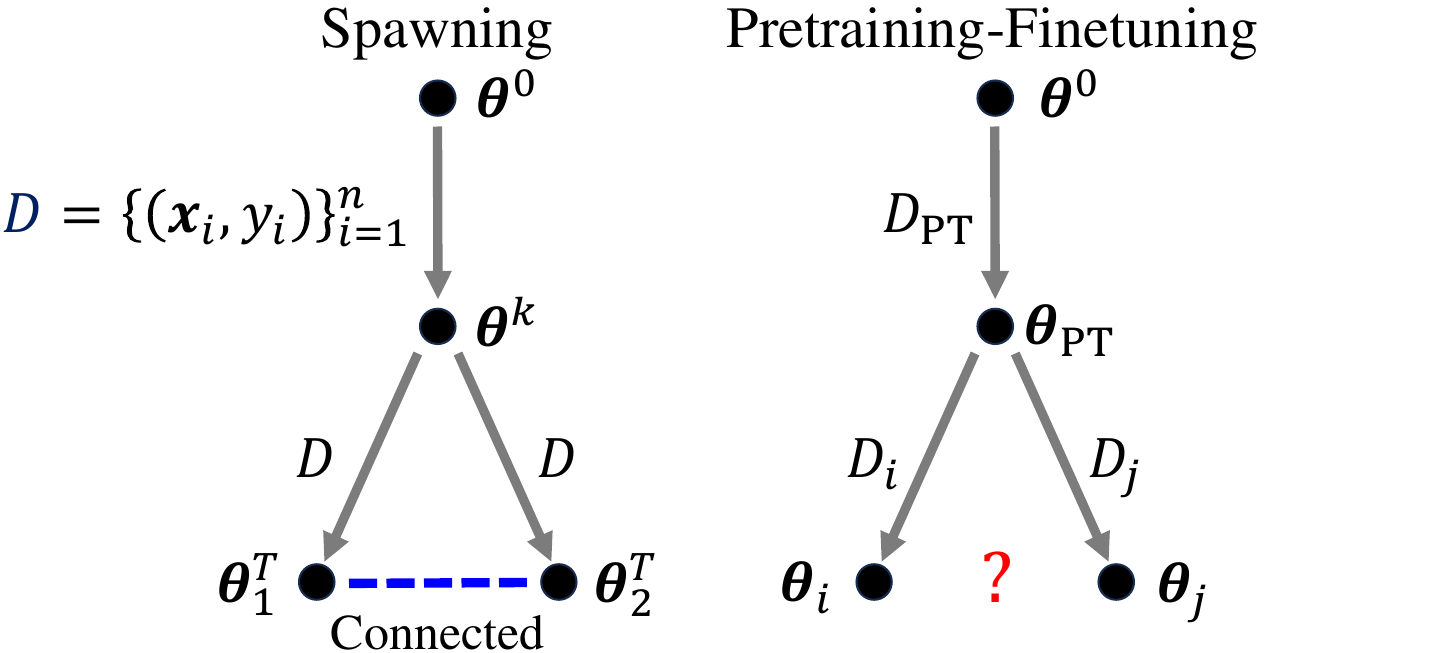}
    \vspace{-15pt}
    \caption{The spawning method and the pretraining-finetuning paradigm. 
    $\boldsymbol{\theta}^{0}$ denotes random initialization of the network weights. 
    For spawning, the network is first trained for $k$ epochs to get $\boldsymbol{\theta}^{k}$, then spawned into two copies and updated until convergence to get $\boldsymbol{\theta}_1^{T}, \boldsymbol{\theta}_2^{T}$. 
    Note $\boldsymbol{\theta}_1^{T}, \boldsymbol{\theta}_2^{T}$ are trained on same task but with different SGD noise.
    With a proper chosen $k$
    , $\boldsymbol{\theta}_1^{T}$ and $\boldsymbol{\theta}_2^{T}$ can satisfy LMC and LLFC. 
    For pretraining-finetuning, the network is first trained on pretraining task $\mathcal{D}_{\rm PT}$ to get  $\boldsymbol{\theta}_{\rm PT}$. 
    Then $\boldsymbol{\theta}_{\rm PT}$ is finetuned on $\mathcal{D}_i$ and $\mathcal{D}_j$ to get $\boldsymbol{\theta}_i$ and $\boldsymbol{\theta}_j$. 
    $\mathcal{D}_i$ and $\mathcal{D}_j$ can be different.}
    \label{fig:spawning_vs_finetuning}
  \end{center}
  \vspace{-20pt}
\end{figure}

Recent works on \emph{Linear Mode Connectivity (LMC)}~\cite{nagarajan2019uniform,frankle2020linear} and \emph{Layerwise Linear Feature Connectivity (LLFC)}~\cite{zhou2023going} shed light on understanding the training dynamics and hidden mechanisms of neural networks.
LMC depicts a linear path in the parameter space of a network where the loss remains approximately constant (see \cref{def:LMC}).
In other words, linearly interpolating the weights of two different models, which are of the same architecture and trained on the same task, could lead to a new model that achieves similar performance as the two original models.
LLFC indicates that the features in the weight-interpolated model are proportional to the linear interpolation of the features in the two original models (see \cref{def:LLFC}).
\citet{frankle2020linear} observed LMC for networks that are jointly trained for a short time before undergoing independent training on the same task, termed as \emph{spawning method} (see \cref{fig:spawning_vs_finetuning}).
\citet{zhou2023going} discovered the models that linearly connected in the loss landscape are also linearly connected in feature space, i.e., satisfy LMC and LLFC simultaneously.

\begin{table}[tb!]
    \centering
\resizebox{0.48\textwidth}{!}{
    \begin{tabular}{l|c|c|c}
      & \makecell[c]{Performance} & \makecell[c]{Equality} & \makecell[c]{Task}  \\
    \hline
    LMC~\cite{frankle2020linear} & Approx. equal & Approx. equal & Single \\
    \hline
    LLFC~\cite{zhou2023going} & Flexible & Proportional to & Single \\
    \hline
    CTL (Ours) & Flexible & Approx. equal & Multiple\\
    \end{tabular}}
    \vspace{-5pt}
    \caption{
    Comparison: LMC requires models with approximately equal performance; LLFC depicts a proportional relation only.
    Both LMC and LLFC focus on models trained on the same task.
    CTL extends LLFC to models finetuned on different tasks.}
    \label{tab:comp}
    \vspace{-10pt}
\end{table}

As shown in \cref{fig:spawning_vs_finetuning}, a connection is identified between the pretraining-finetuning paradigm and the spawning method, as both entail training models from a same pretrained checkpoint. 
Therefore, a natural question arises: are models, initialized from a common pretrained checkpoint\footnote{In this work, we consider finetuned models that start from a common pretrained checkpoint.} but finetuned on \emph{different tasks}, linearly connected in the loss landscape or feature space, akin to the models obtained by the spawning method satisfying LMC and LLFC?

In this work, we discover that the finetuned models are linearly connected in internal features even though there is no such connectivity in the loss landscape, i.e., LLFC holds but LMC not.
Indeed, we identify a stronger notion of linearity than LLFC: if we linearly interpolate the weights of two models that finetuned on different tasks, the features in the weight-interpolated model are \emph{approximately equal} to the linear interpolation of features in the two finetuned models at each layer, namely \emph{Cross-Task Linearity (CTL)} as termed in this paper (see comparison among LMC, LLFC, and CTL in \cref{tab:comp}).
To be precise, let $\boldsymbol{\theta}_i$ and $\boldsymbol{\theta}_j$ be the weights of two finetuned models, and $f^{(\ell)}(\boldsymbol{\theta})$ be the features in the model of weights $\boldsymbol{\theta}$ at $\ell$-th layer.
We say that $\boldsymbol{\theta}_i$ and $\boldsymbol{\theta}_j$ satisfy CTL if $\forall \ell, \forall \alpha \in [0, 1]$,
\begin{align*}
    f^{(\ell)}( \alpha {\boldsymbol{\theta}}_i + (1-\alpha) {\boldsymbol{\theta}}_j ) \approx \alpha f^{(\ell)}({\boldsymbol{\theta}}_i) + (1-\alpha) f^{(\ell)}({\boldsymbol{\theta}}_j).
\end{align*}
CTL may not be universal for arbitrary networks trained on tasks, yet we provide comprehensive empirical evidence supporting that CTL consistently occurs for the finetuned models across a wide range of settings. 
We conjecture that in the pretraining-finetuning paradigm, neural networks can roughly function as linear maps, mapping from the parameter space to the feature space.

Based on the observed CTL in the pretraining-finetuning paradigm, we obtain novel insights into two widely-used model merging/editing techniques: \emph{model averaging}~\cite{DBLP:conf/uai/IzmailovPGVW18,matena2022merging,pmlr-v202-rame23a,rame2022diverse,wortsman2022model,Wortsman_2022_CVPR} and \emph{task arithmetic}~\cite{ilharco2022patching,ilharco2023editing,ortiz-jimenez2023task}.

\textbf{i)} Model averaging takes the average of weights of multiple models, which are finetuned on the same task but with different hyperparameter configurations, so as to improve accuracy and robustness. 
We explain the averaging of weights as the averaging of features at each layer, building a stronger connection between model averaging and logits ensemble than before.

\textbf{ii)} Task arithmetic merges the weights of models, that are finetuned on different tasks, via simple arithmetic operations, shaping the behaviour of the resulting model accordingly. 
We translate the arithmetic operation in the parameter space into the operations in the feature space, yielding a feature-learning explanation for task arithmetic.

Furthermore, we delve deeper into the root cause of CTL. 
We empirically investigate various factors contributing to the holding of CTL, highlighting the role of pretraining. 
We also take a primary attempt to prove CTL and find that the emergence of CTL is associated with the flatness of the network landscape and the distance between the weights of two finetuned models.

In summary, our work reveals a linear connection between finetuned models, offering significant insights into model merging/editing techniques. 
This, in turn, advances our understanding of underlying mechanisms of pretraining and finetuning from a feature-centric perspective.

\section{Related Work}
\label{sec:related_work}
\textbf{(Linear) Mode Connectivity.}
\citet{freeman2017topology,draxler2018essentially,garipov2018loss} noted Mode Connectivity (MC), where different minima in the loss landscape can be connected by a non-linear path of nearly constant loss.
\citet{nagarajan2019uniform,frankle2020linear} discovered that the path of nearly constant loss can be linear, for models that are jointly trained for a short time before undergoing independent training, termed Linear Mode Connectivity (LMC).
\citet{fort2020deep} analyzed LMC from a perspective of the Neural Tangent Kernel dynamics.
\citet{entezari2022the,ainsworth2023git} showed that even independently trained networks can satisfy LMC after accounting for permutation invariance. 
Studies~\cite{liang2018understanding,venturi2018spurious,nguyen2018loss,nguyen2019connected,kuditipudi2019explaining,ferbach2023proving,zhao2023understanding,zhou2023going} have attempted to prove (linear) mode connectivity from various perspectives. 
\citet{adilova2023layerwise} studied the layerwise behaviour of LMC under federated learning settings.
\citet{qin-etal-2022-exploring} studied MC in the context of pretrained language models.
\citet{mirzadeh2021linear,juneja2023linear} investigated finetuning from the lens of LMC.
\citet{zhou2023going} identified a stronger connectivity than LMC, namely Layerwise Linear Feature Connectivity (LLFC), and observed LLFC always co-occurs with LMC.
\cite{chen2024going} expand the concept of feature similarity with LLFC.

\textbf{Model Merging/Editing.}
Recent studies find averaging the parameters of finetuned models over the same task leads to improved performance and generalization abilities~\cite{DBLP:conf/uai/IzmailovPGVW18,matena2022merging,pmlr-v202-rame23a,rame2022diverse,wortsman2022model,Wortsman_2022_CVPR}.
Moreover, the averaging of weights from models finetuned over tasks enables multi-task abilities~\cite{ilharco2022patching,li2022branchtrainmerge,yadav2023tiesmerging,jin2023dataless,stoica2023zipit,yu2023language}. 
\citet{singh2020model,pmlr-v162-liu22k} show that the weights of independently trained neural networks can be merged after aligning the neurons.
Moreover, \citet{ilharco2023editing,ortiz-jimenez2023task} extend the simple averaging to arithmetic operations in the parameter space, enabling a finer-grained control of the model behaviours.

\section{Backgrounds and Preliminaries}
\label{sec:background_and_preliminaries}

\textbf{Notation Setup.}
Unless explicitly stated otherwise, we consider a classification dataset/task, denoted as $\mathcal{D} = \{(\boldsymbol{x}_i, y_i)\}_{i=1}^n$ where $\boldsymbol{x}_i \in \mathbb{R}^{d_{0}}$ is the input and $y_i \in [c]$ is the label of the $i$-th datapoint.
Here, $d_{0}$ is the input dimension, $[c] = \{1, 2, \ldots, c\}$ and $c$ is the number of classes. 
We use ${\boldsymbol{X}} \in \mathbb R^{d_{0} \times n}$ to stack all the input data into a matrix.

We consider an $L$-layer neural network defined as $f(\boldsymbol{\theta}; \boldsymbol{x})$, where $\boldsymbol{\theta}$ denotes the model parameters, $\boldsymbol{x}$ is the input, and $f(\boldsymbol{\theta}; \boldsymbol{x})\in\R^c$. 
$f^{(\ell)}(\boldsymbol{\theta}; \boldsymbol{x}) \in \R^{d_\ell}$ represents the internal feature (post-activation) in the network at the $\ell$-th layer.
Here, $d_\ell$ denotes the dimension of the $\ell$-th layer ($0\le \ell \le L$) and $f^{(L)}(\boldsymbol{\theta}; \boldsymbol{x}) = f(\boldsymbol{\theta}; \boldsymbol{x})$.
For an input matrix $\boldsymbol{X}$, we use $f^{(\ell)}(\boldsymbol{\theta}; \boldsymbol{X}) \in \R^{d_{\ell}\times n}$ to denote the collection of features on all the datapoints.
When $\boldsymbol{X}$ is clear from the context, we simply write $f^{(\ell)}(\boldsymbol{\theta}) = f^{(\ell)}(\boldsymbol{\theta}; \boldsymbol{X})$. 
The expected loss on dataset $\mathcal{D}$ is denoted by $\mathcal{L}(\boldsymbol{\theta}) = \mathbb{E}_{(\boldsymbol{x}, y)\in \mathcal{D}}\left[L\left(f\left(\boldsymbol{\theta};\boldsymbol{x}\right), y\right)\right]$, where $L$ represents the loss function.
Our analysis focuses on models trained on a training set, with all investigations evaluated on a separate test set.

\textbf{Linear Mode Connectivity (LMC).} 
\begin{definition}[\textbf{Linear Mode Connectivity}~\cite{nagarajan2019uniform,frankle2020linear}]\label{def:LMC}
Given dataset $\mathcal{D}$ and two modes\footnote{A \emph{mode} refers to the obtained solution after training.} ${\boldsymbol{\theta}}_i$ and ${\boldsymbol{\theta}}_j$ such that $\mathcal{L}({\boldsymbol{\theta}}_i) \approx \mathcal{L}({\boldsymbol{\theta}}_j)$ on $\mathcal{D}$, we say ${\boldsymbol{\theta}}_i$ and ${\boldsymbol{\theta}}_j$ are linearly connected in the loss landscape if they satisfy $\forall \alpha\in[0,1]$, \begin{align*}
     \mathcal{L}(\alpha {\boldsymbol{\theta}}_i + (1 - \alpha) {\boldsymbol{\theta}}_j) \approx \mathcal{L}({\boldsymbol{\theta}}_i) \approx \mathcal{L}({\boldsymbol{\theta}}_j).
\end{align*}
\end{definition}
As \cref{def:LMC} shows, LMC indicates different optima can be connected via a simple linear path of nearly constant loss. 
Previous studies~\cite{frankle2020linear} observed LMC for networks that start from a common pretrained checkpoint and undergo independent training on the same task until convergence, commonly referred as \emph{spawning method}~\cite{fort2020deep,zhou2023going}.

\textbf{Layerwise Linear Feature Connectivity (LLFC).}
\begin{definition}[\textbf{Layerwise Linear Feature Connectivity}~\cite{zhou2023going}]\label{def:LLFC}
    Given dataset $\mathcal{D}$ and two modes ${\boldsymbol{\theta}}_i$, ${\boldsymbol{\theta}}_j$ of an $L$-layer neural network $f$, the modes ${\boldsymbol{\theta}}_i$ and ${\boldsymbol{\theta}}_j$ are said to be linearly connected in feature space on $\mathcal{D}$ if $\forall \ell \in [L], \forall  \alpha \in [0,1]$ such that, \begin{align*}
        f^{(\ell)}( \alpha {\boldsymbol{\theta}}_i + (1-\alpha) {\boldsymbol{\theta}}_j) \propto \alpha f^{(\ell)}({\boldsymbol{\theta}}_i) + (1-\alpha) f^{(\ell)}({\boldsymbol{\theta}}_j).
    \end{align*}
\end{definition}
In \cref{def:LLFC}, LLFC states that the features (post-activation) in the interpolated model $\boldsymbol{\theta}_{\alpha} = \alpha \boldsymbol{\theta}_i + (1-\alpha) \boldsymbol{\theta}_j$ are proportional to the linear interpolation of the features in $\boldsymbol{\theta}_i$ and $\boldsymbol{\theta}_j$ at each layer.

\citet{zhou2023going} introduced LLFC, which defines a stronger notion of linear connectivity than LMC, and noted its consistent co-occurrence with LMC. 
Specifically, if two modes $\boldsymbol{\theta}_i$ and $\boldsymbol{\theta}_j$ satisfy LMC, then they also approximately satisfy LLFC. 
Moreover, it can be proven that LLFC directly induces LMC for models with equal loss (see \cref{thm:LLFC_induces_LMC}).
Therefore, they believed that LLFC is a more fundamental property than LMC.

\begin{theorem}[\textbf{LLFC Induces LMC} (Proof in \cref{suppl:proof_of_thm1})]\label{thm:LLFC_induces_LMC}
    Given dataset $\mathcal{D}$, convex loss function $L$, and two modes ${\boldsymbol{\theta}}_i$ and ${\boldsymbol{\theta}}_j$ with equal loss on $\mathcal{D}$, i.e., $\mathcal{L}({\boldsymbol{\theta}}_i) = \mathcal{L}({\boldsymbol{\theta}}_j)$, suppose the two modes ${\boldsymbol{\theta}}_i$, ${\boldsymbol{\theta}}_j$ satisfy LLFC on $\mathcal D$ with exact equality, then for all $\alpha \in [0, 1]$, \begin{align*}
        \mathcal{L}(\alpha {\boldsymbol{\theta}}_i + (1-\alpha) {\boldsymbol{\theta}}_j) \leq \mathcal{L}(\boldsymbol{\theta}_i) = \mathcal{L}(\boldsymbol{\theta}_j).
    \end{align*}
\end{theorem}

\textbf{Main Experimental Setup.} 
In \cref{sec:LLFC_to_linearity}, we conduct experiments on standard continue learning benchmark datasets, including Rotated MNIST~\cite{lecun1998gradient} and Split CIFAR-100~\cite{krizhevsky2009learning}, with MLP and ResNet-18~\cite{kaiming2016residual}.
We follow the same training procedures and hyper-parameters as in \citet{mirzadeh2021linear}.
In \cref{sec:model_averaging,sec:task_arithmetic}, we directly adopt the finetuned ViTs~\cite{dosovitskiy2020image}/T5s~\cite{raffel2020exploring} checkpoints open-sourced by \citet{wortsman2022model,ilharco2023editing} and perform experiments on various image and text datasets.
Due to space limit, we defer more experimental settings to \cref{suppl:settings}.

\section{Cross-Task Linearity}
\label{sec:linearity}
\begin{figure*}[tb!]
  \begin{center}
    \includegraphics[width=0.9108\textwidth]{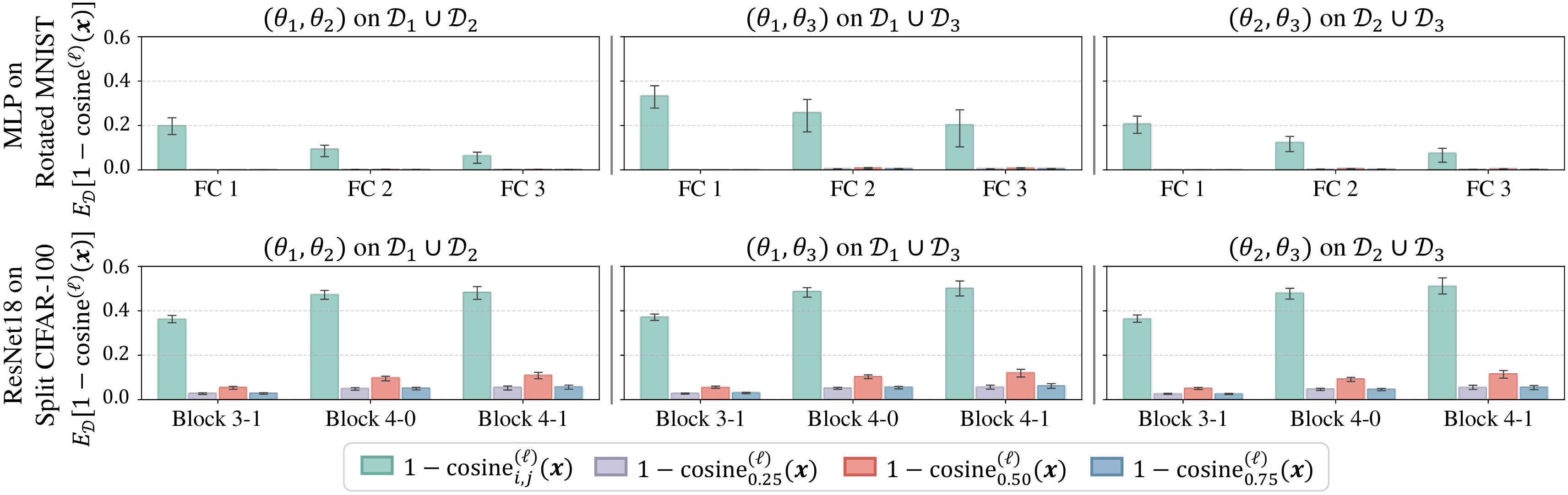}
    \vspace{-5pt}
    \caption{
        Verification of CTL.
        Compare $\mathbb{E}_{\mathcal{D}}[1-{\rm cosine}_{\alpha}^{(\ell)}(\boldsymbol{x})]$ with $\mathbb{E}_{\mathcal{D}}[1-\text{cosine}_{i,j}^{(\ell)}(\boldsymbol{x})]$.
        Here, $\{\boldsymbol{\theta}_i\}_{i=1}^3$ and $\{\mathcal{D}_i\}_{i=1}^3$ denotes finetuned models and their corresponding downstream tasks. 
        For Rotated MNIST, models are pretrained on MNIST and finetuned on variants of MNIST where digits are at different angles.
        For Split CIFAR-100, models are pretrained and finetuned on disjoint sets of 5 classes from CIFAR-100.
        The bottom and top of the error bar represent the lower and upper quartile of the values across the dataset, respectively.
        The results are reported for last three layers/blocks, with $\alpha \in \{0.25, 0.5, 0.75\}$.
        More results in \cref{suppl:exp_CTL}.
    }
    \label{fig:LLFC_main_cosine}
  \end{center}
\end{figure*}

\begin{figure}[tb!]
  \begin{center}
    \includegraphics[width=0.4613\textwidth]{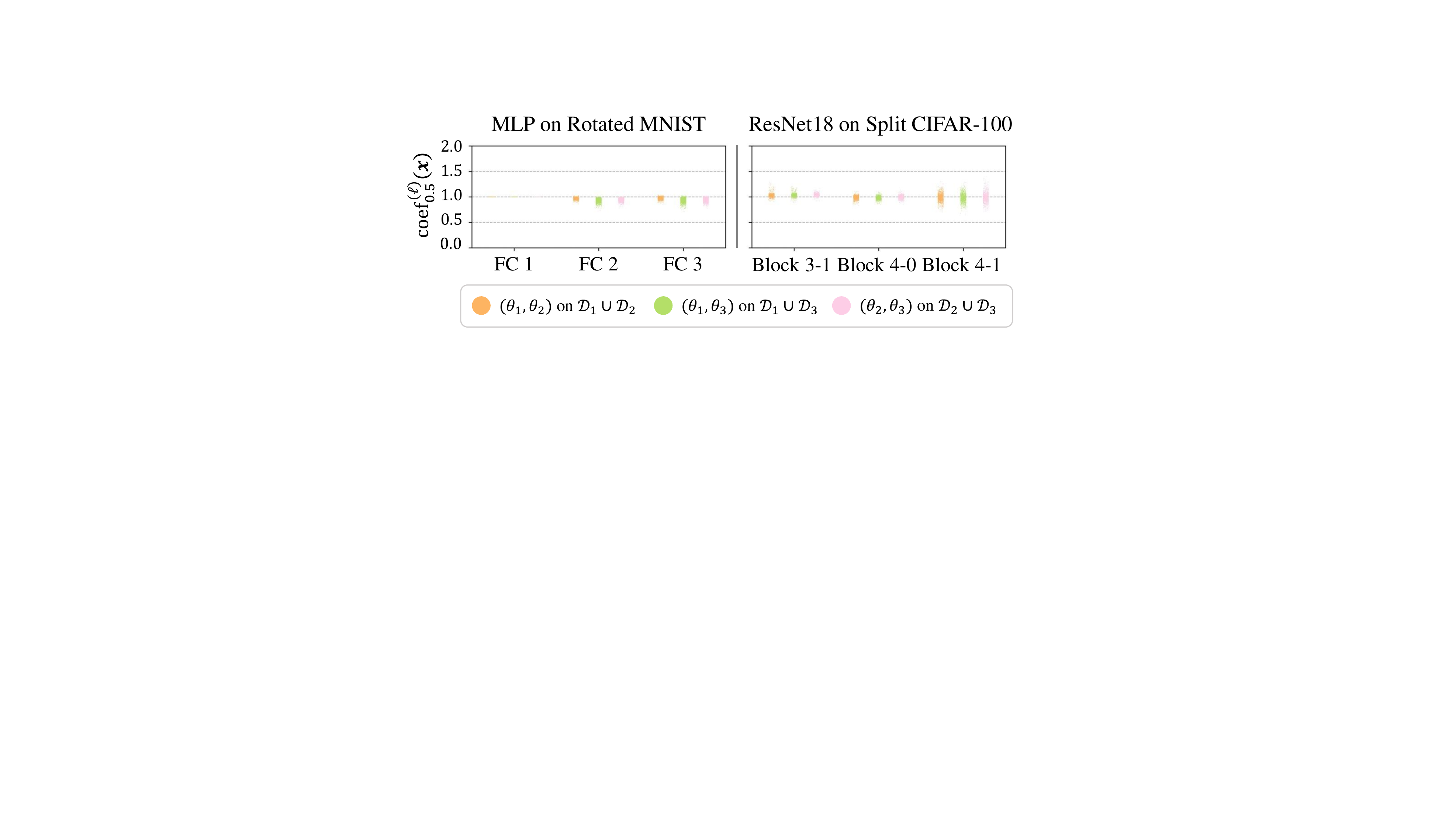}
    \vspace{-5pt}
    \caption{
        Verification of CTL.
        Distribution of $\text{coef}_{\alpha}^{(\ell)}(\boldsymbol{x})$ across the dataset. 
        Here, $\alpha = 0.5$. 
        $\{\boldsymbol{\theta}_i\}_{i=1}^3$ and $\{\mathcal{D}_i\}_{i=1}^3$ denotes finetuned models and their corresponding downstream tasks.
        We follow the same training settings as in \cref{fig:LLFC_main_cosine}.
        The results are reported for last three layers/blocks.
        More results in \cref{suppl:exp_CTL}.
    }
    \label{fig:LLFC_main_coef}
    \vspace{-15pt}
  \end{center}
\end{figure}
In this section, we provide empirical evidence indicating that the finetuned models are linearly connected in the feature space, even though there is no such connectivity in the loss landscape, i.e. LLFC holds but LMC not. 
Taking a step further, we identify a stronger notion of linearity, namely Cross-Task Linearity (CTL), which approximately characterizes neural networks as linear maps in the pretraining-finetuning paradigm.
From these observations, we offer novel insights into model averaging and task arithmetic.

\subsection{Extend LMC and LLFC to CTL}\label{sec:LLFC_to_linearity}
\textbf{LMC Fails in the Pretraining-Finetuning Paradigm.}
LMC might not hold for the finetuned models, as the pre-conditions of \cref{def:LMC} are not met, i.e., the models finetuned on different tasks might not have approximately equal optimal loss on the same task. Indeed, the studies on LMC are motivated by the interest in studying optima of the same loss landscape.
It means LMC depicts the property of models trained on the same task. Therefore, it is clear that finetuned models might not satisfy LMC, even on the pretraining task, where catastrophic forgetting may occur~\cite{mccloskey1989catastrophic}.
Similar phenomena were observed in \citep{mirzadeh2021linear,juneja2023linear}.

\textbf{LLFC Holds in the Pretraining-Finetuning Paradigms.}
Despite no connectivity in the loss landscape, we surprisingly find that the finetuned models are linearly connected in the feature space. Here, we extend the original LLFC to the cases where models are finetuned on different tasks.

To verify LLFC for finetuned models, we conduct extensive experiments across a range of settings.
Specifically, we consider a set of finetuned models\footnote{For simplicity, we often denote models of the same architecture as $\boldsymbol{\theta}$ instead of $f(\boldsymbol{\theta})$.}, $\boldsymbol{\Theta}=\{\boldsymbol{\theta}_i\}_{i=1}^k$, which are initialized from a common pretrained checkpoint $\boldsymbol{\theta}_{\rm PT}$ but finetuned on different tasks. 
Here, the downstream tasks for finetuning are denoted as $\{\mathcal{D}_i\}_{i=1}^k$, respectively.
Then, for each pair of finetuned models $(\boldsymbol{\theta}_i, \boldsymbol{\theta}_j) \in \boldsymbol{\Theta}^2$, on each datapoint $(\boldsymbol{x}, \boldsymbol{y}) \in \mathcal{D}_i \cup \mathcal{D}_j$, we measure the cosine similarity between the features in the weight-interpolated model $\boldsymbol{\theta}_{\alpha} = \alpha \boldsymbol{\theta}_i + (1-\alpha) \boldsymbol{\theta}_j$ and the linear interpolation of the features in $\boldsymbol{\theta}_i$ and $\boldsymbol{\theta}_j$ at each layer $\ell$, denoted as $\text{cosine}_{\alpha}^{(\ell)}(\boldsymbol{x}) = \cos [f^{(\ell)}(\boldsymbol{\theta}_{\alpha}; \boldsymbol{x}), \alpha f^{(\ell)}(\boldsymbol{\theta}_i; \boldsymbol{x}) + (1-\alpha) f^{(\ell)}(\boldsymbol{\theta}_j; \boldsymbol{x})]$. 
We compare $\text{cosine}_{\alpha}^{(\ell)}$ to the baseline cosine similarity between the features in $\boldsymbol{\theta}_i$ and $\boldsymbol{\theta}_j$ in the same layer, i.e, $\text{cosine}_{i,j}^{(\ell)}(\boldsymbol{x}) = \cos[f^{(\ell)}(\boldsymbol{\theta}_{i}; \boldsymbol{x}), f^{(\ell)}(\boldsymbol{\theta}_{j}; \boldsymbol{x})]$.
In \cref{fig:LLFC_main_cosine}, the values of $\mathbb{E}_{\mathcal{D}}[1-{\rm cosine}_{\alpha}^{(\ell)}(\boldsymbol{x})]$ consistently approach $0$ across a range of layers, $\alpha$, and various pairs of $(\boldsymbol{\theta}_i, \boldsymbol{\theta}_j)$, under different task settings.
The small error bars indicate a consistent behaviour across each datapoint in $\mathcal{D}_i \cup \mathcal{D}_j$.
Additionally, the values of baseline $\mathbb{E}_{\mathcal{D}}[1-\text{cosine}_{i,j}^{(\ell)}(\boldsymbol{x})]$ deviate from $0$, excluding the trivial scenario where $f^{(\ell)}({\boldsymbol{\theta}}_i)$ and $f^{(\ell)}({\boldsymbol{\theta}}_j)$ are already close enough. 
The results confirm that LLFC holds in the pretraining-finetuning paradigm.

\textbf{CTL Occurs in the Pretraining-Finetuning Paradigm.}
Building upon the observations of LLFC for finetuned models, we identify a stronger notion of linearity than LLFC, termed as Cross-Task Linearity (CTL).
Precisely, given a pair of finetuned models $(\boldsymbol{\theta}_i, \boldsymbol{\theta}_j) \in \boldsymbol{\Theta}^2$ and downstream tasks $\mathcal{D}_i$ and $\mathcal{D}_j$ respectively, we say them satisfy CTL on $\mathcal{D}_i \cup \mathcal{D}_j$ if $\forall \ell \in [L], \forall  \alpha \in [0,1]$, \begin{align*}
    f^{(\ell)}( \alpha {\boldsymbol{\theta}}_i + (1-\alpha) {\boldsymbol{\theta}}_j) \approx \alpha f^{(\ell)}({\boldsymbol{\theta}}_i) + (1-\alpha) f^{(\ell)}({\boldsymbol{\theta}}_j).
\end{align*}
Beyond LLFC, CTL enforces the approximate equality.
We have validated the features in weight-interpolated model $\boldsymbol{\theta}_{\alpha}$ and the linear interpolation of features in $\boldsymbol{\theta}_{i}$ and $\boldsymbol{\theta}_{j}$ have similar directions. 
To further validate CTL, we compare the length of their features at each layer $\ell$.
Specifically, for each pair of finetuned models $(\boldsymbol{\theta}_i, \boldsymbol{\theta}_j) \in \boldsymbol{\Theta}^2$, on each datapoint $\boldsymbol{x} \in \mathcal{D}_i \cup \mathcal{D}_j$, we measure ${\rm coef}_{\alpha}^{(\ell)}(\boldsymbol{x}) = \frac{\|f^{(\ell)}(\boldsymbol{\theta}_{\alpha}; \boldsymbol{x})\| \text{\ cosine}_{\alpha}^{(\ell)}(\boldsymbol{x})}{\|\alpha f^{(\ell)}(\boldsymbol{\theta}_i; \boldsymbol{x}) + (1-\alpha) f^{(\ell)}(\boldsymbol{\theta}_j; \boldsymbol{x})\|}$ \footnote{$\|f^{(\ell)}(\boldsymbol{\theta}_{\alpha}; \boldsymbol{x})\| \text{\ cosine}_{\alpha}^{(\ell)}(\boldsymbol{x})$ denotes the length of the projection of $f^{(\ell)}(\boldsymbol{\theta}_{\alpha}; \boldsymbol{x})$ onto the vector $\alpha f^{(\ell)}(\boldsymbol{\theta}_i; \boldsymbol{x}) + (1-\alpha) f^{(\ell)}(\boldsymbol{\theta}_j; \boldsymbol{x})$.}. 
In \cref{fig:LLFC_main_coef}, the values of ${\rm coef}_{\alpha}^{(\ell)}(\boldsymbol{x})$ are close to $1$ across various layers and different pairs of $(\boldsymbol{\theta}_i, \boldsymbol{\theta}_j)$, under different task settings.
Together with the results in \cref{fig:LLFC_main_cosine}, we confirm that CTL often occurs for the finetuned models.

In summary, we find that neural networks approximately function as linear maps in the pretraining-finetuning paradigm, mapping from the parameter space to the feature space.
This viewpoint enables us to study model merging/editing from a feature-learning perspective.

\begin{conjecture}[\textbf{Transitivity of CTL.}]\label{con:transitivity}
    Given models $\boldsymbol{\theta}_i, \boldsymbol{\theta}_j, \boldsymbol{\theta}_k$. We have $(\boldsymbol{\theta}_i, \boldsymbol{\theta}_k)$ satisfy CTL if $(\boldsymbol{\theta}_i, \boldsymbol{\theta}_j)$ and $(\boldsymbol{\theta}_j, \boldsymbol{\theta}_k)$ satisfy CTL.
\end{conjecture}

In addition, we conjecture the transitivity of CTL (see \cref{con:transitivity}). 
This is inferred from our results in \cref{fig:LLFC_main_cosine,fig:LLFC_main_coef,fig:LLFC_appendix_mlp_ctl,fig:LLFC_appendix_mlp_coef,fig:LLFC_appendix_resnet18_ctl,fig:LLFC_appendix_resnet18_coef} that $(\boldsymbol{\theta}_i, \boldsymbol{\theta}_k)$ is observed to satisfy CTL when $(\boldsymbol{\theta}_i, \boldsymbol{\theta}_j)$ and $(\boldsymbol{\theta}_j, \boldsymbol{\theta}_k)$ satisfy CTL.
We will later leverage the transitivity of CTL to prove the theorems.

\begin{figure}[tb!]
  \begin{center}
    \includegraphics[width=0.32\textwidth]{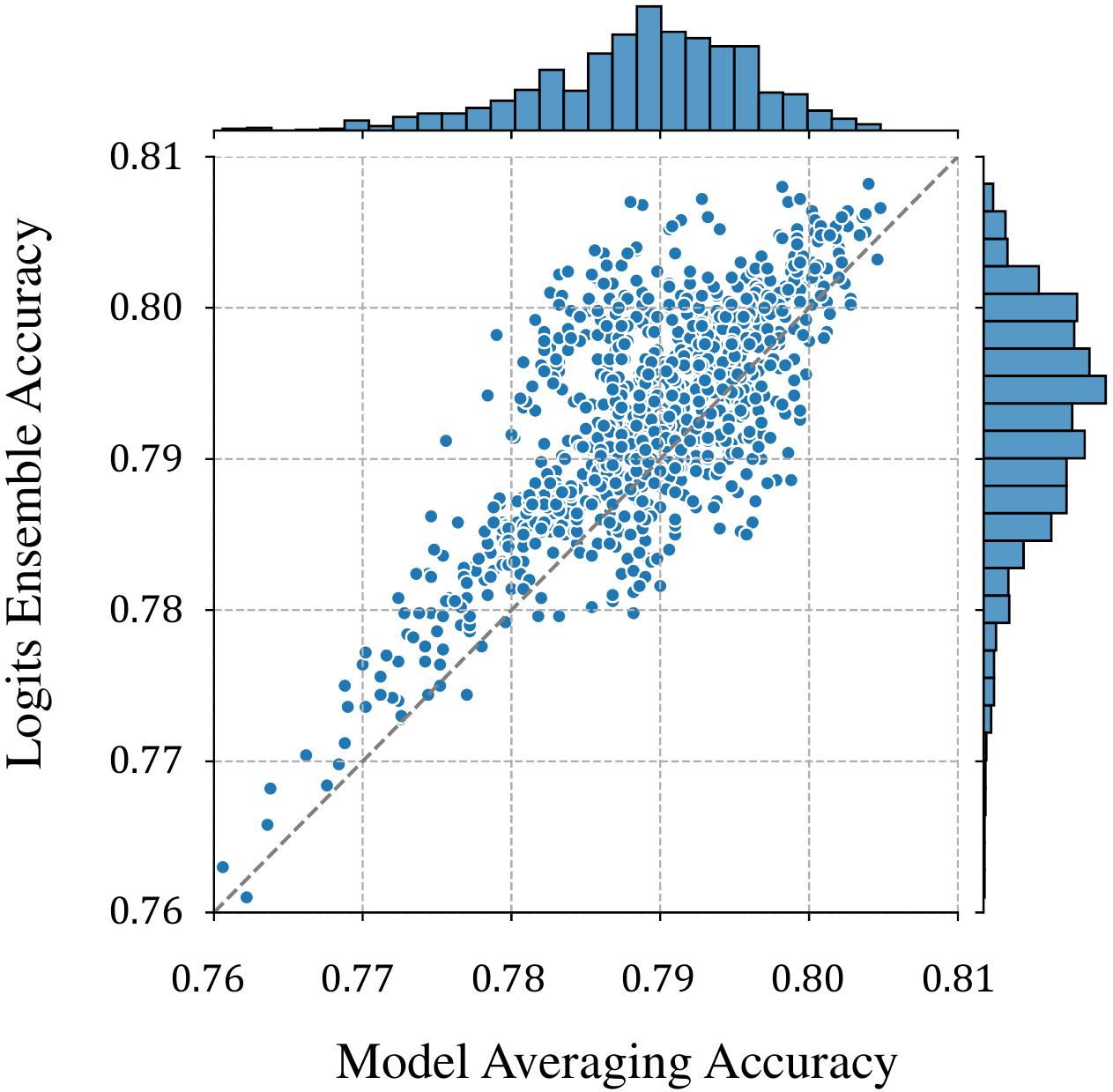}
    \vspace{-5pt}
    \caption{
        Linear correlation between the model averaging accuracy and the logits ensemble accuracy.
        Each datapoint represents three models fine-tuned on ImageNet with varying hyperparameters, denoted as $\{\boldsymbol{\theta}\}_{i=1}^3$.
        The x-axis represents accuracy of {\small $f(\frac{1}{3}\sum_{i=1}^3\boldsymbol{\theta}_i)$}, while the y-axis represents accuracy of {\small $\frac{1}{3}\sum_{i=1}^3f(\boldsymbol{\theta}_i)$}.
        The grey dashed line represents $y=x$.
    }
    \label{fig:ensemble_vs_averaging}
    \vspace{-15pt}
  \end{center}
\end{figure}

\begin{figure*}[tb!]
  \begin{center}
    \includegraphics[width=0.8423\textwidth]{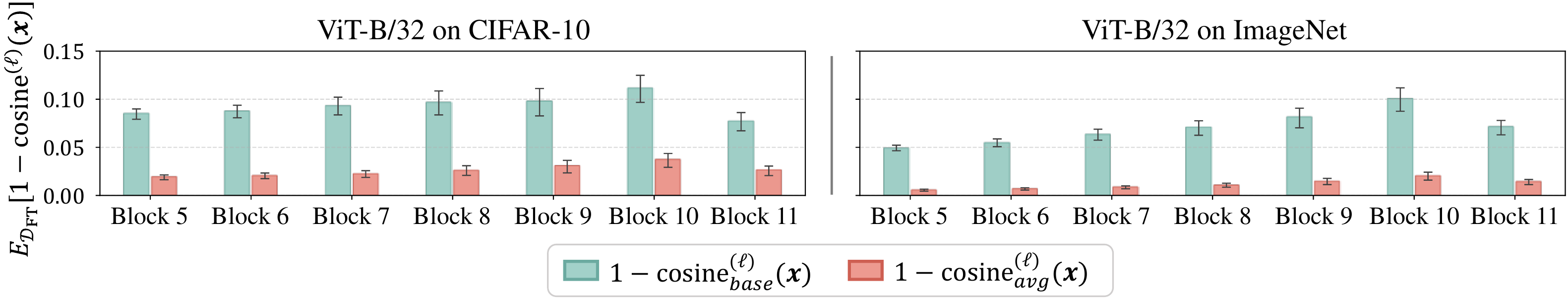}
   \vspace{-5pt}
   \caption{
        Verification of CTL in model averaging.
        Compare $\mathbb{E}_{\mathcal{D}}[1-{\rm cosine}_{avg}^{(\ell)}(\boldsymbol{x})]$ with $\mathbb{E}_{\mathcal{D}}[1-{\rm cosine}_{base}^{(\ell)}(\boldsymbol{x})]$.
        The bottom and top of the error bar represent the lower and upper quartile of the values across the dataset, respectively.
        The results are reported for last $7$ blocks of ViT-B/32 models, that are finetuned on CIFAR-10 and ImageNet, respectively.
        More results in \cref{suppl:exp_model_avg}.
    }
    \label{fig:LLFC_averaging_cosine}
    \vspace{-5pt}
  \end{center}
\end{figure*}

\begin{figure}[tb!]
  \begin{center}
    \includegraphics[width=0.45\textwidth]{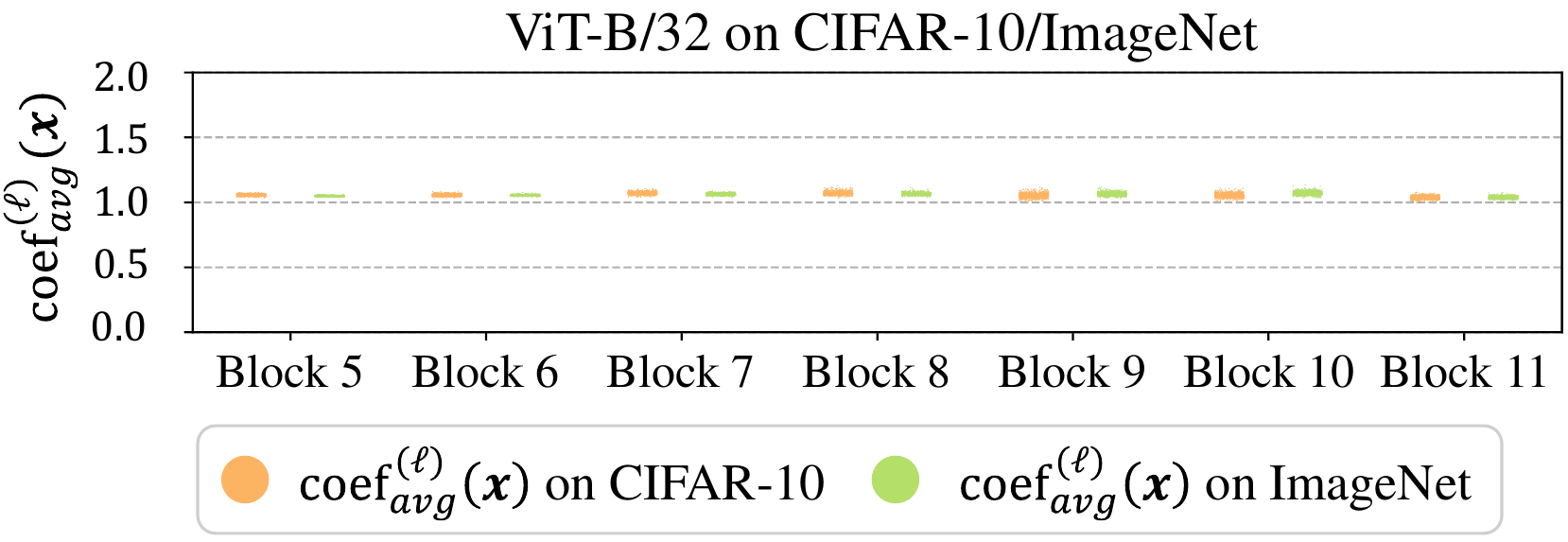}
    \vspace{-5pt}
    \caption{
        Verification of CTL in model averaging.
        Distribution of ${\rm coef}_{avg}^{(\ell)}(\boldsymbol{x})$ across the dataset. 
        The results are reported for last $7$ blocks of ViT-B/32 models finetuned on CIFAR-10 and ImageNet.
        More results in \cref{suppl:exp_model_avg}.
    }
    \label{fig:LLFC_averaging_coef}
    \vspace{-15pt}
  \end{center}
\end{figure}

\subsection{Insights into Model Averaging}
\label{sec:model_averaging}

Recent studies~\cite{wortsman2022model,Wortsman_2022_CVPR} discovered that averaging the weights of multiple models fine-tuned on the same task but with different hyperparameter configurations often leads to improved accuracy and robustness. 
This approach, termed as \emph{model averaging}, can be formulated as {\small $f(\frac{1}{k}\sum_{i=1}^k\boldsymbol{\theta}_i)$}.
Here, $\{\boldsymbol{\theta}_i\}_{i=1}^k$ represents the set of finetuned models, and the downstream task for finetuning is denoted as $\mathcal{D}_{\rm FT}$.
Alternatively, as another way to combine multiple models, the logits ensemble simply averages the outputs of different models, i.e., {\small $\frac{1}{k}\sum_{i=1}^kf(\boldsymbol{\theta}_i)$}.
Both methods are effective in improving overall model performance in practice and indeed a linear correlation has been observed between the accuracy of model averaging and logits ensemble (see \cref{fig:ensemble_vs_averaging}).
Here, we build a stronger connection between model averaging and logits ensemble in the feature space.

Specifically, we discover that the features in model averaging can be approximated by the averaging of features in each individual finetuned model, i.e., $\forall \ell \in [L]$, 
\begin{align}
    f^{(\ell)}\left(\frac{1}{k}\sum_{i=1}^k\boldsymbol{\theta}_i\right) \approx \frac{1}{k}\sum_{i=1}^k f^{(\ell)}({\boldsymbol{\theta}}_i). \label{eq:linearity_multi}
\end{align}
We conduct extensive experiments to validate our discovery. 
Similar to \cref{sec:LLFC_to_linearity}, on each datapoint $\boldsymbol{x} \in \mathcal{D}_{\rm FT}$, we measure the cosine similarity between the features in model averaging {\small $f^{(\ell)}(\frac{1}{k}\sum_{i=1}^k\boldsymbol{\theta}_i)$} and the averaging of features in each model {\small $\frac{1}{k}\sum_{i=1}^k f^{(\ell)}({\boldsymbol{\theta}}_i)$} at each layer $\ell$, denoted as {\small ${\rm cosine}_{avg}^{(\ell)}(\boldsymbol{x}) = \cos[f^{(\ell)}(\frac{1}{k}\sum_{i=1}^k\boldsymbol{\theta}_i; \boldsymbol{x}), \frac{1}{k}\sum_{i=1}^k f^{(\ell)}({\boldsymbol{\theta}}_i; \boldsymbol{x})]$}
Additionally, we compare ${\rm cosine}_{avg}^{(\ell)}(\boldsymbol{x})$ with the baseline {\small ${\rm cosine}_{base}^{(\ell)}(\boldsymbol{x}) =\frac{1}{k}\sum_{i=1}^k\cos[f^{(\ell)}(\frac{1}{k}\sum_{j=1}^k\boldsymbol{\theta}_j; \boldsymbol{x}), f^{(\ell)}({\boldsymbol{\theta}}_i; \boldsymbol{x})]$}.
We compute {\small ${\rm coef}_{avg}^{(\ell)}(\boldsymbol{x}) = \frac{\|f^{(\ell)}\left(\frac{1}{k}\sum_{i=1}^k\boldsymbol{\theta}_i; \boldsymbol{x}\right)\| {\rm cosine}_{avg}^{(\ell)}(\boldsymbol{x})}{\|\frac{1}{k}\sum_{i=1}^k f^{(\ell)}({\boldsymbol{\theta}}_i; \boldsymbol{x})\|}$} to validate the features have similar length.
In \cref{fig:LLFC_averaging_cosine}, the values of $\mathbb{E}_{\mathcal{D}}[1-{\rm cosine}_{avg}^{(\ell)}(\boldsymbol{x})]$ closely approach $0$ compared with the baseline $\mathbb{E}_{\mathcal{D}}[1-\text{cosine}_{base}^{(\ell)}(\boldsymbol{x})]$, and in \cref{fig:LLFC_averaging_coef}, the values of ${\rm coef}_{avg}^{(\ell)}(\boldsymbol{x})$ closely approximate $1$.
In conclusion, model averaging roughly aggregates the features in each individual finetuned model at each layer.

It is not difficult to see that our discovery could directly imply the observed linear correlation between the model averaging accuracy and the logits ensemble accuracy, particularly when \cref{eq:linearity_multi} is applied to the output layer.
Apparently, our discovery unveils a finer-grained characterization of the linear correlation between model averaging and logits ensemble.
Indeed, our discovery can be viewed as a generalization of CTL to the case of multiple models in the pretraining-finetuning paradigm (see \cref{thm:linearity_multi}). 
Hence, we conclude that CTL establishes a stronger connection between model averaging and logits ensemble in the feature space, thus further explaining the effectiveness of model averaging from a feature-learning perspective.

\begin{theorem}[\textbf{CTL Generalizes to Multiple Models} (Proof in \cref{suppl:proof_of_thm2})]\label{thm:linearity_multi}
    Given dataset $\mathcal{D}$ and a set of modes $\boldsymbol{\Theta}$ where each pair of modes $(\boldsymbol{\theta}_i, \boldsymbol{\theta}_j) \in \boldsymbol{\Theta}^2$ satisfy CTL on $\mathcal{D}$, assume transitivity of CTL (see \cref{con:transitivity}), then for any $\{\boldsymbol{\theta}_i\}_{i=1}^k \in \boldsymbol{\Theta}$ and $\{\alpha_i\}_{i=1}^k \in [0, 1]$, subject to the constraint that $\sum_{i=1}^{k} \alpha_{i} = 1$, we have
    \begin{align*}
        f^{(\ell)}\left(\sum_{i=1}^k \alpha_i \boldsymbol{\theta}_i\right) \approx \sum_{i=1}^k \alpha_i f^{(\ell)}({\boldsymbol{\theta}}_i), \quad \forall \ell \in [L].
    \end{align*}
\end{theorem}

\subsection{Insights into Task Arithmetic}
\label{sec:task_arithmetic}
\citet{ilharco2023editing} introduced task arithmetic for editing pretrained models using task vectors, which are obtained by subtracting the pretrained weights from the finetuned weights.
Specifically, considering a pretrained model $\boldsymbol{\theta}_{\rm PT}$ and a set of finetuned models $\{\boldsymbol{\theta}_i\}_{i=1}^k$ with corresponding downstream tasks $\{\mathcal{D}_i\}_{i=1}^k$, the task vectors $\{\tau_i\}_{i=1}^k$ are defined as $\tau_i = \boldsymbol{\theta}_i - \boldsymbol{\theta}_{PT}$. 
Arithmetic operations, including addition and negation, can be applied to the task vectors to obtain a new task vector $\tau_{\rm new}$, and the new task vector is then applied to the pretrained weights with a scaling term $\lambda$, i.e., $\boldsymbol{\theta}_{\rm new} = \boldsymbol{\theta}_{\rm PT} + \lambda \tau_{\rm new}$.
It allows to control the behavior of the edited model via simple arithmetic operations on task vectors.
In this subsection, we aim to explain the effectiveness of task arithmetic from a feature-learning perspective.

\begin{figure*}[tb!]
  \begin{center}
    \includegraphics[width=0.9474\textwidth]{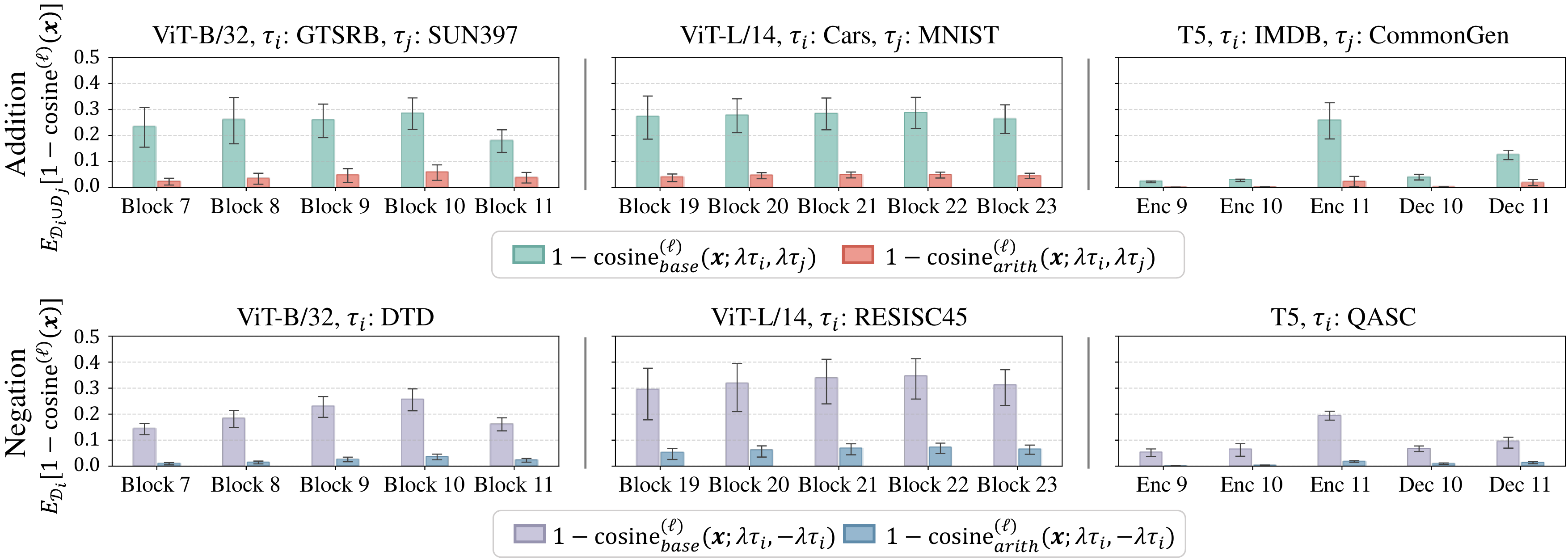}
    \vspace{-5pt}
    \caption{
        Verification of CTL in task arithmetic.
        \textbf{Addition}: Compare $\mathbb{E}_{\mathcal{D}}[1-{\rm cosine}_{arith}^{(\ell)}(\boldsymbol{x}; \lambda \tau_i, \lambda \tau_j)]$ with $\mathbb{E}_{\mathcal{D}}[1-{\rm cosine}_{base}^{(\ell)}(\boldsymbol{x}; \lambda \tau_i, \lambda \tau_j)]$.
        \textbf{Negation}: Compare $\mathbb{E}_{\mathcal{D}}[1-{\rm cosine}_{arith}^{(\ell)}(\boldsymbol{x}; \lambda \tau_i, -\lambda \tau_i)]$ with $\mathbb{E}_{\mathcal{D}}[1-{\rm cosine}_{base}^{(\ell)}(\boldsymbol{x}; \lambda \tau_i, -\lambda \tau_i)]$.
        The bottom and top of the error bar represent the lower and upper quartile of the values across the dataset, respectively.
        The results are reported for last $5$ blocks of finetuned models under different settings, with $\lambda = 0.4$~\cite{ilharco2023editing}.
        More results in \cref{suppl:exp_addi,suppl:exp_neg}.
    }
    \label{fig:LLFC_task_cosine}
    \vspace{-5pt}
  \end{center}
\end{figure*}

\begin{figure*}[tb!]
  \begin{center}
    \includegraphics[width=0.9264\textwidth]{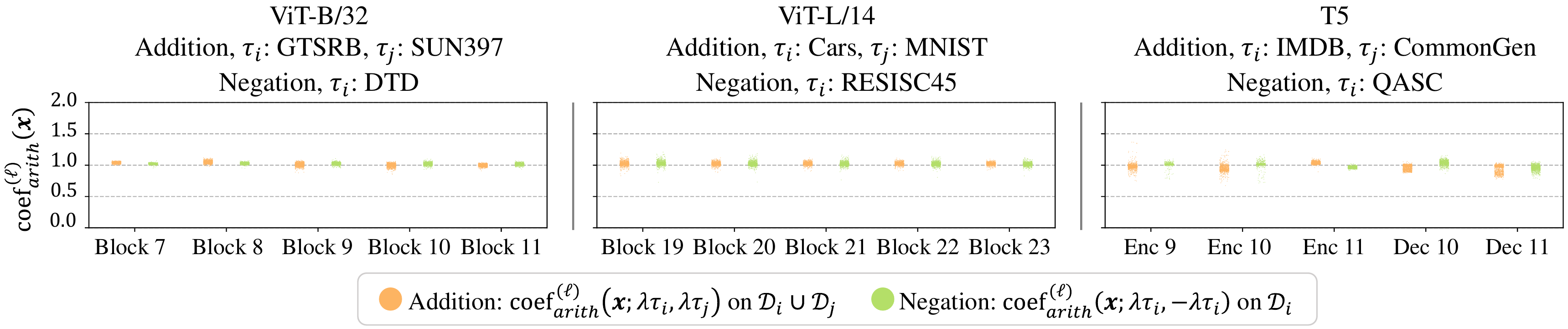}
    \vspace{-5pt}
    \caption{
        Verification of CTL in task arithmetic.
        \textbf{Addition}: Distribution of ${\rm coef}_{arith}^{(\ell)}(\boldsymbol{x}; \lambda \tau_i, \lambda \tau_j)$.
        \textbf{Negation}: Distribution of ${\rm coef}_{arith}^{(\ell)}(\boldsymbol{x}; \lambda \tau_i, -\lambda \tau_i)$.
        The results are reported for last $5$ blocks of finetuned models under different settings, with $\lambda = 0.4$~\cite{ilharco2023editing}.
        More results in \cref{suppl:exp_addi,suppl:exp_neg}.
    }
    \label{fig:LLFC_task_coef}
    \vspace{-15pt}
  \end{center}
\end{figure*}

\textbf{CTL Explains Learning via Addition.}
An intriguing discovery in task arithmetic is that the addition of task vectors builds multi-task models. 
For instance, with a proper chosen $\lambda$, $f\left(\boldsymbol{\theta}_{\rm PT} + \lambda \left(\tau_i + \tau_j\right)\right)$ demonstrate comparable performance on both tasks $\mathcal{D}_i$ and $\mathcal{D}_j$.
Despite this surprising observation, it is not well understood why addition in the parameter space leads to the multi-task abilities.

We aim to interpret the addition operation from a feature-learning perspective. Assuming CTL holds for the edited models, we can easily derive that $\forall \ell \in [L]$, 
\begin{align}
    \begin{split}
        &f^{(\ell)}\left(\boldsymbol{\theta}_{\rm PT} + \lambda (\tau_i + \tau_j)\right) \\
        \approx& \frac{1}{2}f^{(\ell)}(\boldsymbol{\theta}_{\rm PT} + 2\lambda\tau_i) +\frac{1}{2}f^{(\ell)}(\boldsymbol{\theta}_{\rm PT} + 2\lambda\tau_j). 
    \end{split}\label{eq:linearity_addi}
\end{align}
We conduct experiments to verify \cref{eq:linearity_addi}.
Specifically, given a pair of task vectors $(\tau_i, \tau_j)$, on each datapoint $\boldsymbol{x} \in \mathcal{D}_i \cup \mathcal{D}_j$, we measure the cosine similarity between LHS and RHS of \cref{eq:linearity_addi}, i.e., ${\rm cosine}_{arith}^{(\ell)}(\boldsymbol{x}; 2\lambda\tau_i, 2\lambda \tau_j) = \cos[f^{(\ell)}\left(\boldsymbol{\theta}_{\rm PT} + \lambda (\tau_i + \tau_j)\right), \frac{1}{2}f^{(\ell)}(\boldsymbol{\theta}_{\rm PT} + 2\lambda\tau_i; \boldsymbol{x}) + \frac{1}{2}f^{(\ell)}(\boldsymbol{\theta}_{\rm PT} + 2\lambda\tau_j; \boldsymbol{x})]$ at each layer $\ell$.
Similarly as before, we compare it with the baseline cosine similarity, i.e., ${\rm cosine}_{base}^{(\ell)}(\boldsymbol{x}; 2\lambda\tau_i, 2\lambda\tau_j) =\cos[f^{(\ell)}(\boldsymbol{\theta}_{\rm PT} + 2\lambda\tau_i; \boldsymbol{x}), f^{(\ell)}(\boldsymbol{\theta}_{\rm PT} + 2\lambda\tau_j; \boldsymbol{x})]$.
Additionally, we examine the approximate equality via {\small ${\rm ceof}_{arith}^{\ell}(\boldsymbol{x}; 2\lambda\tau_i, 2\lambda \tau_j) = \frac{\|f^{(\ell)}\left(\boldsymbol{\theta}_{\rm PT} + \lambda (\tau_i + \tau_j)\right)\| {\rm \ cosine}_{arith}^{(\ell)}(\boldsymbol{x}; 2\lambda\tau_i, 2\lambda \tau_j)}{\|\frac{1}{2}f^{(\ell)}(\boldsymbol{\theta}_{\rm PT} + 2\lambda\tau_i; \boldsymbol{x}) + \frac{1}{2}f^{(\ell)}(\boldsymbol{\theta}_{\rm PT} + 2\lambda\tau_j; \boldsymbol{x})\|}$}. 
In \cref{fig:LLFC_task_cosine}, the values of $\mathbb{E}_{\mathcal{D}}[1-{\rm cosine}_{arith}^{(\ell)}(\boldsymbol{x}; 2\lambda \tau_i, 2\lambda\tau_j)]$ are close to $0$ compared with $\mathbb{E}_{\mathcal{D}}[1-{\rm cosine}_{base}^{(\ell)}(\boldsymbol{x}; 2\lambda \tau_i, 2\lambda\tau_j)]$, and in \cref{fig:LLFC_task_coef}, the values of ${\rm ceof}_{arith}^{\ell}(\boldsymbol{x}; 2\lambda\tau_i, 2\lambda \tau_j)$ are distributed around 1.
Hence we conclude that the features in the model applied with the addition of two task vectors can be approximated by the addition of the features in two models, each applied with a single task vector.

\begin{figure}[tb!]
  \begin{center}
    \includegraphics[width=0.30\textwidth]{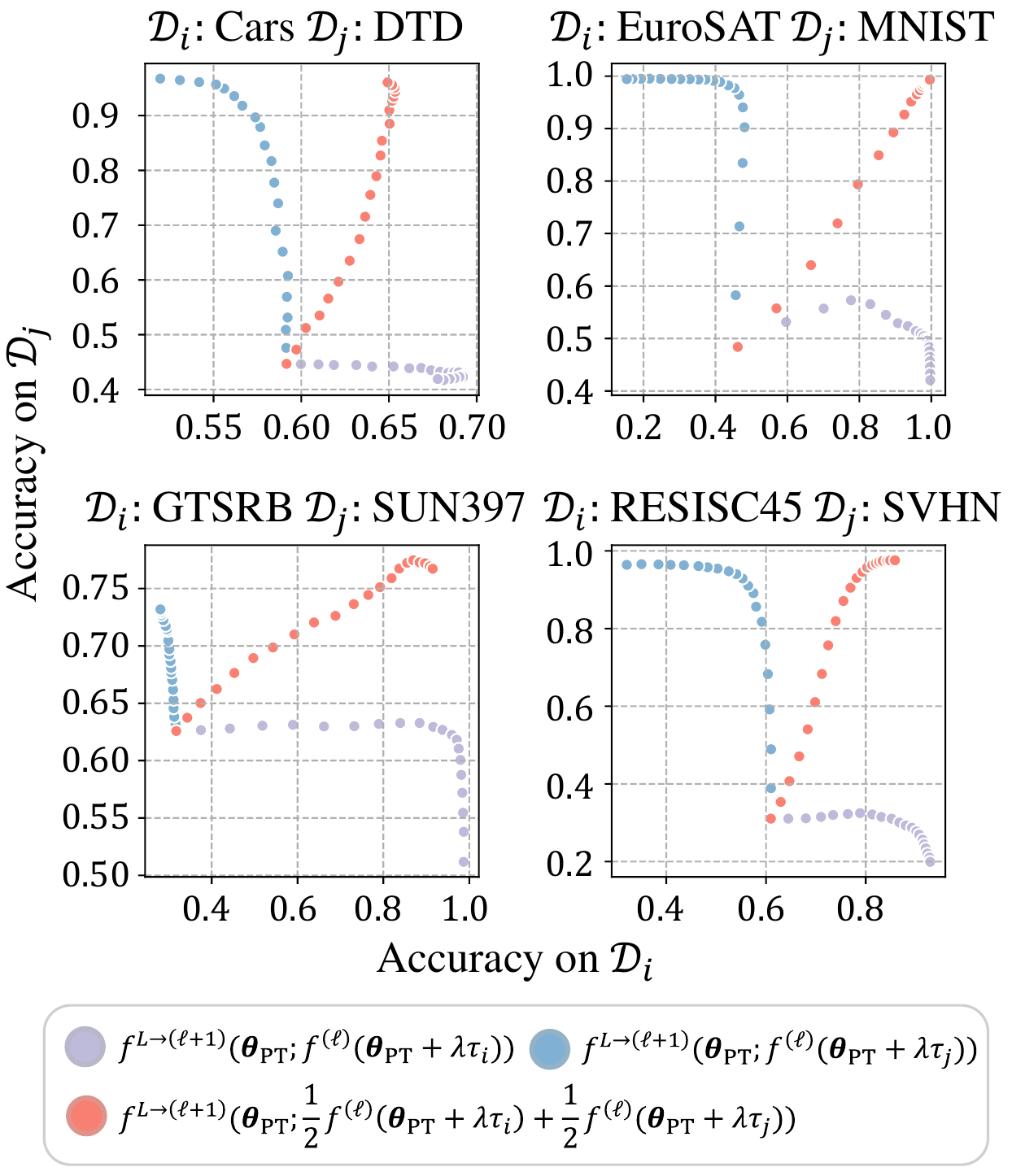}
    \vspace{-5pt}
    \caption{
        Accuracy of $f^{L\leftarrow (\ell + 1)}(\boldsymbol{\theta}_{\rm PT}; \frac{1}{2}f^{(\ell)}(\boldsymbol{\theta}_{\rm PT} + \lambda \tau_i) +\frac{1}{2}f^{(\ell)}(\boldsymbol{\theta}_{\rm PT} + \lambda \tau_j))$, $f^{L\leftarrow (\ell + 1)}(\boldsymbol{\theta}_{\rm PT}; f^{(\ell)}(\boldsymbol{\theta}_{\rm PT} + \lambda \tau_i))$ and $f^{L\leftarrow (\ell + 1)}(\boldsymbol{\theta}_{\rm PT}; f^{(\ell)}(\boldsymbol{\theta}_{\rm PT} + \lambda \tau_j))$ on both $\mathcal{D}_i$ (x-axis) and $\mathcal{D}_j$ (y-axis). 
        Results are reported for ViT-B/32 with various combinations of task vectors and different values of $\lambda \in [0.05, 1]$. 
        The stitching layers $\ell$ is chosen to be Block-$9$.
    }
    \label{fig:stitch_addi}
    \vspace{-15pt}
  \end{center}
\end{figure}

Though \cref{eq:linearity_addi} has transformed the addition from the parameter space to the feature space, the reason why addition in the feature space constructs multi-task models remains unclear.
In fact, we discover that if we replace the features in the pretrained model by the average of the features in two finetuned model, i.e., $f^{L\leftarrow (\ell + 1)}(\boldsymbol{\theta}_{\rm PT}; \frac{1}{2}f^{(\ell)}(\boldsymbol{\theta}_{i}) +\frac{1}{2}f^{(\ell)}(\boldsymbol{\theta}_{j}))$, the model with replaced features could demonstrate abilities on both $\mathcal{D}_i$ and $\mathcal{D}_j$.
Here, $f^{L\leftarrow (\ell + 1)}(\boldsymbol{\theta}; \cdot)$ denotes the mapping from the internal features of the network $f(\boldsymbol{\theta})$ at $\ell$-th layer to the final output.
This feature replacement shares a similar methodology with the model stitching\footnote{Model stitching~\cite{lenc2015understanding,bansal2021revisiting} is a widely used technique for analyzing the internal representations of networks. It stitches the front part of one model with the back part of another model by a learnable linear layer. If stitched model retains a good performance on target task, we say that the two model share a similar representation at the stitching layer. In our case, no learnable linear layer is employed.}, and thus, we term the model with replaced features as the stitched model. 
In \cref{fig:stitch_addi}, across various combinations of $\tau_i$ and $\tau_j$ and different values of $\lambda$, the stitched model $f^{L\leftarrow (\ell + 1)}(\boldsymbol{\theta}_{\rm PT}; \frac{1}{2}f^{(\ell)}(\boldsymbol{\theta}_{\rm PT} + \lambda \tau_i) +\frac{1}{2}f^{(\ell)}(\boldsymbol{\theta}_{\rm PT} + \lambda \tau_j))$ achieves comparable performance on both $\mathcal{D}_i$ and $\mathcal{D}_j$, while $f^{L\leftarrow (\ell + 1)}(\boldsymbol{\theta}_{\rm PT}; f^{(\ell)}(\boldsymbol{\theta}_{\rm PT} + \lambda \tau_i))$ and $f^{L\leftarrow (\ell + 1)}(\boldsymbol{\theta}_{\rm PT}; f^{(\ell)}(\boldsymbol{\theta}_{\rm PT} + \lambda \tau_j))$ are only capable of single tasks.
Therefore, we conclude that the addition in the feature space actually aggregates the task-specific information from both tasks, thereby bridging the multi-task abilities and CTL.

\begin{figure}[tb!]
  \begin{center}
    \includegraphics[width=0.40\textwidth]{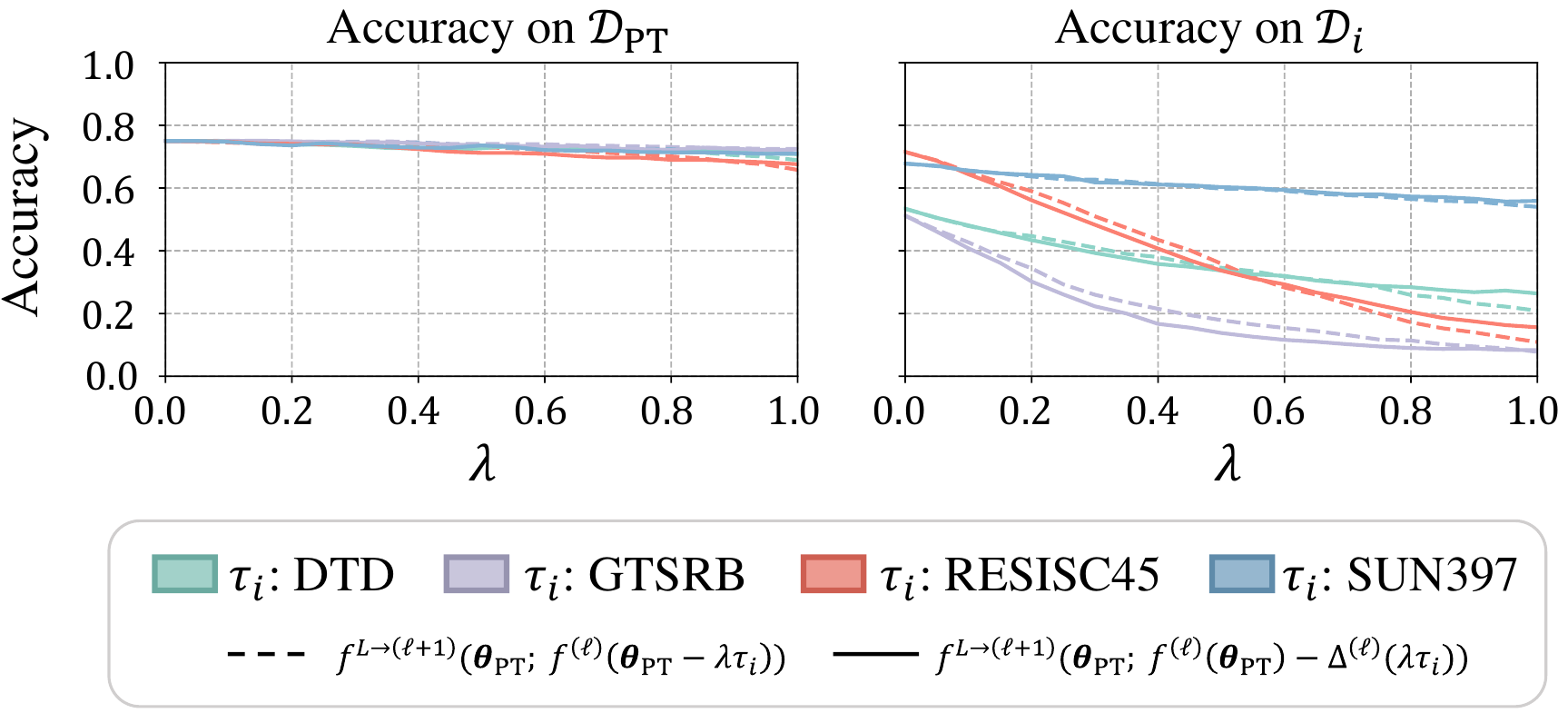}
    \vspace{-5pt}
    \caption{
        Accuracy of $f^{L\to (\ell + 1)}(\boldsymbol{\theta}_{\rm PT}; f^{(\ell)}(\boldsymbol{\theta}_{\rm PT})-\Delta^{(\ell)}(\lambda \tau_i))$ and $f^{L\to (\ell + 1)}(\boldsymbol{\theta}_{\rm PT}; f^{(\ell)}(\boldsymbol{\theta}_{\rm PT} -\lambda \tau_i))$ on $\mathcal{D}_{\rm PT}$ and $\mathcal{D}_i$ v.s. $\lambda$. 
        Results are reported for ViT-L/14 with different task vectors. 
        The stitching layer $\ell$ is chosen to be Block-$19$.
    }
    \label{fig:stitch_neg}
    \vspace{-15pt}
  \end{center}
\end{figure}

\textbf{CTL Explains Forgetting via Negation.}
Another surprising finding in task arithmetic is that negating a task vector removes the ability of the pretrained model on the corresponding task.
Specifically, the edited model $f(\boldsymbol{\theta}_{\rm PT} - \lambda \tau_i)$ forgets its proficiency on $\mathcal{D}_i$ while maintaining its performance elsewhere.
Further exploration of the underlying reasons for this forgetting effect is encouraged.

We still explain the negation operation from a feature-learning perspective.
Assume CTL satisfied for the edited models, we can simply obtain that, $\forall \ell \in [L]$, 
\begin{align}
    \small
    \begin{split}
        f^{(\ell)}\left(\boldsymbol{\theta}_{\rm PT}\right) \approx \frac{1}{2}f^{(\ell)}(\boldsymbol{\theta}_{\rm PT} + \lambda\tau_i) +\frac{1}{2}f^{(\ell)}\left(\boldsymbol{\theta}_{\rm PT} - \lambda\tau_i\right).
    \end{split}\label{eq:linearity_neg}
\end{align}
To verify \cref{eq:linearity_neg}, given a task vector $\tau_i$, on each datapoint $\boldsymbol{x} \in \mathcal{D}_i$, we measure ${\rm cosine}_{arith}^{(\ell)}(\boldsymbol{x}; \lambda\tau_i, -\lambda \tau_i)$ and compare it with ${\rm cosine}_{base}^{(\ell)}(\boldsymbol{x}; \lambda\tau_i, -\lambda \tau_i)$. 
We also compute ${\rm coef}_{arith}^{(\ell)}(\boldsymbol{x}; \lambda\tau_i, -\lambda \tau_i)$. 
The results in \cref{fig:LLFC_task_cosine,fig:LLFC_task_coef} validate our hypothesis in \cref{eq:linearity_neg}.

\cref{eq:linearity_neg} interprets the negation in the parameter space as the negation in the feature space, as can be rewritten as:
\begin{align}
    \begin{split}
        f^{(\ell)}\left(\boldsymbol{\theta}_{\rm PT} - \lambda \tau_i\right) \approx f^{(\ell)}\left(\boldsymbol{\theta}_{\rm PT}\right) - \Delta^{(\ell)}(\lambda \tau_i),
    \end{split}\label{eq:linearity_neg_re}
\end{align}
where $\Delta^{(\ell)}(\lambda \tau_i) = f^{(\ell)}\left(\boldsymbol{\theta}_{\rm PT} + \lambda \tau_i\right) - f^{(\ell)}\left(\boldsymbol{\theta}_{\rm PT}\right)$. 
Intuitively, $\Delta^{(\ell)}(\lambda \tau_i)$ encodes the extra information specific to the task $\mathcal{D}_i$.
Therefore, $f^{(\ell)}\left(\boldsymbol{\theta}_{\rm PT} - \lambda \tau_i\right)$ loses the task-specific information of $\Delta^{(\ell)}(\lambda \tau_i)$, while retaining most information of $f^{(\ell)}\left(\boldsymbol{\theta}_{\rm PT}\right)$.

We now examine the ability of the negation in the feature space through model stitching. Specifically, we measure the accuracy of the stitched model $f^{L\leftarrow (\ell + 1)}(\boldsymbol{\theta}_{\rm PT}; f^{(\ell)}(\boldsymbol{\theta}_{\rm PT})-\Delta^{(\ell)}(\lambda \tau_i))$ on the downstream task $\mathcal{D}_i$ and the pretraining task $\mathcal{D}_{\rm PT}$.
In \cref{fig:stitch_neg}, with the increase of $\lambda$, the accuracy of $f^{L\leftarrow (\ell + 1)}(\boldsymbol{\theta}_{\rm PT}; f^{(\ell)}(\boldsymbol{\theta}_{\rm PT})-\Delta^{(\ell)}(\lambda \tau_i))$ drops significantly on $\mathcal{D}_{i}$ while keeping nearly constant on $\mathcal{D}_{\rm PT}$.
We also evaluate $f^{L\leftarrow (\ell + 1)}(\boldsymbol{\theta}_{\rm PT}; f^{(\ell)}(\boldsymbol{\theta}_{\rm PT} -\lambda \tau_i))$, which shows a similar performance to $f^{L\leftarrow (\ell + 1)}(\boldsymbol{\theta}_{\rm PT}; f^{(\ell)}(\boldsymbol{\theta}_{\rm PT})-\Delta^{(\ell)}(\lambda \tau_i))$ on both tasks, thus further validating \cref{eq:linearity_neg_re}. In conclusion, CTL translates the negation in the parameter space as the negation in the feature space, which further induces the aforementioned forgetting effect.

\textbf{CTL Implies Task Arithmetic.}
\citet{ortiz-jimenez2023task} proposed the weight disentanglement as the necessary condition to perform task arithmetic.
We show that the weight disentanglement is roughly a consequence of CTL (see \cref{thm:ctl_connects_wd} and discussion in \cref{suppl:discussion}).

\begin{theorem}[\textbf{CTL Connects to Weight Disentanglement. (Proof in \cref{suppl:proof_of_thm3})}]\label{thm:ctl_connects_wd}
    Given pretrained model $\boldsymbol{\theta}_{\rm PT}$ and a set of task vectors $\Upsilon=\{\tau_i\}_{i=1}^k$, suppose each pair of edited models $(\boldsymbol{\theta}_{\rm PT} + \lambda_i \tau_i, \boldsymbol{\theta}_{\rm PT} + \lambda_j \tau_j)$ satisfy CTL when $\lambda_i, \lambda_j \in [-\beta, \beta]$, assume the transitivity of CTL (see \cref{con:transitivity}), then $\forall \{\alpha_i\}_{i=1}^k \in [-\frac{\beta}{k+1}, \frac{\beta}{k+1}]$,
    \begin{align*}
        f\left(\boldsymbol{\theta}_{\rm PT} + \sum_{i=1}^k\alpha_i \tau_i; \boldsymbol{x}\right) \approx \sum_{i=1}^{k}g_i(\alpha_i \tau_i; \boldsymbol{x}) + g_0(\boldsymbol{x})
    \end{align*}
    where $g_i(\alpha_i \tau_i;\boldsymbol{x}) = \frac{1}{k+1}f(\boldsymbol{\theta}_{\rm PT} + (k+1)\alpha_i \tau_i; \boldsymbol{x})$ and $g_0(\boldsymbol{x}) = \frac{1}{k+1}f(\boldsymbol{\theta}_{\rm PT}; \boldsymbol{x})$. 
\end{theorem}

\section{Unveiling the Root Cause of CTL}\label{sec:linearity_emerge}
We have seen CTL consistently occurs in the pretraining-finetuning paradigm, roughly characterizing networks as linear maps from the parameter space to the feature space. 
In this section, we aim to unveil the root cause of CTL.
We explore various factors contributing to the emergence of CTL, emphasizing the role of pretraining.
We also take a theoretical attempt to prove CTL.

\textbf{Factors Contributing to CTL.}
We investigate the impact of two factors, the number of pretraining/finetuning epoch and the task similarity, on the emergence of CTL.

\begin{figure}[tb!]
  \begin{center}
    \includegraphics[width=0.48\textwidth]{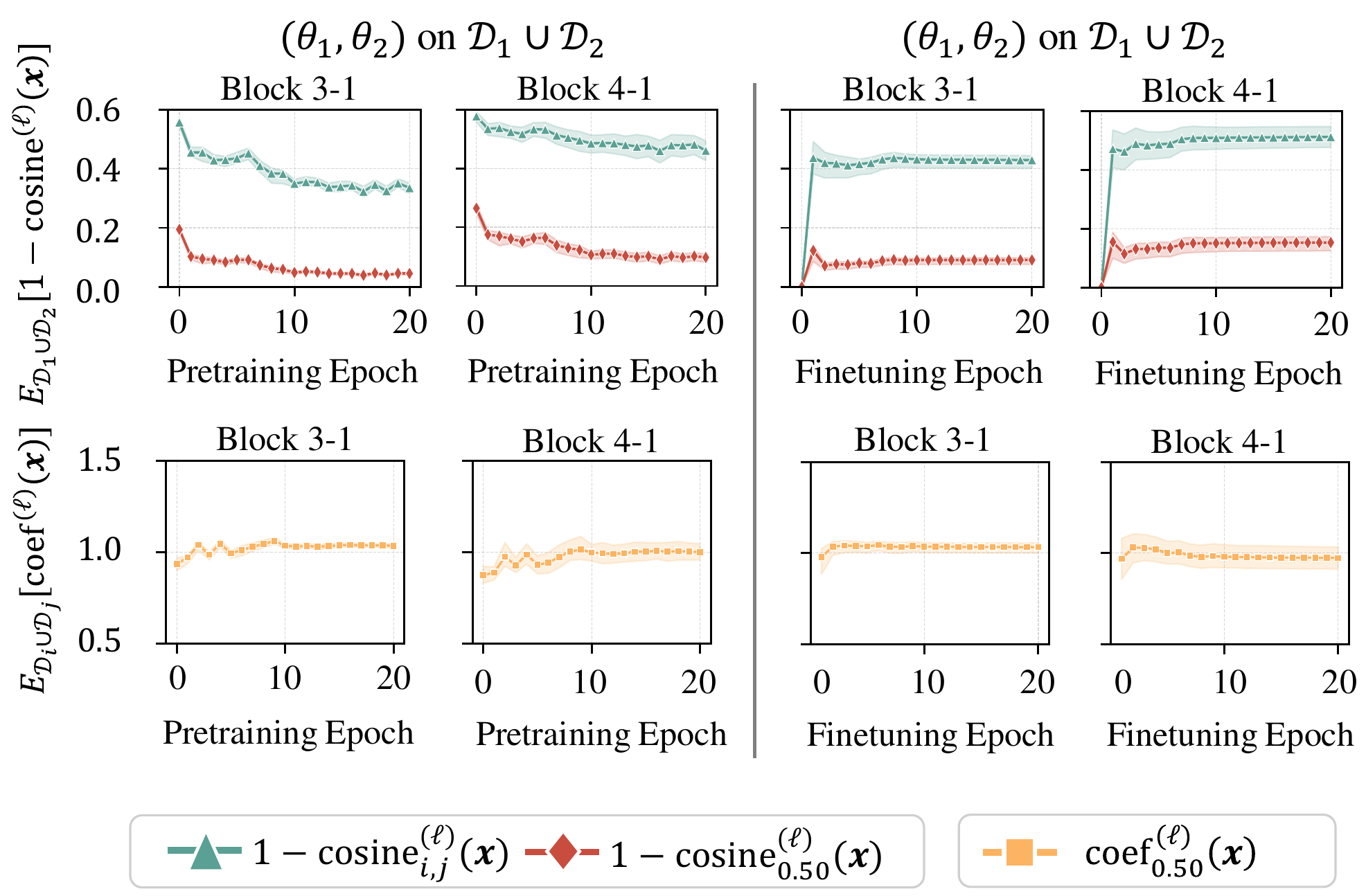}
    \vspace{-5pt}
    \caption{
    The impact of the number of pretraining/finetuning epochs on the emergence of CTL. 
    We show the change of $1-{\rm cosine}_{0.5}^{(\ell)}(\boldsymbol{x})$ and ${\rm coef}_{0.5}^{(\ell)}(\boldsymbol{x})$ w.r.t. the number of pretraining/finetuning epoch. 
    \textbf{Left}: We vary the number of pretraining epochs from 0 to 20 and fix the number of finetuning epochs to 10.
    \textbf{Right}: We fix the number of pretraining epochs to 10 and vary the number of finetuning epochs from 0 to 20.
    Results are reported for ResNet-18s pretrained and finetuned on Split CIFAR-100.
    }
    \label{fig:pretrain_finetune_epoch}
    \vspace{-15pt}
  \end{center}
\end{figure}

\begin{figure}[tb!]
  \begin{center}
    \includegraphics[width=0.484436\textwidth]{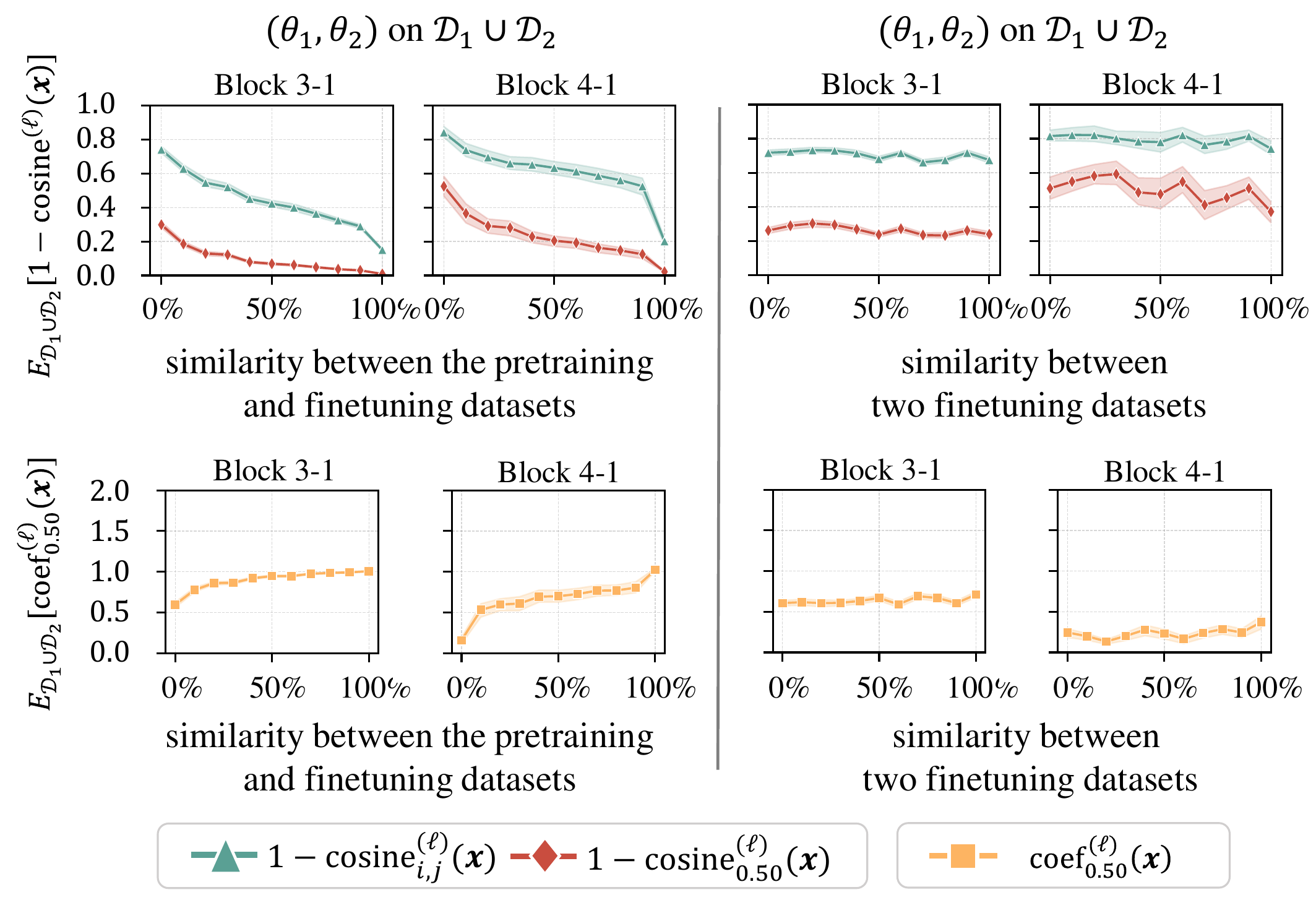}
    \vspace{-5pt}
    \caption{
    The impact of task similarity on the emergence of CTL.
    We show the change of $1-{\rm cosine}_{0.5}^{(\ell)}(\boldsymbol{x})$ and ${\rm coef}_{0.5}^{(\ell)}(\boldsymbol{x})$ w.r.t. the ratio of replaced samples (task  similarity).
    \textbf{Left}: We replace the samples from the two finetuning dataset with the samples from the pretraining dataset.
    \textbf{Right}: We replace the samples from one finetuning dataset with the samples from the other finetuning dataset.
    Results are reported for ResNet-18s pretrained and finetuned on Split ImageNet.
    }
    \label{fig:task_difference}
    \vspace{-15pt}
  \end{center}
\end{figure}

\textbf{\textit{i)} The number of pretraining/finetuning epochs.}
First, we study the impact of the pretraining epochs on CTL.
Specifically, we vary the number of pretraining epochs and fix the the number of finetuning epochs.
Consistently with \cref{sec:LLFC_to_linearity}, we measure $\text{cosine}_{0.5}^{(\ell)}(\boldsymbol{x})$, $\text{cosine}_{i,j}^{(\ell)}(\boldsymbol{x})$ and ${\rm coef}_{0.5}^{(\ell)}(\boldsymbol{x})$ for each pair of finetuned models.
In \cref{fig:pretrain_finetune_epoch} (left), $\mathbb{E}_{\mathcal{D}}[1-{\rm cosine}_{0.5}^{(\ell)}(\boldsymbol{x})]$ decreases as the number of pretraining epochs increases and the values of ${\rm coef}_{0.5}^{(\ell)}(\boldsymbol{x})$ gradually concentrate around $1$.
Results indicate that increasing pretraining epochs promotes the emergence of CTL.

Second, we study the impact of the finetuning epochs on CTL.
Similarly, we fix the number of pretraining epochs and vary the number of pretraining epochs.
In \cref{fig:pretrain_finetune_epoch} (right), the values of $\mathbb{E}_{\mathcal{D}}[1-{\rm cosine}_{0.5}^{(\ell)}(\boldsymbol{x})]$ deviate slightly from $0$ once finetuning for at least one epoch, while there is no clear trend for ${\rm coef}_{0.5}^{(\ell)}(\boldsymbol{x})$. 
Results imply that the number of finetuning epochs has lesser impact on CTL.

\textbf{\textit{ii)} The task similarity.}
First, we study the impact of the similarity between the pretraining and finetuning tasks on CTL. 
However, directly quantifying task similarity is challenging.
Therefore, we adopt an alternative method: replacing the samples from the finetuning tasks $\mathcal{D}_i$ and $\mathcal{D}_j$ with the samples from the pretraining task $\mathcal{D}_{\rm PT}$ and finetuning the models on $\mathcal{D}_i$ and $\mathcal{D}_j$ with these replaced samples.
The ratio of replaced samples, i.e., {\small $\frac{\text{\# of replaced samples}}{\text{\# of total samples in} \mathcal{D}_i \backslash \mathcal{D}_j}$}, serves as a measure of task similarity\footnote{Notably, when the ratio is $1$, $\mathcal{D}_i$, $\mathcal{D}_j$ and $\mathcal{D}_{\rm PT}$ are identical.}. 
We then calculate $\text{cosine}_{0.5}^{(\ell)}(\boldsymbol{x})$ and ${\rm coef}_{0.5}^{(\ell)}(\boldsymbol{x})$ for each pair of finetuned models at different ratio of replaced samples.
In \cref{fig:task_difference} (left), as the ratio of replaced samples increases, $\mathbb{E}_{\mathcal{D}}[1-{\rm cosine}_{0.5}^{(\ell)}(\boldsymbol{x})]$ decreases, and the values of ${\rm coef}_{0.5}^{(\ell)}(\boldsymbol{x})$ gradually approaches $1$. 
Results indicate that the similarity between pretraining task and finetuning task promotes the emergence of CTL.

Second, we study the impact of the similarity between the two finetuning tasks on CTL. 
Similarly, we replace the samples from one finetuning task with the samples from the other finetuning task and use the ratio of replaced samples as a measure of task similarity\footnote{When the ratio is $1$, $\mathcal{D}_i$ and $\mathcal{D}_j$ are identical yet different from $\mathcal{D}_{\rm PT}$}.
In \cref{fig:task_difference} (right), the values of $\mathbb{E}_{\mathcal{D}}[1-{\rm cosine}_{0.5}^{(\ell)}(\boldsymbol{x})]$ remain nearly constant regardless of how similar the two finetuning tasks are, and the distributions of ${\rm coef}_{0.5}^{(\ell)}(\boldsymbol{x})$ change little as well. 
Results demonstrate that the similarity between finetuning tasks has lesser effects on CTL.

In summary, our investigation indicate that the number of pretraining epochs and the similarity between pretraining and finetuning tasks promote the emergence of CTL, emphasizing the role of pretraining.

\textbf{Theoretical Attempt to Prove CTL.}
In addition, we take a first step to prove CTL.
Specifically, we prove that the emergence of CTL is related to the flatness of the landscape of $f(\cdot)$ and the distance between two finetuned models (see \cref{thm:linearity_emerge}).
We affirm our theorem by demonstrating a strong correlation between $\delta_{i,j}$ and $\frac{\alpha (1-\alpha) \lambda_{\text{max}}}{2}\norm{\boldsymbol{\theta}_i - \boldsymbol{\theta}_j}^2$ (Results in \cref{suppl:exp_verify_thm}).
\begin{theorem}[\textbf{The Emergence of CTL} (Proof in \cref{suppl:proof_of_thm1})]\label{thm:linearity_emerge}
    Suppose $f(\boldsymbol{\theta}): \R^{p} \mapsto \R$ is third differentiable in an open convex set $\boldsymbol{\Theta}$ and the its Hessian norm at $\boldsymbol{\theta}_0$ is bounded by $\lambda_{\text{min}} \leq \norm{\nabla^2 f(\boldsymbol{\theta}_0)} \leq \lambda_{\text{max}}$, then
\begin{align*}
        \small
        \begin{split}
            \delta_{i,j}=&\abs{f(\alpha \boldsymbol{\theta}_i + (1-\alpha) \boldsymbol{\theta}_j) - \alpha f(\boldsymbol{\theta}_i) - (1-\alpha) f(\boldsymbol{\theta}_j)} \\
            \leq& \frac{\alpha (1-\alpha) \lambda_{\text{max}}}{2}\norm{\boldsymbol{\theta}_i - \boldsymbol{\theta}_j}^2 + \mathcal{E},
        \end{split}
    \end{align*}
where {\small $\mathcal{E}=O\left(\max \left(\norm{ \alpha \boldsymbol{\theta}_i + (1-\alpha) \boldsymbol{\theta}_j - \boldsymbol{\theta}_0}^3, \alpha \norm{\boldsymbol{\theta}_i - \boldsymbol{\theta}_0}^3, \right.\right.$ $\left. \left. (1-\alpha) \norm{\boldsymbol{\theta}_j - \boldsymbol{\theta}_0}^3\right)\right)$} is the high order term.
\end{theorem}
\begin{remark}
    Previous studies found linearizing models is insufficient to explain LMC and task arithmetic~\cite{fort2020deep,ortiz-jimenez2023task}.
    In \cref{thm:linearity_emerge}, instead of linearizing models, we provide a more realistic setting.
\end{remark}

\section{Conclusion and Limitations}
In this work, we identified Cross-Task Linearity (CTL) as a prevalent phenomenon that consistently occurs for finetuned models, approximately characterizing neural networks as linear maps in the pretraining-finetuning paradigm.
Based on the observed CTL, we obtained novel insights into two widely-used model merging/editing techniques: model averaging and task arithmetic.
Furthermore, we studied the root cause of CTL, highlighting the role of pretraining.

Our current work primarily focuses on empirical findings, despite a theoretical attempt to prove CTL in \cref{sec:linearity_emerge}.
We defer a thorough theoretical analysis to future work.
Additionally, on the practical side, we leave comprehensive exploration of CTL on Large Language Models to future.

\newpage


\section*{Acknowledgments}
This work was in part supported by the National Natural Science Foundation of China (62222607), the Shanghai Municipal Science and Technology Major Project (2021SHZDZX0102), the National Key R\&D Program of China (2022ZD0160104) and Shanghai Rising Star Program (23QD1401000).

\section*{Impact Statement}
This paper refers to the understanding of deep neural networks especially for pretrained models which themselves have significant impact to the society. In particular, we hope our technique can help better removing the unwanted behavior of trained neural networks.

\bibliography{ref}
\bibliographystyle{icml2024}

\newpage
\appendix
\onecolumn

\section{Connection with Prior Works}\label{suppl:discussion}

A recent work~\cite{ortiz-jimenez2023task} argue against the linearization hypothesis, which suggesting that finetuning occurs in the linear regime, by demonstrating the significant impact of non-linear terms on model behavior during training and merging. 
There is seemingly a contradiction between our main discovery, CTL, and the findings of  \citet{ortiz-jimenez2023task}, as CTL roughly characterizing the neural network as linear maps in the pretraining-finetuning paradigm.
In this section, we clarify that our main discovery, CTL, indeed does not contradict with the findings of \citet{ortiz-jimenez2023task}.

Let us first revisit the linearization hypothesis. 
As mentioned in \citet{ortiz-jimenez2023task}, the linearization hypothesis was originally proposed to explain the phenomenon that the outputs of weight-averaged model linearly correlate with the averaging of the outputs of each individual model~\cite{ilharco2022patching,wortsman2022model,Wortsman_2022_CVPR}. 
In fact, this linear correlation phenomenon can be viewed as a special case of our main discovery, Cross-Task Linearity (CTL),  when applied to the output layer. 
Our work indeed generalizes this phenomenon to the intermediate states at each layer and a broader setting of pretraining-finetuning paradigm (see full discussion in \cref{sec:model_averaging}). 
Therefore, our work is closely aligned with \citet{ilharco2022patching,wortsman2022model,Wortsman_2022_CVPR}.

Now let us turn to \citet{ortiz-jimenez2023task}. 
They demonstrated that the finetuned models cannot be accurately approximated by their post-hoc linearization and thus rejected the linearization hypothesis. 
The post-hoc linearization refers to the first-order Taylor expansion of the finetuned models at the pretrained checkpoint and the linearization hypothesis implies the finetuned model can be perfectly approximated by this linearization. 
However, they found that the performance of these post-hoc linearized models failed to match that of the original finetuned models, neither in single tasks nor in task arithmetic. 
The post-hoc linearization ensures the strict linearity of the finetuned models, which could directly explain both the linear correlation phenomenon and our main discovery, CTL. 
However, such linearization exhibits a significant discrepancy with the original finetuned models in terms of performance and thereby cannot fully explain the effectiveness of task arithmetic/model averaging.

In this work, we depart from the linearization hypothesis and instead focus on an approximate version of linearity. i.e, CTL. 
For the experiment side, we observe the approximate linearity, CTL, rather than strict linearity on the original finetuned model under standard experimental settings. 
For the theory side, we do not assume any linearization of finetuned models. 
Though such approximate linearity as well as the previous linear correlation phenomenon are surprising for neural networks, which are often viewed as highly non-linear functions, we take a preliminary attempt to prove that the approximate linearity can emerge when the product of the sharpness of the loss landscape and the squared Euclidean distance between two finetuned models’ weights is sufficiently small (see \cref{thm:linearity_emerge}). 
Therefore, our discovery not only establishes a strong connection with previous findings on the linear correlation phenomenon but also departs from the linearization hypothesis, turning to the approximate linearity.  

Moreover, our discovery, CTL, also connects to the Weight Disentanglement, proposed in section 4 of \citet{ortiz-jimenez2023task}, regarded as a pre-condition for finetuned models to perform task arithmetic. 
In \cref{thm:ctl_connects_wd}, we demonstrate that assuming the transitivity of CTL (see \cref{con:transitivity}), CTL can disentangle the edited models applied with multiple task vectors into a sum of localized components, each controlled by a single task vector.
Therefore, CTL closely relates to the Weight Disentanglement. 
In fact, the condition of the Weight Disentanglement is too idealistic to be satisfied in practice.
However, CTL serves as a relaxed version of the Weight Disentanglement, which still effectively explaining the task arithmetic.

In conclusion, our work indeed closely aligns with prior studies, not only \citet{ilharco2022patching,wortsman2022model,Wortsman_2022_CVPR} but also \citet{ortiz-jimenez2023task}. 
We depart from the linearization hypothesis and identify an approximate version of linearity, CTL, which generalizes the previous linear correlation phenomenon and still efficiently explains the effectiveness of the model merging/task arithmetic techniques. 

\section{Missing Proofs}\label{suppl:proofs}

\subsection{Preliminary Lemmas}\label{suppl:lemmas}

\begin{definition}[\textbf{Transitivity of CTL.}]\label{def:suppl_transitivity}
    Given models $\boldsymbol{\theta}_i, \boldsymbol{\theta}_j, \boldsymbol{\theta}_k$. We have $(\boldsymbol{\theta}_i, \boldsymbol{\theta}_k)$ satisfy CTL if $(\boldsymbol{\theta}_i, \boldsymbol{\theta}_j)$ and $(\boldsymbol{\theta}_j, \boldsymbol{\theta}_k)$ satisfy CTL.
\end{definition}

\begin{lemma}[\textbf{CTL holds for two-model weight interpolations.}]\label{lem:two_interpolation}
    Given two models $\boldsymbol{\theta}_i$ and $\boldsymbol{\theta}_j$ satisfy CTL, then $\forall \gamma \in [0, 1]$, we have $\boldsymbol{\theta}_i'=\gamma \boldsymbol{\theta}_i + (1-\gamma)\boldsymbol\theta_j$ and $\boldsymbol{\theta}_j$ satisfy CTL.
\end{lemma}

\begin{proof}
    $\forall \alpha, \gamma \in [0, 1], \forall \ell \in [L]$, 
    \begin{align*}
        f^{(\ell)}(\alpha \boldsymbol{\theta}_i' + (1-\alpha)\boldsymbol{\theta}_j)
        &= f^{(\ell)}(\alpha (\gamma \boldsymbol{\theta}_i + (1-\gamma)\boldsymbol\theta_j) + (1-\alpha)\boldsymbol{\theta}_j) \\
        &= f^{(\ell)}(\alpha \gamma \boldsymbol{\theta}_i + (1-\alpha\gamma)\boldsymbol{\theta}_j) \\
        &\approx \alpha\gamma f^{(\ell)}(\boldsymbol{\theta}_i) + (1-\alpha\gamma)f^{(\ell)}(\boldsymbol{\theta}_j) \\
        &\approx \alpha (\gamma f^{(\ell)}(\boldsymbol{\theta}_i) + (1-\gamma)f^{(\ell)}(\boldsymbol{\theta}_j)) + (1-\alpha)f^{(\ell)}(\boldsymbol{\theta}_j) \\
        &\approx \alpha f^{(\ell)}(\boldsymbol{\theta}_i') + (1-\alpha) f^{(\ell)} (\boldsymbol{\theta}_j)
    \end{align*}
    Therefore, $\boldsymbol{\theta}_i'$ and $\boldsymbol{\theta}_j$ satisfy CTL, and this finishes our proof.
\end{proof}

\begin{lemma}[\textbf{CTL holds for multi-model weight interpolations.}]\label{lem:multi_interpolation}
    Given a set of models $\Theta =\{\boldsymbol{\theta}_i\}_{i=1}^k$, suppose each pair of models $(\boldsymbol{\theta}_i, \boldsymbol{\theta}_j) \in \Theta^2$ satisfy CTL, assume transitivity of CTL (see \cref{def:suppl_transitivity}), then for any $\{\alpha_i\}_{i=1}^k \in [0, 1]$ subject to $\sum_{i=1}^k \alpha_i=1$, we have the weight-interpolated model $\boldsymbol{\theta}' = \sum_{i=1}^k \alpha_i \boldsymbol{\theta}_i$ satisfy CTL with each model from $\Theta$.
\end{lemma}

\begin{proof}
    We proceed by induction on $k$. 
    When $k=2$, \cref{lem:two_interpolation} directly implies \cref{lem:multi_interpolation}.

    Assume \cref{lem:multi_interpolation} holds for some $k' \geq 2$. 
    Assume that for any $\{\alpha_i\}_{i=1}^{k'} \in [0, 1]$ subject to  $\sum_{i=1}^{k'} \alpha_i=1$ and any $\boldsymbol{\theta}_j \in \Theta$, we have $\boldsymbol{\theta}' = \sum_{i=1}^{k'} \alpha_i \boldsymbol{\theta}_i$ and $\boldsymbol{\theta}_j$ satisfy CTL.

    Now we need to show \cref{lem:multi_interpolation} holds when $k=k'+1$. 
    That is, we need to show for any $\{\alpha_i\}_{i=1}^{k'+1} \in [0, 1]$ subject to  $\sum_{i=1}^{k'+1} \alpha_i=1$ and any $\boldsymbol{\theta}_j \in \Theta$, we have $\boldsymbol{\theta}' = \sum_{i=1}^{k'+1} \alpha_i \boldsymbol{\theta}_i$ and $\boldsymbol{\theta}_j$ satisfy CTL. 
    As \cref{lem:multi_interpolation} holds when $k=k'$, we have the weight-interpolated model $\boldsymbol{\theta}'' = \sum_{i\neq j} \frac{\alpha_i}{1-\alpha_j} \boldsymbol{\theta}_i$ satisfy CTL with each model $\boldsymbol{\theta}_p \in \Theta \setminus \{\boldsymbol\theta_j\}$. 
    As both $(\boldsymbol{\theta}'', \boldsymbol{\theta}_p)$ and $(\boldsymbol{\theta}_p, \boldsymbol{\theta}_j)$ satisfy CTL, we have $\boldsymbol{\theta}''$ and $\boldsymbol{\theta}_j$ satisfy CTL as well. 
    According to \cref{lem:two_interpolation}, the weight-interpolated model $\alpha_i \boldsymbol{\theta}_j + (1-\alpha_j) \boldsymbol{\theta}'' = \sum_{i=1}^{k'+1} \alpha_i \boldsymbol{\theta}_i$ and $\boldsymbol\theta_j$ satisfy CTL, and this finishes our proof.
\end{proof}

\subsection{Proof of \cref{thm:LLFC_induces_LMC}}\label{suppl:proof_of_thm1}

\begin{theorem}[\textbf{LLFC Induces LMC}]
    Given dataset $\mathcal{D}$, convex loss function $L$, and two modes ${\boldsymbol{\theta}}_i$ and ${\boldsymbol{\theta}}_j$ with equal loss on $\mathcal{D}$, i.e., $\mathcal{L}({\boldsymbol{\theta}}_i) = \mathcal{L}({\boldsymbol{\theta}}_j)$, suppose the two modes ${\boldsymbol{\theta}}_i$, ${\boldsymbol{\theta}}_j$ satisfy LLFC on $\mathcal D$ with exact equality, then for all $\alpha \in [0, 1]$, \begin{align*}
        \mathcal{L}(\alpha {\boldsymbol{\theta}}_i + (1-\alpha) {\boldsymbol{\theta}}_j) \leq \mathcal{L}(\boldsymbol{\theta}_i) = \mathcal{L}(\boldsymbol{\theta}_j).
    \end{align*}
\end{theorem}

\begin{proof}
    Since $\boldsymbol{\theta}_i$ and $\boldsymbol{\theta}_j$ satisfy LLFC on $\mathcal{D}$ with exact equality, we have \begin{align*}
        f\left(\alpha \boldsymbol{\theta}_i + \left(1-\alpha\right) \boldsymbol{\theta}_j; \boldsymbol{x}\right) = \alpha f\left(\boldsymbol{\theta}_i; \boldsymbol{x}\right) + \left(1-\alpha\right) f\left(\boldsymbol{\theta}_j; \boldsymbol{x}\right), \quad\forall \boldsymbol{x} \in \mathcal{D}.
    \end{align*}
    Since the loss function $L$ is convex to model outputs, we have \begin{align*}
        L\left(f\left(\alpha \boldsymbol{\theta}_i + \left(1-\alpha\right) \boldsymbol{\theta}_j; \boldsymbol{x}\right), y\right) 
        &= L\left(\alpha f\left(\boldsymbol{\theta}_i; \boldsymbol{x}\right) + \left(1-\alpha\right) f\left(\boldsymbol{\theta}_j; \boldsymbol{x}\right), y\right) \\
        &\leq \alpha L\left(f\left(\boldsymbol{\theta}_i; \boldsymbol{x}\right), y\right) + \left(1-\alpha\right) L\left(f\left(\boldsymbol{\theta}_j; \boldsymbol{x}\right), y\right), \quad \forall (\boldsymbol{x}, y) \in \mathcal{D}.
    \end{align*}
    Therefore, \begin{align*}
        \mathbb{E}_{(\boldsymbol{x}, y)\in \mathcal{D}}\left[L\left(f\left(\alpha \boldsymbol{\theta}_i + \left(1-\alpha\right) \boldsymbol{\theta}_j; \boldsymbol{x}\right), y\right)\right] &\leq \mathbb{E}_{(\boldsymbol{x}, y)\in \mathcal{D}}\left[\alpha L\left(f\left(\boldsymbol{\theta}_i; \boldsymbol{x}\right), y\right)\right] + \mathbb{E}_{(\boldsymbol{x}, y)\in \mathcal{D}}\left[\left(1-\alpha\right) L\left(f\left(\boldsymbol{\theta}_j; \boldsymbol{x}\right), y\right)\right] \\
        \mathcal{L}\left(\alpha {\boldsymbol{\theta}}_i + (1-\alpha) {\boldsymbol{\theta}}_j\right) &\leq \alpha \mathcal{L}(\boldsymbol{\theta}_i) + \left(1-\alpha\right)\mathcal{L}(\boldsymbol{\theta}_j)
    \end{align*}
    According to the condition that $\mathcal{L}(\boldsymbol{\theta}_i) = \mathcal{L}(\boldsymbol{\theta}_j)$, we finally obtain that \begin{align*}
        \mathcal{L}(\alpha {\boldsymbol{\theta}}_i + (1-\alpha) {\boldsymbol{\theta}}_j) \leq \mathcal{L}(\boldsymbol{\theta}_i) = \mathcal{L}(\boldsymbol{\theta}_j).
    \end{align*}
\end{proof}

\subsection{Proof of \cref{thm:linearity_multi}}\label{suppl:proof_of_thm2}

\begin{theorem}[\textbf{CTL Generalizes to Multiple Models}]\label{thm:suppl_linearity_multi}
    Given dataset $\mathcal{D}$ and a set of modes $\boldsymbol{\Theta}$ where each pair of modes $(\boldsymbol{\theta}_i, \boldsymbol{\theta}_j) \in \boldsymbol{\Theta}^2$ satisfy CTL on $\mathcal{D}$, assume transitivity of CTL (see \cref{def:suppl_transitivity}), then for any $\{\boldsymbol{\theta}_i\}_{i=1}^k \in \boldsymbol{\Theta}$ and $\{\alpha_i\}_{i=1}^k \in [0, 1]$, subject to the constraint that $\sum_{i=1}^{k} \alpha_{i} = 1$, we have
    \begin{align*}
        f^{(\ell)}\left(\sum_{i=1}^k \alpha_i \boldsymbol{\theta}_i\right) \approx \sum_{i=1}^k \alpha_i f^{(\ell)}({\boldsymbol{\theta}}_i), \quad \forall \ell \in [L].
    \end{align*}
\end{theorem}

\begin{proof}
    We proceed by induction on $k$.
    When $k = 2$, \cref{thm:suppl_linearity_multi} clearly holds. 
    For any $(\boldsymbol{\theta}_1, \boldsymbol{\theta}_2) \in \boldsymbol{\Theta}^2$, CTL holds on $\mathcal{D}$.
    Then, $\forall \alpha_1 \in [0, 1]$ and $\alpha_2 = 1-\alpha_1 \in [0,1]$, we have \begin{align*}
        f^{(\ell)}(\alpha_1 \boldsymbol{\theta}_1 + \alpha_2\boldsymbol{\theta}_2) 
        &= f^{(\ell)}\left(\sum_{i=1}^2 \alpha_i \boldsymbol{\theta}_i\right) \\
        &\approx \alpha_1 f^{(\ell)}(\boldsymbol{\theta}_1) + \alpha_2 f^{(\ell)}(\boldsymbol{\theta}_2)\\
        &\approx \sum_{i=1}^2 \alpha_i f^{(\ell)}(\boldsymbol{\theta}_i), \quad \forall \ell \in [L].
    \end{align*}
    
    Assume \cref{thm:suppl_linearity_multi} holds for some $k' \geq 2$. 
    That is, assume that for any $\{\boldsymbol{\theta}_i\}_{i=1}^{k'} \in \boldsymbol{\Theta}$ and $\{\alpha_i\}_{i=1}^{k'} \in [0, 1]$, subject to the constraint that $\sum_{i=1}^{k'} \alpha_{i} = 1$, we have
    \begin{align*}
        f^{(\ell)}\left(\sum_{i=1}^{k'} \alpha_i \boldsymbol{\theta}_i\right) \approx \sum_{i=1}^{k'} \alpha_i f^{(\ell)}({\boldsymbol{\theta}}_i), \quad \forall \ell \in [L].
    \end{align*}
    
    Now we need to show \cref{thm:suppl_linearity_multi} holds when $k = k'+1$. 
    For any set of modes $\{\boldsymbol{\theta}_i\}_{i=1}^{k'+1} \in \boldsymbol{\Theta}$ and any set of coefficients $\{\alpha_i\}_{i=1}^{k'+1} \in [0, 1]$, we define $\boldsymbol{\theta}_{avg, k'} = \sum_{i=1}^{k'}  \frac{\alpha_{i}}{1-\alpha_{k'+1}}\boldsymbol{\theta}_i$. 
    According to \cref{lem:multi_interpolation}, CTL holds for $\boldsymbol{\theta}_{avg, k'}$ and $\boldsymbol{\theta}_{k'+1}$, then we have \begin{align*}
        f^{(\ell)}(\alpha \boldsymbol{\theta}_{k'+1} + (1-\alpha)\boldsymbol{\theta}_{avg, k'}) \approx \alpha f^{(\ell)}(\boldsymbol{\theta}_{k'+1}) + (1-\alpha)f^{(\ell)}(\boldsymbol{\theta}_{avg, k'}), \quad \forall \alpha \in [0, 1], \forall \ell \in [L].
    \end{align*}
    Substituting $\alpha$ with $\alpha_{k'+1}$, we can obtain \begin{align*}
        f^{(\ell)}(\alpha_{k'+1} \boldsymbol{\theta}_{k'+1} + (1-\alpha_{k'+1})\boldsymbol{\theta}_{avg, k'}) 
        &= f^{(\ell)}\left(\sum_{i=1}^{k'+1} \alpha_i \boldsymbol{\theta}_i\right) \\
        &\approx \alpha_{k'+1} f^{(\ell)}(\boldsymbol{\theta}_{k'+1}) + (1-\alpha_{k'+1})f^{(\ell)}(\boldsymbol{\theta}_{avg, k'}), \quad \forall \alpha \in [0, 1], \forall \ell \in [L].
    \end{align*}
    Knowing that \cref{thm:suppl_linearity_multi} holds true when $k=k'$, then we have \begin{align*}
        f^{(\ell)}\left(\sum_{i=1}^{k'+1} \alpha_i \boldsymbol{\theta}_i\right)
        &\approx \alpha_{k'+1} f^{(\ell)}(\boldsymbol{\theta}_{k'+1}) + (1-\alpha_{k'+1})f^{(\ell)}\left(\sum_{i=1}^{k'}\frac{\alpha_{i}}{1-\alpha_{k'+1}}\boldsymbol{\theta}_i\right) \\ 
        &\approx \alpha_{k'+1} f^{(\ell)}(\boldsymbol{\theta}_{k'+1}) + (1-\alpha_{k'+1}) \sum_{i=1}^{k'}\frac{\alpha_{i}}{1-\alpha_{k'+1}}f^{(\ell)}(\boldsymbol{\theta}_i)  \\
        &\approx \sum_{i=1}^{k'+1} \alpha_i f^{(\ell)}({\boldsymbol{\theta}}_i), \quad \forall \ell \in [L].
    \end{align*}
    Therefore, \cref{thm:suppl_linearity_multi} holds true when $k=k'+1$, and this finishes our proof.
\end{proof}

\subsection{Proof of \cref{thm:ctl_connects_wd}}\label{suppl:proof_of_thm3}

\begin{theorem}[\textbf{CTL Connects to Weight Disentanglement.}]\label{thm:suppl_ctl_connects_wd}
    Given pretrained model $\boldsymbol{\theta}_{\rm PT}$ and a set of task vectors $\Upsilon=\{\tau_i\}_{i=1}^k$, suppose each pair of edited models $(\boldsymbol{\theta}_{\rm PT} + \lambda_i \tau_i, \boldsymbol{\theta}_{\rm PT} + \lambda_j \tau_j)$ satisfy CTL when $\lambda_i, \lambda_j \in [-\beta, \beta]$, assume the transitivity of CTL (see \cref{def:suppl_transitivity}), then $\forall \{\alpha_i\}_{i=1}^k \in [-\frac{\beta}{k+1}, \frac{\beta}{k+1}]$,
    \begin{align*}
        f\left(\boldsymbol{\theta}_{\rm PT} + \sum_{i=1}^k\alpha_i \tau_i; \boldsymbol{x}\right) \approx \sum_{i=1}^{k}g_i(\alpha_i \tau_i; \boldsymbol{x}) + g_0(\boldsymbol{x})
    \end{align*}
    where $g_i(\alpha_i \tau_i;\boldsymbol{x}) = \frac{1}{k+1}f(\boldsymbol{\theta}_{\rm PT} + (k+1)\alpha_i \tau_i; \boldsymbol{x})$ and $g_0(\boldsymbol{x}) = \frac{1}{k+1}f(\boldsymbol{\theta}_{\rm PT}; \boldsymbol{x})$. 
\end{theorem}

\begin{proof}
    We proceed by induction on $k$. 
    When $k=1$, above \cref{thm:suppl_ctl_connects_wd} clearly holds. 
    For any $\tau_1 \in \Upsilon$ and $\lambda_1 \in [-\beta, \beta]$, CTL holds for $(\boldsymbol{\theta}_{\rm PT}, \boldsymbol{\theta}_{\rm PT} + \lambda_1 \tau_1)$. Then, $\forall \alpha_1 \in [-\frac{\beta}{2}, \frac{\beta}{2}]$, we have 
    \begin{align*}
        f(\boldsymbol{\theta}_{\rm PT} + \alpha_1 \tau_1) \approx \frac{1}{2}f(\boldsymbol{\theta}_{\rm PT} + 2\alpha_1 \tau_1;\boldsymbol{x}) +  \frac{1}{2}f(\boldsymbol{\theta}_{\rm PT};\boldsymbol{x})
    \end{align*}

    Assume \cref{thm:suppl_ctl_connects_wd} holds for some $k' \geq 1$. 
    That is, assume that for any $\{\tau_i\}_{i=1}^{k'} \in \Upsilon$ and $\{\alpha_i\}_{i=1}^{k'} \in [-\frac{\beta}{k'+1}, \frac{\beta}{k'+1}]$, we have 
    \begin{align*}
        f\left(\boldsymbol{\theta}_{\rm PT} + \sum_{i=1}^{k'}\alpha_i \tau_i; \boldsymbol{x}\right) \approx \sum_{i=1}^{k'}\frac{1}{k'+1}f(\boldsymbol{\theta}_{\rm PT} + (k'+1)\alpha_i \tau_i; \boldsymbol{x}) + \frac{1}{k'+1}f(\boldsymbol{\theta}_{\rm PT}; \boldsymbol{x})
    \end{align*}
    Now we need to show \cref{thm:suppl_ctl_connects_wd} holds when $k=k'+1$. 
    According to \cref{lem:multi_interpolation}, it is clear to see that the edited model $\boldsymbol{\theta}_{\rm PT} + (k'+2)\alpha_{k'+1} \tau_{k'+1}$ and the weight-averaged model $\frac{1}{k'+1}\boldsymbol{\theta}_{\rm PT} +\frac{1}{k'+1}\sum_{i=1}^{k'} \boldsymbol{\theta}_{\rm PT} + (k'+2)\alpha_i \tau_i$ satisfy CTL. 
    Then, we have
    \begin{align*}
    &f\left(\boldsymbol{\theta}_{\rm PT} + \sum_{i=1}^{k'+1}\alpha_i \tau_i; \boldsymbol{x}\right)\\
    =& f\left(\frac{1}{k'+2}\left(\boldsymbol{\theta}_{\rm PT} + (k'+2)\alpha_{k'+1} \tau_{k'+1} \right)+ \frac{k'+1}{k'+2}\left(\frac{1}{k'+1}\boldsymbol{\theta}_{\rm PT} +\frac{1}{k'+1}\sum_{i=1}^{k'} \boldsymbol{\theta}_{\rm PT} + (k'+2)\alpha_i \tau_i\right); \boldsymbol{x}\right)\\
    \approx& \frac{1}{k'+2} f(\boldsymbol{\theta}_{\rm PT} + (k'+2)\alpha_{k'+1} \tau_{k'+1}; \boldsymbol{x}) + \frac{k'+1}{k'+2} f\left(\boldsymbol{\theta}_{\rm PT} +\sum_{i=1}^{k'} \frac{k'+2}{k'+1}\alpha_i \tau_i; \boldsymbol{x}\right) \\
    \approx& \frac{1}{k'+2} f(\boldsymbol{\theta}_{\rm PT} + (k'+2)\alpha_{k'+1} \tau_{k'+1}; \boldsymbol{x}) + \frac{k'+1}{k'+2} \left(\sum_{i=1}^{k'}\frac{1}{k'+1}f(\boldsymbol{\theta}_{\rm PT} + (k'+2)\alpha_i \tau_i; \boldsymbol{x}) + \frac{1}{k'+1}f(\boldsymbol{\theta}_{\rm PT}; \boldsymbol{x})\right) \\
    \approx& \sum_{i=1}^{k'+1}\frac{1}{k'+2}f(\boldsymbol{\theta}_{\rm PT} + (k'+2)\alpha_i \tau_i; \boldsymbol{x}) + \frac{1}{k'+2}f(\boldsymbol{\theta}_{\rm PT}; \boldsymbol{x})
    \end{align*}
    Therefore, \cref{thm:suppl_ctl_connects_wd} holds when $k=k'+1$, and this finishes our proof.
\end{proof}

\subsection{Proof of \cref{thm:linearity_emerge}}\label{suppl:proof_of_thm4}

\begin{theorem}[\textbf{The Emergence of CTL}]
    Suppose $f(\boldsymbol{\theta}): \R^{p} \mapsto \R$ is third differentiable in an open convex set $\boldsymbol{\Theta}$ and the its Hessian norm at $\boldsymbol{\theta}_0$ is bounded by $\lambda_{\text{min}} \leq \norm{\nabla^2 f(\boldsymbol{\theta}_0)} \leq \lambda_{\text{max}}$, then
\begin{align*}
    \delta_{i,j}=\abs{f(\alpha \boldsymbol{\theta}_i + (1-\alpha) \boldsymbol{\theta}_j) - \alpha f(\boldsymbol{\theta}_i) - (1-\alpha) f(\boldsymbol{\theta}_j)} \leq \frac{\alpha (1-\alpha) \lambda_{\text{max}}}{2}\norm{\boldsymbol{\theta}_j - \boldsymbol{\theta}_i}^2
    + \mathcal{E} ,
\end{align*}
where $\mathcal{E}=O\parenth{\max \parenth{\norm{ \alpha \boldsymbol{\theta}_i + (1-\alpha) \boldsymbol{\theta}_j - \boldsymbol{\theta}_0}^3, \alpha \norm{\boldsymbol{\theta}_i - \boldsymbol{\theta}_0}^3, (1-\alpha) \norm{\boldsymbol{\theta}_j - \boldsymbol{\theta}_0}^3}}$ is the high order term.
\end{theorem}
\begin{proof}
    Since $f$ is third differentiable in an open convex set $\boldsymbol{\Theta}$, then by Taylor’s Theorem, for any $\boldsymbol{\theta}_0, \boldsymbol{\theta} \in \boldsymbol{\Theta}$,
    \begin{gather*}
        f(\boldsymbol{\theta}) = f(\boldsymbol{\theta}_0) + \nabla f(\boldsymbol{\theta}_0)^\top \parenth{\boldsymbol{\theta} - \boldsymbol{\theta}_0} + \frac{1}{2}\parenth{\boldsymbol{\theta} - \boldsymbol{\theta}_0}^\top \nabla^2 f(\boldsymbol{\theta}_0) \parenth{\boldsymbol{\theta} - \boldsymbol{\theta}_0} + R_{\boldsymbol{\theta}_0, 3}(\boldsymbol{\theta} - \boldsymbol{\theta}_0) .
    \end{gather*}
    where the remainder term $R_{\boldsymbol{\theta}_0, 3}(\boldsymbol{\theta} - \boldsymbol{\theta}_0) = O \parenth{\norm{\boldsymbol{\theta} - \boldsymbol{\theta}_0}^3}$. Suppose $\lambda_{\text{min}} \leq \norm{\nabla^2 f(\boldsymbol{\theta}_0)} \leq \lambda_{\text{max}}$, we have
    \begin{equation*}
        \frac{\lambda_{\text{min}}}{2} \norm{\boldsymbol{\theta} - \boldsymbol{\theta}_0}^2 
        \leq \frac{1}{2}\parenth{\boldsymbol{\theta} - \boldsymbol{\theta}_0}^\top \nabla^2 f(\boldsymbol{\theta}_0) \parenth{\boldsymbol{\theta} - \boldsymbol{\theta}_0} 
        \leq \frac{\lambda_{\text{max}}}{2} \norm{\boldsymbol{\theta} - \boldsymbol{\theta}_0}^2 .
    \end{equation*}

    Then for $f(\alpha \boldsymbol{\theta}_i + (1-\alpha) \boldsymbol{\theta}_j) $, by Taylor’s Theorem,
    \begin{align*}
        &f(\alpha \boldsymbol{\theta}_i + (1-\alpha) \boldsymbol{\theta}_j) \\
        &= f(\boldsymbol{\theta}_0) + \nabla f(\boldsymbol{\theta}_0)^\top \parenth{\alpha \boldsymbol{\theta}_i + (1-\alpha) \boldsymbol{\theta}_j - \boldsymbol{\theta}_0}  \\
        &\quad+ \frac{1}{2}\parenth{\alpha \boldsymbol{\theta}_i+ (1-\alpha) \boldsymbol{\theta}_j - \boldsymbol{\theta}_0}^\top \nabla^2 f(\boldsymbol{\theta}_0) \parenth{\alpha \boldsymbol{\theta}_i + (1-\alpha) \boldsymbol{\theta}_j - \boldsymbol{\theta}_0} + R_{\boldsymbol{\theta}_0, 3}(\alpha \boldsymbol{\theta}_i + (1-\alpha) \boldsymbol{\theta}_j - \boldsymbol{\theta}_0)\\
        &= f(\boldsymbol{\theta}_0) + \nabla f(\boldsymbol{\theta}_0)^\top \parenth{\alpha \parenth{\boldsymbol{\theta}_i - \boldsymbol{\theta}_0} + (1-\alpha) \parenth{\boldsymbol{\theta}_j - \boldsymbol{\theta}_0}}  \\
        &\quad+ \frac{1}{2}\parenth{\alpha \boldsymbol{\theta}_i+ (1-\alpha) \boldsymbol{\theta}_j - \boldsymbol{\theta}_0}^\top \nabla^2 f(\boldsymbol{\theta}_0) \parenth{\alpha \boldsymbol{\theta}_i + (1-\alpha) \boldsymbol{\theta}_j - \boldsymbol{\theta}_0} + R_{\boldsymbol{\theta}_0, 3}(\alpha \boldsymbol{\theta}_i + (1-\alpha) \boldsymbol{\theta}_j - \boldsymbol{\theta}_0)\\
        &= \alpha \parenth{f(\boldsymbol{\theta}_0) + \nabla f(\boldsymbol{\theta}_0)^\top \parenth{\boldsymbol{\theta}_i - \boldsymbol{\theta}_0}} + (1-\alpha) \parenth{f(\boldsymbol{\theta}_0) + \nabla f(\boldsymbol{\theta}_0)^\top \parenth{\boldsymbol{\theta}_i - \boldsymbol{\theta}_0}} \\
        &\quad+ \frac{1}{2}\parenth{\alpha \boldsymbol{\theta}_i+ (1-\alpha) \boldsymbol{\theta}_j - \boldsymbol{\theta}_0}^\top \nabla^2 f(\boldsymbol{\theta}_0) \parenth{\alpha \boldsymbol{\theta}_i + (1-\alpha) \boldsymbol{\theta}_j - \boldsymbol{\theta}_0} + R_{\boldsymbol{\theta}_0, 3}(\alpha \boldsymbol{\theta}_i + (1-\alpha) \boldsymbol{\theta}_j - \boldsymbol{\theta}_0)\\
        &= \alpha \parenth{f(\boldsymbol{\theta}_i) - \frac{1}{2}\parenth{\boldsymbol{\theta}_i - \boldsymbol{\theta}_0}^\top \nabla^2 f(\boldsymbol{\theta}_0) \parenth{\boldsymbol{\theta}_i - \boldsymbol{\theta}_0} - R_{\boldsymbol{\theta}_0, 3}(\boldsymbol{\theta}_i - \boldsymbol{\theta}_0)}  \\
        &\quad + (1-\alpha) \parenth{f(\boldsymbol{\theta}_j) - \frac{1}{2}\parenth{\boldsymbol{\theta}_j - \boldsymbol{\theta}_0}^\top \nabla^2 f(\boldsymbol{\theta}_0) \parenth{\boldsymbol{\theta}_j - \boldsymbol{\theta}_0} - R_{\boldsymbol{\theta}_0, 3}(\boldsymbol{\theta}_j - \boldsymbol{\theta}_0)} \\
        &\quad+ \frac{1}{2}\parenth{\alpha \boldsymbol{\theta}_i+ (1-\alpha) \boldsymbol{\theta}_j - \boldsymbol{\theta}_0}^\top \nabla^2 f(\boldsymbol{\theta}_0) \parenth{\alpha \boldsymbol{\theta}_i + (1-\alpha) \boldsymbol{\theta}_j - \boldsymbol{\theta}_0} + R_{\boldsymbol{\theta}_0, 3}(\alpha \boldsymbol{\theta}_i + (1-\alpha) \boldsymbol{\theta}_j - \boldsymbol{\theta}_0)\\
        &= \alpha f(\boldsymbol{\theta}_i) + (1-\alpha) f(\boldsymbol{\theta}_j) \\
        &\quad+ \frac{1}{2}\parenth{\alpha \parenth{\boldsymbol{\theta}_i -  \boldsymbol{\theta}_0} + (1-\alpha)\parenth{\boldsymbol{\theta}_j - \boldsymbol{\theta}_0}}^\top \nabla^2 f(\boldsymbol{\theta}_0) \parenth{\alpha \parenth{\boldsymbol{\theta}_i -  \boldsymbol{\theta}_0} + (1-\alpha)\parenth{\boldsymbol{\theta}_j - \boldsymbol{\theta}_0}} \\
        &\quad- \frac{\alpha}{2}\parenth{\boldsymbol{\theta}_i - \boldsymbol{\theta}_0}^\top \nabla^2 f(\boldsymbol{\theta}_0) \parenth{\boldsymbol{\theta}_i - \boldsymbol{\theta}_0} 
        - \frac{(1-\alpha)}{2}\parenth{\boldsymbol{\theta}_j - \boldsymbol{\theta}_0}^\top \nabla^2 f(\boldsymbol{\theta}_0) \parenth{\boldsymbol{\theta}_j - \boldsymbol{\theta}_0} \\
        &\quad+ R_{\boldsymbol{\theta}_0, 3}(\alpha \boldsymbol{\theta}_i + (1-\alpha) \boldsymbol{\theta}_j - \boldsymbol{\theta}_0) - \alpha R_{\boldsymbol{\theta}_0, 3}(\boldsymbol{\theta}_i - \boldsymbol{\theta}_0) - (1-\alpha)R_{\boldsymbol{\theta}_0, 3}(\boldsymbol{\theta}_j - \boldsymbol{\theta}_0) \\
        &= \alpha f(\boldsymbol{\theta}_i) + (1-\alpha) f(\boldsymbol{\theta}_j) \\
        &\quad- \frac{\alpha (1-\alpha)}{2}\parenth{\boldsymbol{\theta}_j - \boldsymbol{\theta}_i}^\top \nabla^2 f(\boldsymbol{\theta}_0) \parenth{\boldsymbol{\theta}_j - \boldsymbol{\theta}_i} \\
        &\quad+ R_{\boldsymbol{\theta}_0, 3}(\alpha \boldsymbol{\theta}_i + (1-\alpha) \boldsymbol{\theta}_j - \boldsymbol{\theta}_0) - \alpha R_{\boldsymbol{\theta}_0, 3}(\boldsymbol{\theta}_i - \boldsymbol{\theta}_0) - (1-\alpha)R_{\boldsymbol{\theta}_0, 3}(\boldsymbol{\theta}_j - \boldsymbol{\theta}_0).
    \end{align*}
Therefore, we have
\begin{align*}
    &\quad \abs{f(\alpha \boldsymbol{\theta}_i + (1-\alpha) \boldsymbol{\theta}_j) - \alpha f(\boldsymbol{\theta}_i) - (1-\alpha) f(\boldsymbol{\theta}_j)} \\
    &= {\small \abs{- \frac{\alpha (1-\alpha)}{2}\parenth{\boldsymbol{\theta}_j - \boldsymbol{\theta}_i}^\top \nabla^2 f(\boldsymbol{\theta}_0) \parenth{\boldsymbol{\theta}_j - \boldsymbol{\theta}_i} 
    + R_{\boldsymbol{\theta}_0, 3}(\alpha \boldsymbol{\theta}_i + (1-\alpha) \boldsymbol{\theta}_j - \boldsymbol{\theta}_0) - \alpha R_{\boldsymbol{\theta}_0, 3}(\boldsymbol{\theta}_i - \boldsymbol{\theta}_0) - (1-\alpha)R_{\boldsymbol{\theta}_0, 3}(\boldsymbol{\theta}_j - \boldsymbol{\theta}_0) }} \\
    &\leq {\small \abs{\frac{\alpha (1-\alpha)}{2}\parenth{\boldsymbol{\theta}_j - \boldsymbol{\theta}_i}^\top \nabla^2 f(\boldsymbol{\theta}_0) \parenth{\boldsymbol{\theta}_j - \boldsymbol{\theta}_i} }
    + \abs{R_{\boldsymbol{\theta}_0, 3}(\alpha \boldsymbol{\theta}_i + (1-\alpha) \boldsymbol{\theta}_j - \boldsymbol{\theta}_0) - \alpha R_{\boldsymbol{\theta}_0, 3}(\boldsymbol{\theta}_i - \boldsymbol{\theta}_0) - (1-\alpha)R_{\boldsymbol{\theta}_0, 3}(\boldsymbol{\theta}_j - \boldsymbol{\theta}_0) }} \\
    &\leq \frac{\alpha (1-\alpha) \lambda_{\text{max}}}{2}\norm{\boldsymbol{\theta}_j - \boldsymbol{\theta}_i}^2
    + O\parenth{\max \parenth{\norm{ \alpha \boldsymbol{\theta}_i + (1-\alpha) \boldsymbol{\theta}_j - \boldsymbol{\theta}_0}^3, \alpha \norm{\boldsymbol{\theta}_i - \boldsymbol{\theta}_0}^3, (1-\alpha) \norm{\boldsymbol{\theta}_j - \boldsymbol{\theta}_0}^3}} .
\end{align*}

where the last inequality is because $\lambda_{\text{min}} \leq \norm{\nabla^2 f(\boldsymbol{\theta}_0)} \leq \lambda_{\text{max}}$ and $R_{\boldsymbol{\theta}_0, 3}(\boldsymbol{\theta} - \boldsymbol{\theta}_0) = O \parenth{\norm{\boldsymbol{\theta} - \boldsymbol{\theta}_0}^3}$.

\end{proof}

\section{More Experimental Results}\label{suppl:results}
\subsection{Detailed Experimental Settings} \label{suppl:settings}

\subsubsection{Experimental Settings in \cref{sec:LLFC_to_linearity}}
\label{app:section4_1}
\textbf{Multi-Layer Perceptron on the Rotated MNIST Dataset.}\\
Following the settings outlined by \citet{mirzadeh2021linear}, we adopt the multi-layer perceptron with two hidden layers with $100$ units for Rotated MNIST dataset.
ReLU activation functions are adopted between linear layers. 
Therefore, the multi-layer perceptron has $4$ linear layers ($1$ for input, $2$ for hidden and $1$ for output) and $3$ ReLU layers.
We pretrain the MLP on normal MNIST and finetune it on Rotated MNIST, where each digit are rotated by a specific angle.
We use rotation angle degrees of $\{0^{\circ},22.5^{\circ}, 45^{\circ}, 67.5^{\circ}, 90^{\circ}\}$. 
Optimization is done with the default SGD algorithm and the learning rate of $1 \times 10^{-1}$, the batch size is set to $64$ and the training epoch is set to $1$ for both pretraining and finetuning. 

We have $4$ finetuned MLPs, yielding $6$ non-repeated combinations of two finetuned models $(\boldsymbol{\theta}_i, \boldsymbol{\theta}_j)$ in total. 
CTL are evaluated for each combinations on the union of their finetuning tasks ($\mathcal{D}_i \cup \mathcal{D}_j$) which have $20,000$ test samples.

\textbf{ResNet-18 on the Split CIFAR-100 Dataset.}\\
Still following the settings outlined by \citet{mirzadeh2021linear}, we adopt the ResNet-18 architecture \cite{kaiming2016residual} on the Split CIFAR-100 dataset.
The Split CIFAR-100 dataset is divided by classes, and $5$ consecutive categories of CIFAR-100 are grouped into one split, having $20$ splits in total.
We use the first split as pretraining task and the second to fifth splits as finetuning tasks.
We pretrain the ResNet-18 on first split and finetune it on the rest $4$ splits respectively, acquiring 4 finetuned ResNet-18 checkpoints. 
No data augmentation techniques are adopted and optimization is done using the default SGD algorithm with learning rate of $5 \times 10^{-2}$.  
The batch size is set to $64$.
The training epoch is set to 10 for both pretraining and finetuning. 

Similar to the setup of the Rotated MNIST experiment, we have $4$ finetuned ResNet-18 models, yielding $6$ non-repeated combinations of two finetuned models $(\boldsymbol{\theta}_i, \boldsymbol{\theta}_j)$ in total.
CTL are evaluated for each combinations and on the union of their finetuning tasks ($\mathcal{D}_i \cup \mathcal{D}_j$) which have $20,000$ test samples.

\subsubsection{Experimental Settings in \cref{sec:model_averaging}}
\textbf{Model Averaging Accuracy v.s. Logits Ensemble Accuracy. } \\
We choose $20$ out of the $72$ ViT-B/32~\cite{dosovitskiy2020image} checkpoints that are finetuned on ImageNet ~\cite{imagenet} and open-sourced by \citet{wortsman2022model}, yielding $\binom{20}{3} = 1140$ non-repeated combinations of three finetuned ViT-B/32 models.
For each combination of the finetuned models, we evaluated model averaging accuracy and logits ensemble accuracy on $10,000$ test samples from ImageNet. 

\textbf{Verification of \cref{eq:linearity_multi}} \\
For ViT-B/32 on CIFAR-10, We train our ViT-B/32 initialized from same CLIP pretrained checkpoint but finetuned on CIFAR-10 dataset with different hyper-parameters to obtain $5$ checkpoints to validate \cref{eq:linearity_multi}. 
For ViT-B/32 on ImageNet, we choose $10$ out of the $72$ ViT-B/32 checkpoints that are finetuned on ImageNet and open-sourced by \citet{wortsman2022model} to validate \cref{eq:linearity_multi}. 
For both cases, we perform experiments on randomly-selected $10,000$ samples from the test set.

It's worth mentioning that in the forward pass of ViT models, the input in the shape of \texttt{(batch\_size, patches\_num, hidden\_dim)} will be permuted to \texttt{(patches\_num, batch\_size, hidden\_dim)}. We permute the internal feature back and reshape it into \texttt{(batch\_size, patches\_num $\times$ hidden\_dim)}.
Now, the dimension of the features is simply \texttt{patches\_num $\times$ hidden\_dim}.

\subsubsection{Experimental Settings in \cref{sec:task_arithmetic}}
\textbf{CTL Explains Learning via Addition.} \\
We present the experimental settings in (\textit{i}) Cross-Task Linearity (CTL) and (\textit{ii}) Model Stitching experiment, respectively. 

\textbf{\textit{i}) Cross-Task Linearity (CTL) experiment.} \\
\textbf{Vision Transformer}~\cite{dosovitskiy2020image}: we evaluate ViT-B/32 and ViT-L/14 on $8$ image classification datasets: Cars~\cite{Krause20133DOR}, DTD~\cite{cimpoi2014describing}, EuroSAT~\cite{helber2019eurosat}, GTSRB~\cite{stallkamp2011german}, MNIST~\cite{LeCun2005TheMD}, RESISC45~\cite{cheng2017remote}, SUN397~\cite{xiao2016sun}, SVHN~\cite{netzer2011reading}). $8$ finetuned ViT-B/32 (ViT-L/14) models generate $\binom{8}{2}=28$ non-repeated combinations of two task vectors in total.
For each combination of the two task vectors, we validate \cref{eq:linearity_addi} on on the union of their finetuning datasets ($\mathcal{D}_i \cup \mathcal{D}_j$) which has $10,000$ test samples in total.

\textbf{T5}~\cite{raffel2020exploring}: we evaluate T5 on $6$ NLP datasets: IMDB~\cite{maas2011learning}, RACE~\cite{lai2017race}, QASC~\cite{khot2020qasc}, MultiNews~\cite{fabbri2019multi}, SQuAD~\cite{rajpurkar2016squad}, CommonGen~\cite{lin2019commongen}, as same setup in \citet{ilharco2023editing}. $6$ finetuned T5-base models generate $\binom{6}{2}=15$ non-repeated combinations  of two task vectors in total.
For each combination of the two task vectors, we validate \cref{eq:linearity_addi} on on the union of their finetuning datasets ($\mathcal{D}_i \cup \mathcal{D}_j$) which has $1,000$ test samples in total. As T5  is a encoder-decoder architecture and sentences are varied in their lengths, we adopt the convention in sentence-T5 \cite{ni2021sentence}, which uses (i) the average pooling of tokens in the encoder to represent the internal feature of a sentence and (ii) the decoder's hidden states when generating first token (which is equivalent to attention pooling) to represent the feature of a sentence in decoder.

\textbf{\textit{ii}) Model Stitching experiment.}\\ 
We only validate ViT architectures (ViT-B/32, ViT-L/14) on the aforementioned $8$ image classification datasets. We follow the \textbf{Cross-Task Linearity (CTL) experiment} settings except for the evaluation data size, which is of $2,000$ in this case. Notably, it is impossible for us to include the results for all $28$ combinations, and thus, part of our experimental results will be presented, which is the same for the other experiments.

\textbf{CTL Explains Learning via Negation.} \\
Similar to the \textbf{Learning via Addition} setup, the datasets and architectures are kept the same.

\textbf{\textit{i}) Cross-Task Linearity (CTL) experiment.} \\
We evaluate both ViT and T5 architectures. For ViT architectures, we evaluate the $8$ finetuned models on their corresponding finetuned datasets, each having $10,000$ test samples. For T5 architectures, we evaluate the $6$ finetuned models on their downstream finetuned datasets, each having $10,000$ test samples.

\textbf{\textit{ii}) Model Stitching experiment.}\\ 
In Model Stitching experiment, we only validate ViT (ViT-B/32, ViT-L/14) architectures. We follow the same settings as above except for evaluation data size, which is of $2,000$ in this case.

\subsubsection{Experimental Settings in \cref{sec:linearity_emerge}}
\textbf{\textit{i)} The number of pretraining/finetuning epochs.} We adopt the \textbf{ResNet-18 on the Split CIFAR-100 Dataset} setting in \cref{app:section4_1}. For validating the impact of pretraining epochs, we vary the number of pretraining epochs from 0 to 20 and fix the number of finetuning epochs to 10. For validating the impact of finetuning epochs, we fix the number of pretraining epochs to 10 and vary the number of finetuning epochs from 0 to 20. We use the combination of $\mathcal{D}_1$ and $\mathcal{D}_2$.

\textbf{\textit{ii)} The task similarity.} 
We use the Split ImageNet-1k~\cite{imagenet} instead of Split CIFAR-100 as pretraining and finetuning datasets in practice to confidentially make sure the differences between datasets significant. Similar to the setting of \textbf{ResNet-18 on the Split CIFAR-100 Dataset} setting in \cref{app:section4_1}, the Split ImageNet-1k dataset is divided by classes, and 10 consecutive categories are grouped into one split. We use the first split as pretraining datasets and the second/third split as $\mathcal{D}_1$/$\mathcal{D}_2$. We use the default training hyper-parameters in torchvision~\cite{torchvision2016} to pretrain/finetune the ResNet-18. 

\subsection{Verification of \cref{thm:linearity_emerge}}\label{suppl:exp_verify_thm}
In this section, we conduct experiments to validate the theoretical analysis presented in \cref{thm:linearity_emerge}. 
Specifically, we demonstrate that $\delta_{i,j}$ exhibits a stronger correlation with $\frac{\alpha (1-\alpha) \lambda_{\text{max}}}{2}\norm{\boldsymbol{\theta}_i - \boldsymbol{\theta}_j}^2$ compared to $\lambda_{\text{max}}$ or $\norm{\boldsymbol{\theta}_i - \boldsymbol{\theta}_j}^2$ alone.

For each pair of $\boldsymbol{\theta}_i$ and $\boldsymbol{\theta}_j$, we calculate the distance between finetuned models, i.e., $\norm{\boldsymbol{\theta}_i - \boldsymbol{\theta}_j}^2$, and $\delta_{i,j}$, i.e., $|f(\alpha \boldsymbol{\theta}_i + (1-\alpha)\boldsymbol{\theta}_j) - \alpha f(\boldsymbol{\theta}_i) - (1-\alpha)f(\boldsymbol{\theta}_j)|$ where $\alpha = 0.5$.
We also compute the largest eigenvalue of the Hessian matrix of $f(\cdot)$ at $\boldsymbol{\theta}_0$, i.e., $\lambda_{\rm max}$.
Here, $\boldsymbol{\theta}_i$ and $\boldsymbol{\theta}_j$ denote models that are initialized from a common checkpoint and finetuned on the same dataset with different hyperparameters. 
The function $f(\cdot)$ represents the loss function $\mathcal{L}(\cdot)$, and $\boldsymbol{\theta}_0$ is simply chosen as $\boldsymbol{\theta}_{\rm PT}$.

For ResNet-20 models finetuned on the CIFAR-10 dataset, we find that if we use $\frac{\alpha (1-\alpha) \lambda_{\text{max}}}{2}\norm{\boldsymbol{\theta}_i - \boldsymbol{\theta}_j}^2$ to fit a regression model to predict $\delta_{i,j}$, the R-squared value of the model is approximately $0.903$. 
However, if we use only $\lambda_{\text{max}}$ or $\norm{\boldsymbol{\theta}_i - \boldsymbol{\theta}_j}^2$ to fit the regression model, the R-squared value of the model is approximately $0.782$ or $0.839$, respectively. 
Therefore, we conclude that $\delta_{i,j}$ indeed demonstrates a strong correlation with $\frac{\alpha (1-\alpha) \lambda_{\text{max}}}{2}\norm{\boldsymbol{\theta}_i - \boldsymbol{\theta}_j}^2$. 
It is worth noting that such correlation is a joint effect of both $\lambda_{\text{max}}$ and $\norm{\boldsymbol{\theta}_i - \boldsymbol{\theta}_j}^2$, implying that either reducing $\lambda_{\text{max}}$ or $\norm{\boldsymbol{\theta}_i - \boldsymbol{\theta}_j}^2$ leads to the fulfillment of CTL.

\subsection{More Verification of CTL}\label{suppl:exp_CTL}
In this section, we provide more experimental results about the CTL on MLP and ResNet-18 in different task combinations and different layers, which shows the CTL holds in the pretraining-finetuning paradigm. 
In \cref{fig:LLFC_appendix_mlp_ctl} and \cref{{fig:LLFC_appendix_mlp_coef}}, we include experimental results of $\mathbb{E}_{\mathcal{D}}[1-{\rm cosine}_{\alpha}^{(\ell)}(\boldsymbol{x})]$ and $\text{coef}_{\alpha}^{(\ell)}(\boldsymbol{x})$ for MLPs on Rotated MNIST dataset. 
In \cref{fig:LLFC_appendix_resnet18_ctl} and \cref{{fig:LLFC_appendix_resnet18_coef}}, we include experimental results of $\mathbb{E}_{\mathcal{D}}[1-{\rm cosine}_{\alpha}^{(\ell)}(\boldsymbol{x})]$ and $\text{coef}_{\alpha}^{(\ell)}(\boldsymbol{x})$ for ResNet-18 on Split CIFAR-100 dataset. 

\subsection{More Verification of \cref{eq:linearity_multi}}\label{suppl:exp_model_avg}
In this section, we provide more experimental results about the CTL in model averaging on ViT-B/32 to validate \cref{eq:linearity_multi}.
In \cref{fig:LLFC_appendix_model_avg_cifar10}, we include the $\mathbb{E}_{\mathcal{D}}[1-{\rm cosine}_{avg}^{(\ell)}(\boldsymbol{x})]$ and ${\rm coef}_{avg}^{(\ell)}(\boldsymbol{x})$ for ViT-B/32 on CIFAR-10. 
In \cref{fig:LLFC_appendix_model_avg_imagenet}, we include the $\mathbb{E}_{\mathcal{D}}[1-{\rm cosine}_{avg}^{(\ell)}(\boldsymbol{x})]$ and ${\rm coef}_{avg}^{(\ell)}(\boldsymbol{x})$ for ViT-B/32 on ImageNet.
Results are reported across all blocks of ViT-B/32.

\subsection{More Verification of \cref{eq:linearity_addi}}\label{suppl:exp_addi}
In this section, we provide more results about the CTL in task arithmetic on ViT-B/32, ViT-L/14 , T5 and Llama-2-13B architectures to validate \cref{eq:linearity_addi}.
We report both $\mathbb{E}_{\mathcal{D}}[1-{\rm cosine}_{arith}^{(\ell)}(\boldsymbol{x}; \lambda \tau_i, \lambda \tau_j)]$ and ${\rm coef}_{arith}^{(\ell)}(\boldsymbol{x}; \lambda \tau_i, \lambda \tau_j)$.
In \cref{fig:LLFC_appendix_linearity_add_vit_b_32}, we provide more results for ViT-B/32 architecture and all the blocks of ViT-B/32 are reported. 
In \cref{fig:LLFC_appendix_linearity_add_vit_l_14}, we provide more results for ViT-L/14 architecture and the last $12$ blocks of ViT-L/14 are reported.  
In \cref{fig:LLFC_appendix_linearity_add_t5}, we provide more results on T5 architecture and the last $6$ encoder blocks and last $6$ decoder blocks of T5 are reported.

\subsection{More Verification of \cref{eq:linearity_neg}}\label{suppl:exp_neg}
In this section, we provide results about the CTL in task arithmetic for more task vectors of ViT-B/32, ViT-L/14 and T5 architectures to validate \cref{eq:linearity_neg}.
We report both $\mathbb{E}_{\mathcal{D}}[1-{\rm cosine}_{arith}^{(\ell)}(\boldsymbol{x}; \lambda \tau_i, -\lambda \tau_i)]$ and ${\rm coef}_{arith}^{(\ell)}(\boldsymbol{x}; \lambda \tau_i, -\lambda \tau_i)$. 
In \cref{fig:LLFC_appendix_linearity_neg_vit_b_32}, we provide more results for ViT-B/32 architecture and all the blocks of ViT-B/32 are reported. 
In \cref{fig:LLFC_appendix_linearity_neg_vit_l_14}, we provide more results for ViT-L/14 architecture and last 12 blocks of ViT-L/14 are reported. 
In \cref{fig:LLFC_appendix_linearity_neg_t5}, we provide more results for T5 architecture and the last 6 encoder blocks and last 6 decoder blocks of T5 are reported.

\begin{figure*}[tb!]
  \begin{center}
    \includegraphics[width=0.9928\textwidth]{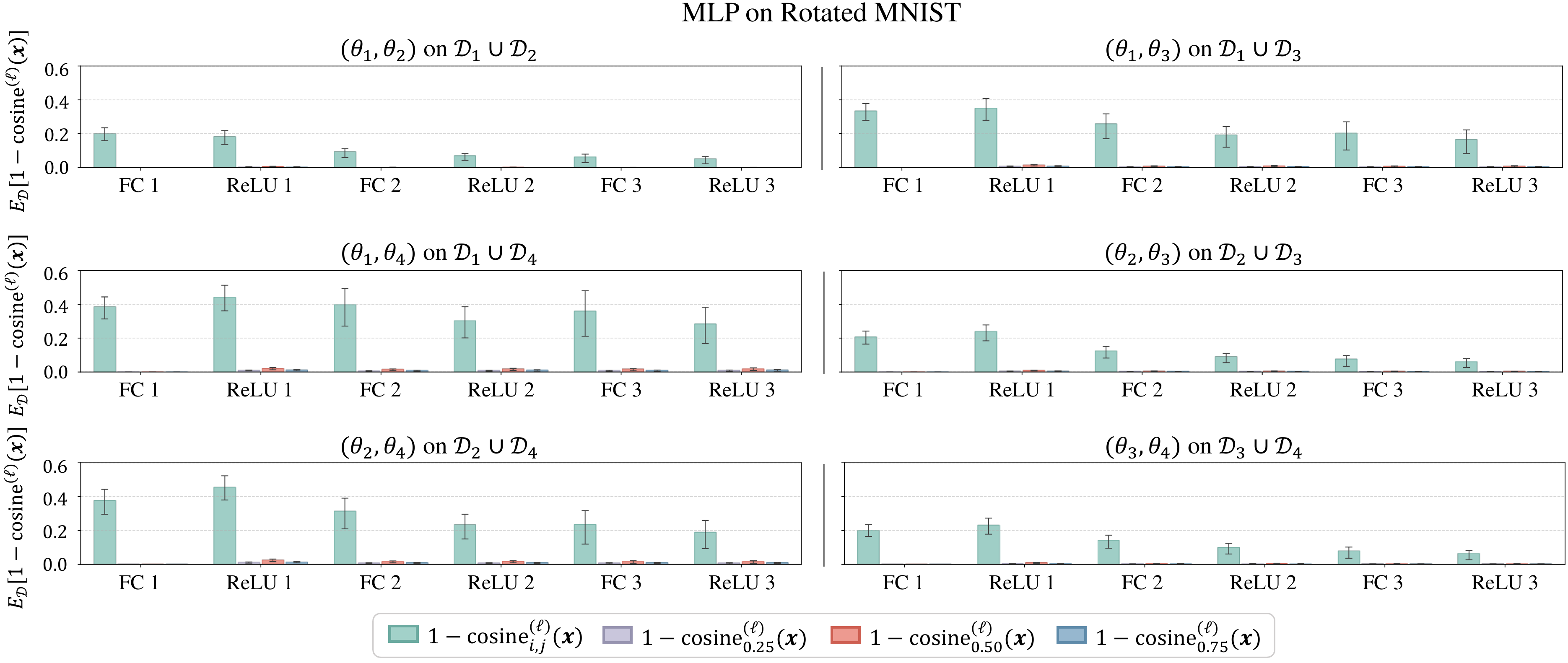}
    \vspace{-5pt}
    \caption{
    Verification of CTL. Compare $\mathbb{E}_{\mathcal{D}}[1-{\rm cosine}_{\alpha}^{(\ell)}(\boldsymbol{x})]$ with $\mathbb{E}_{\mathcal{D}}[1-\text{cosine}_{i,j}^{(\ell)}(\boldsymbol{x})]$.
    Here, $\{\boldsymbol{\theta}_i\}_{i=1}^4$ and $\{\mathcal{D}_i\}_{i=1}^4$ denotes four finetuned MLPs on corresponding Rotated MNIST with rotation $\in \{ 22.5^{\circ}, 45^{\circ}, 67.5^{\circ}, 90^{\circ} \}$ respectively. 
    The results are reported for all layers of finetuned MLP, with $\alpha \in \{0.25, 0.5, 0.75\}$.
    }
    \label{fig:LLFC_appendix_mlp_ctl}
    \vspace{-10pt}
  \end{center}
\end{figure*}

\begin{figure*}[tb!]
  \begin{center}
    \includegraphics[width=0.87\textwidth]{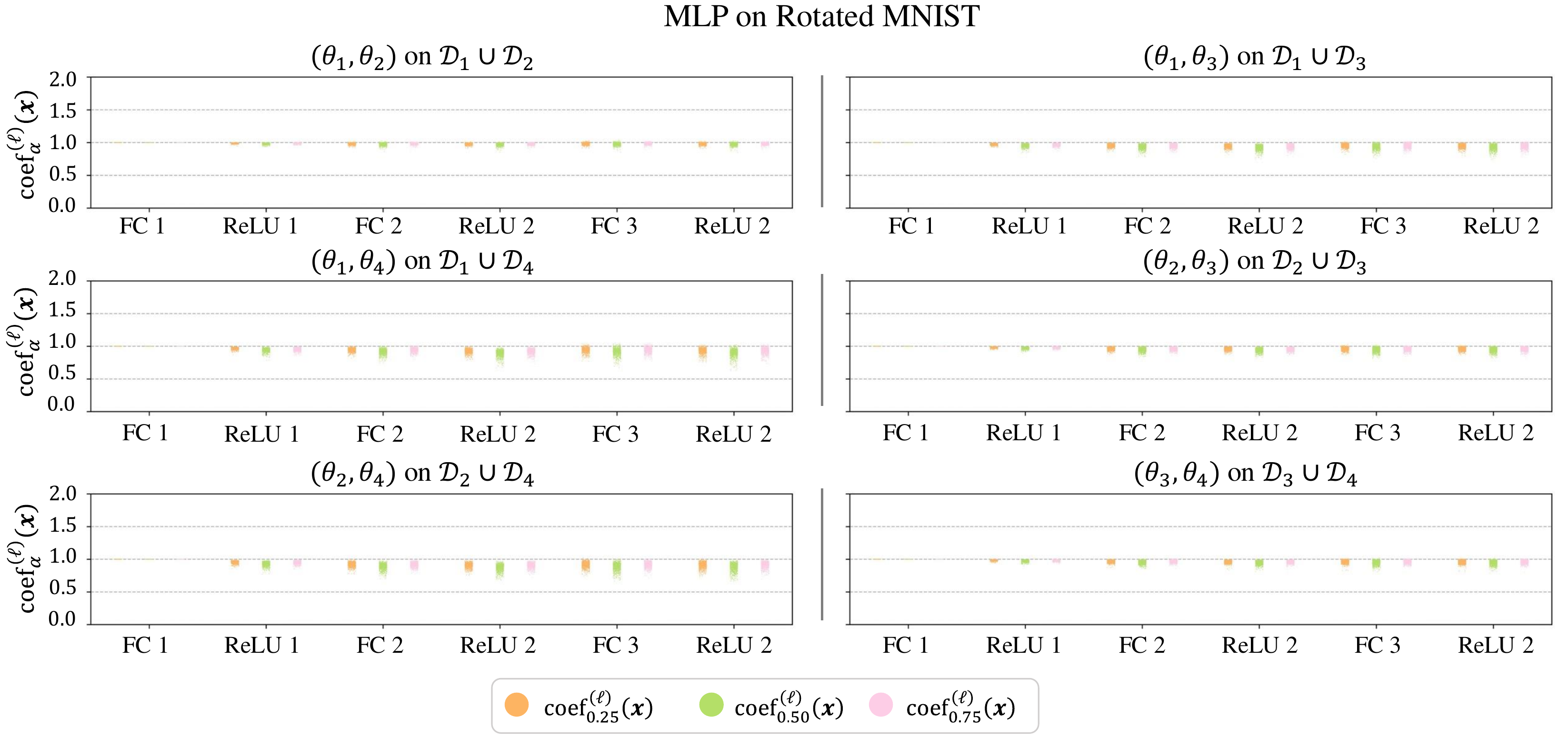}
    \vspace{-5pt}
    \caption{
    Verification of CTL. Distribution of $\text{coef}_{\alpha}^{(\ell)}(\boldsymbol{x})$ across the datasets. 
    Here, $\{\boldsymbol{\theta}_i\}_{i=1}^4$ and $\{\mathcal{D}_i\}_{i=1}^4$ denotes four finetuned MLPs on the corresponding Rotated MNIST with rotation $\in \{ 22.5^{\circ}, 45^{\circ}, 67.5^{\circ}, 90^{\circ} \}$ respectively. 
    The results are reported for all layers except classification head of finetuned MLPs, with $\alpha \in \{0.25, 0.5, 0.75\}$.
    }
    \label{fig:LLFC_appendix_mlp_coef}
    \vspace{-10pt}
  \end{center}
\end{figure*}

\begin{figure*}[tb!]
  \begin{center}
  \includegraphics[width=0.83\textwidth]{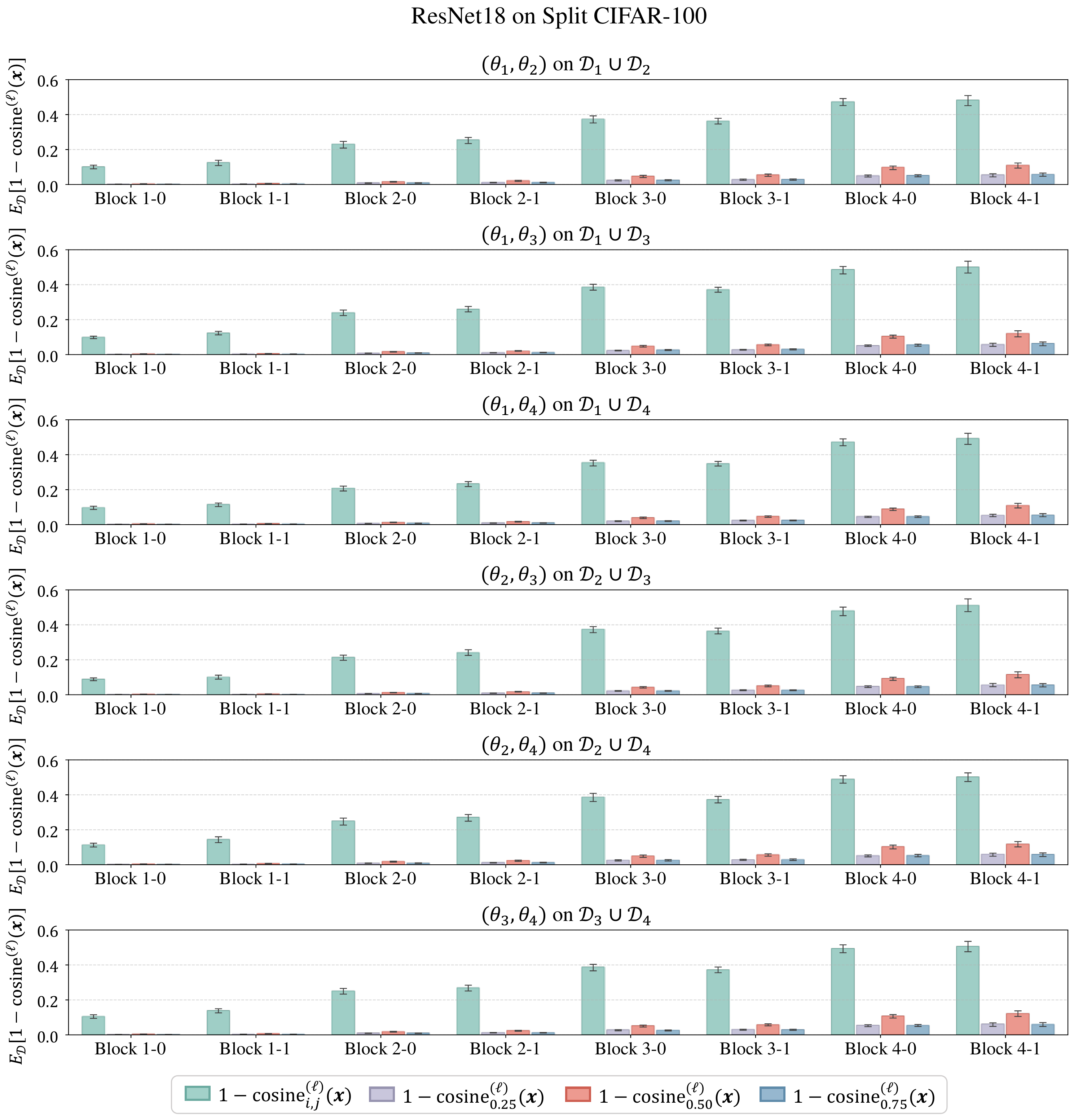}
    \vspace{-5pt}
    \caption{
    Verification of CTL. Compare $\mathbb{E}_{\mathcal{D}}[1-{\rm cosine}_{\alpha}^{(\ell)}(\boldsymbol{x})]$ with $\mathbb{E}_{\mathcal{D}}[1-\text{cosine}_{i,j}^{(\ell)}(\boldsymbol{x})]$.
    Here, $\{\boldsymbol{\theta}_i\}_{i=1}^4$ and $\{\mathcal{D}_i\}_{i=1}^4$ denotes four finetuned ResNet-18s and the second to fifth splits in Split CIFAR-100 respectively. 
    The results are reported for all blocks of finetuned ResNet-18 models, with $\alpha \in \{0.25, 0.5, 0.75\}$.
    }
    \label{fig:LLFC_appendix_resnet18_ctl}
    \vspace{-10pt}
  \end{center}
\end{figure*}

\begin{figure*}[tb!]
  \begin{center}
    \includegraphics[width=0.69\textwidth]{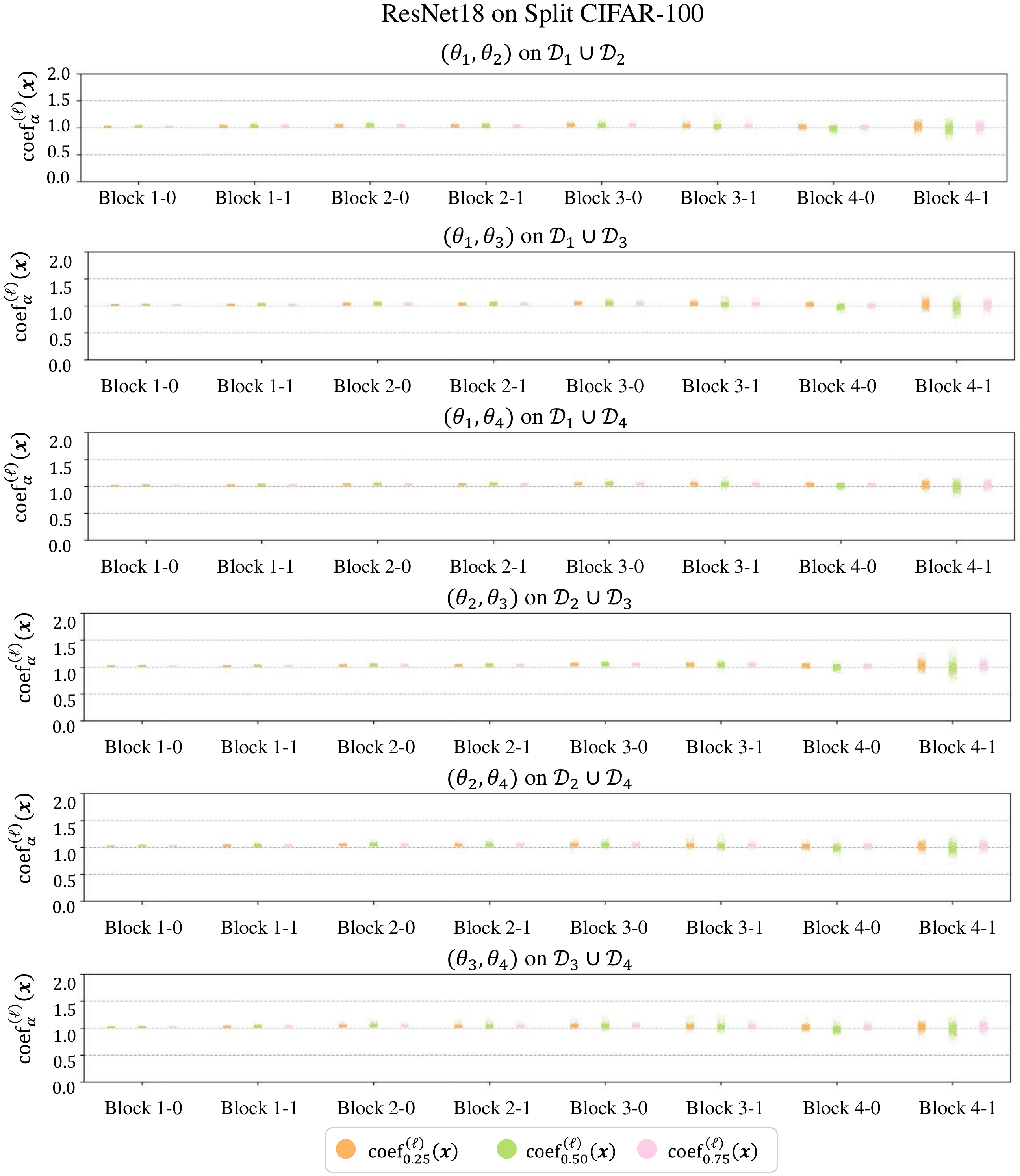}
    \vspace{-5pt}
    \caption{
    Verification of CTL. Distribution of $\text{coef}_{\alpha}^{(\ell)}(\boldsymbol{x})$ across the datasets. 
    Here, $\{\boldsymbol{\theta}_i\}_{i=1}^4$ and $\{\mathcal{D}_i\}_{i=1}^4$ denotes four finetuned ResNet-18s and the second to fifth splits in Split CIFAR-100 respectively. 
    The results are reported for all blocks of finetuned ResNet-18 models, with $\alpha \in \{0.25, 0.5, 0.75\}$.
    }
    \label{fig:LLFC_appendix_resnet18_coef}
    \vspace{-10pt}
  \end{center}
\end{figure*}
\begin{figure*}[tb!]
  \begin{center}
  \includegraphics[width=0.8861\textwidth]{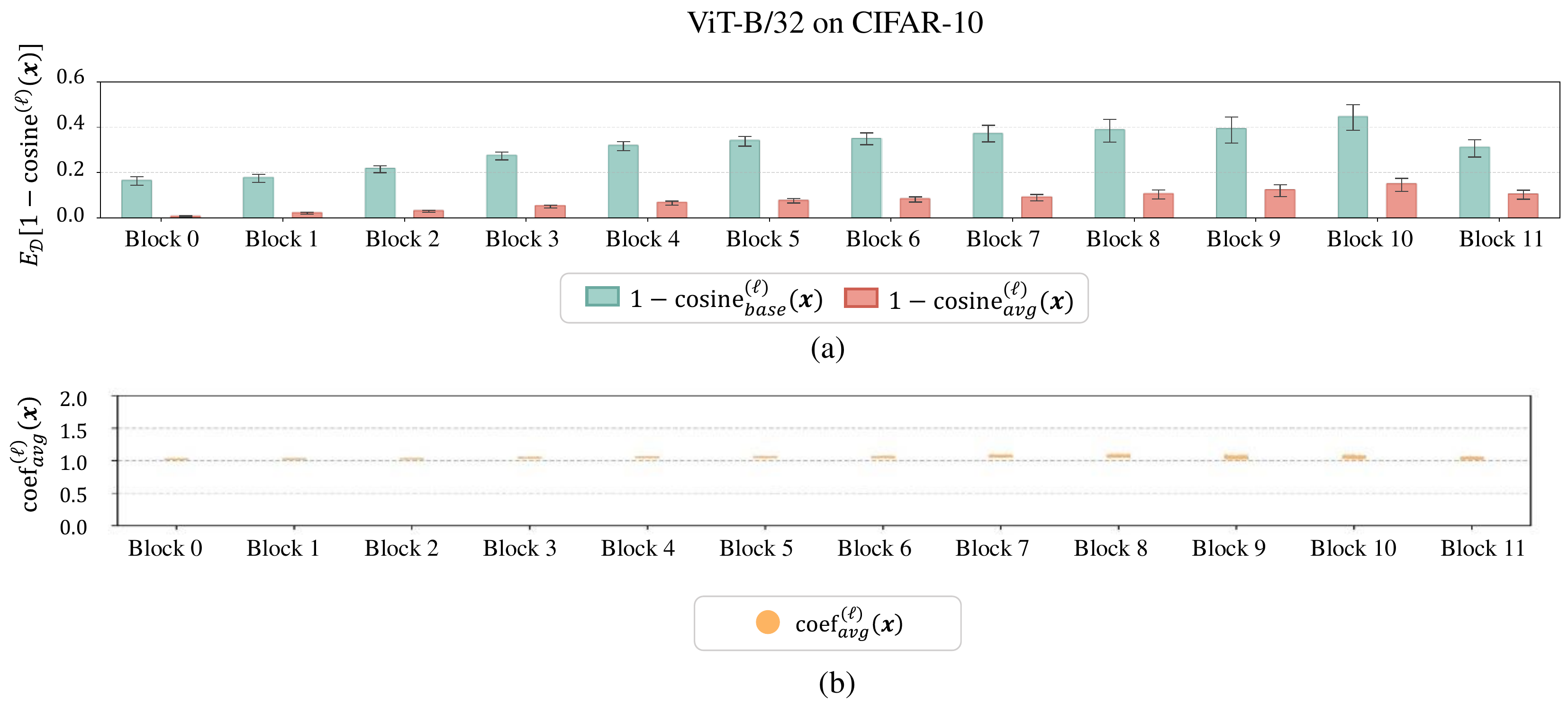}
    \vspace{-5pt}
    \caption{
    Verification of CTL in model averaging. (a) Compare $\mathbb{E}_{\mathcal{D}}[1-{\rm cosine}_{avg}^{(\ell)}(\boldsymbol{x})]$ with  $\mathbb{E}_{\mathcal{D}}[1-{\rm cosine}_{base}^{(\ell)}(\boldsymbol{x})]$. (b) Distribution of ${\rm coef}_{avg}^{(\ell)}(\boldsymbol{x})$ on CIFAR-10. The results are reported for all blocks of ViT-B/32 models finetuned on CIFAR-10 with different hyper-parameters.
    }
    \label{fig:LLFC_appendix_model_avg_cifar10}
    \vspace{-10pt}
  \end{center}
\end{figure*}

\begin{figure*}[tb!]
  \begin{center}
  \includegraphics[width=0.8861\textwidth]{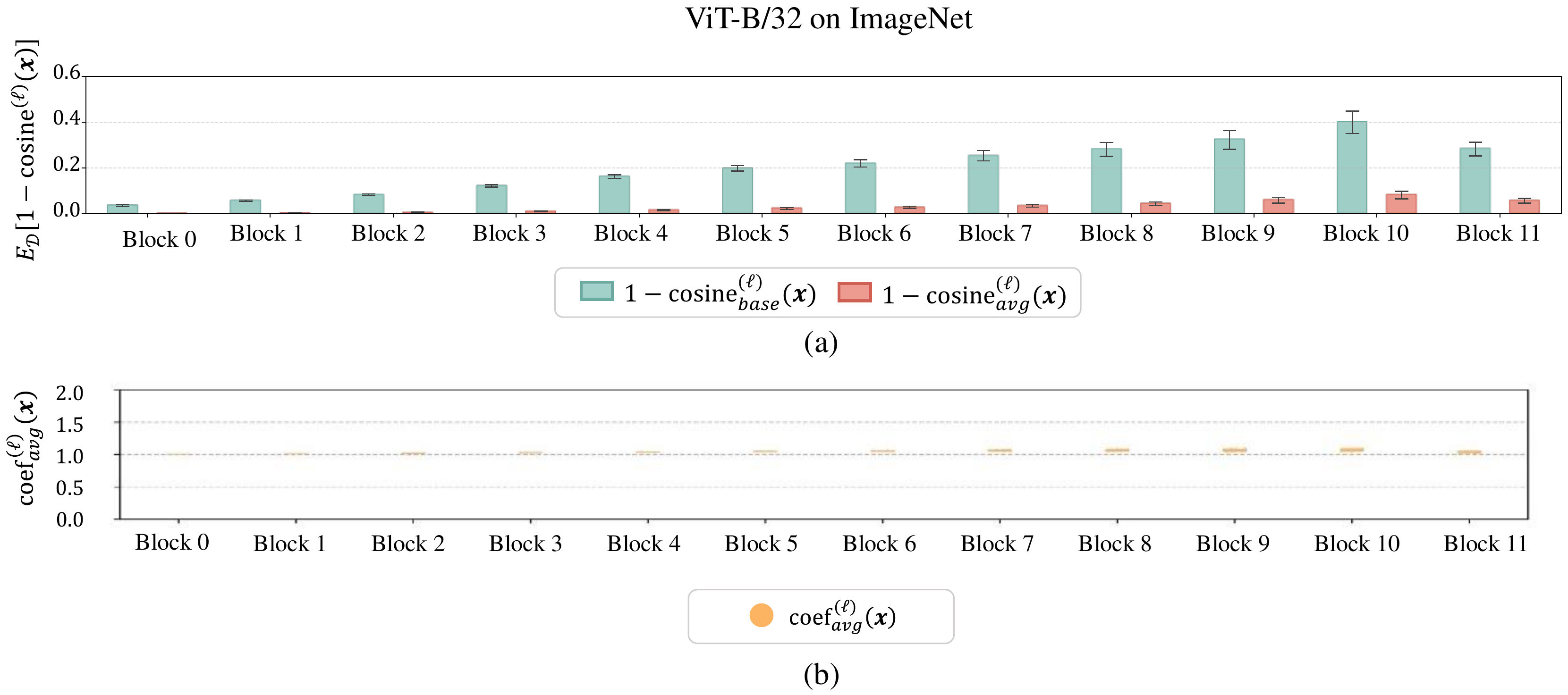}
    \vspace{-5pt}
    \caption{
    Verification of CTL in model averaging. (a) Compare $\mathbb{E}_{\mathcal{D}}[1-{\rm cosine}_{avg}^{(\ell)}(\boldsymbol{x})]$ with  $\mathbb{E}_{\mathcal{D}}[1-{\rm cosine}_{base}^{(\ell)}(\boldsymbol{x})]$. (b) Distribution of ${\rm coef}_{avg}^{(\ell)}(\boldsymbol{x})$ on ImageNet. The results are reported for all blocks of ViT-B/32 models finetuned on ImageNet with different hyper-parameters.
    }
    \label{fig:LLFC_appendix_model_avg_imagenet}
    \vspace{-5pt}
  \end{center}
\end{figure*}

\begin{figure*}[tb!]
  \begin{center}
  \includegraphics[width=0.7416\textwidth]{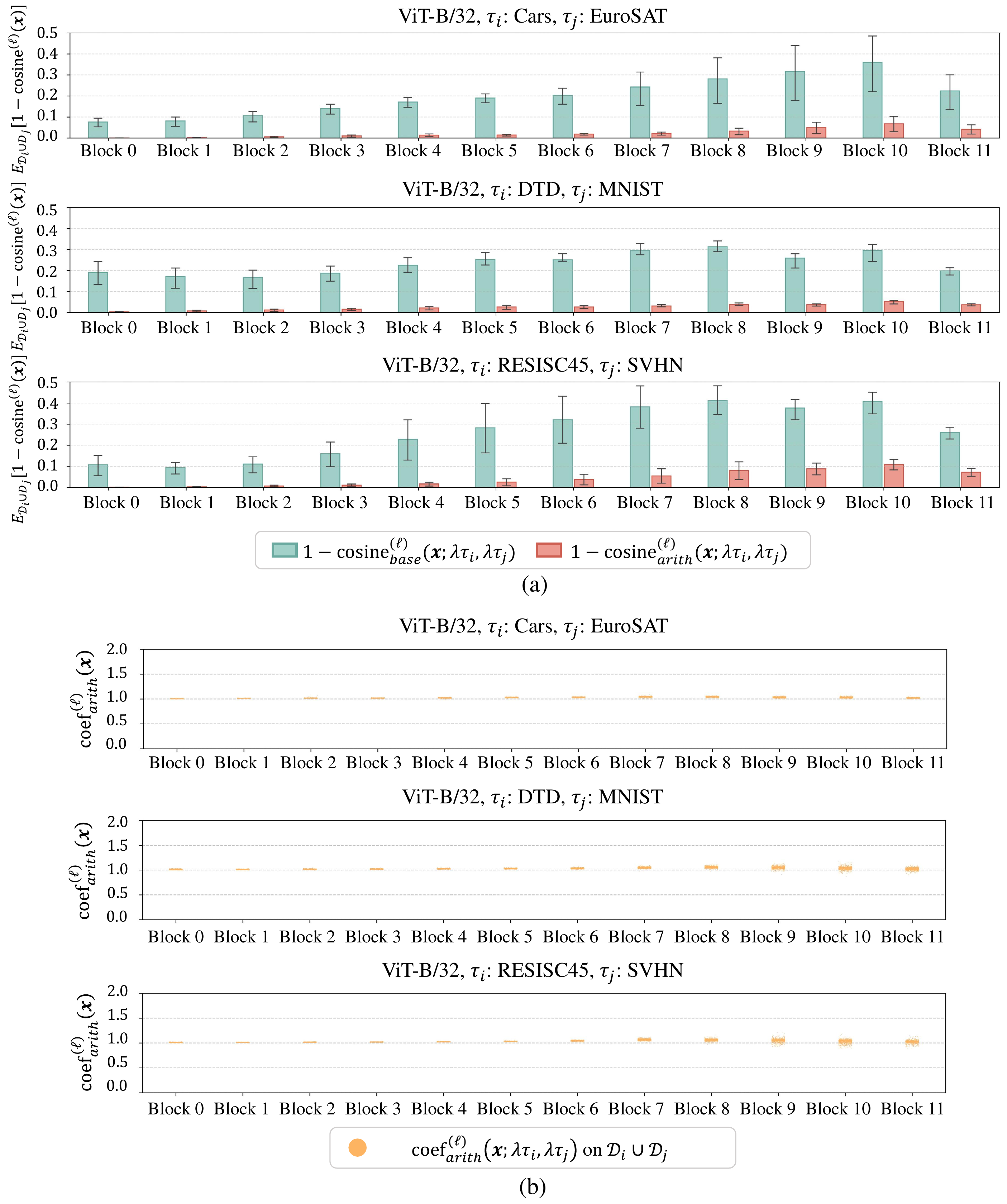}
    \vspace{-5pt}
    \caption{
    Verification of \textbf{Learning via Addition} in task arithmetic. (a) Compare $\mathbb{E}_{\mathcal{D}}[1-{\rm cosine}_{arith}^{(\ell)}(\boldsymbol{x}; \lambda \tau_i, \lambda \tau_j)]$ with $\mathbb{E}_{\mathcal{D}}[1-{\rm cosine}_{base}^{(\ell)}(\boldsymbol{x}; \lambda \tau_i, \lambda \tau_j)]$. The bottom and top of the error bar represent the lower and upper quartile of the values across the dataset, respectively.
    (b) Distribution of ${\rm coef}_{arith}^{(\ell)}(\boldsymbol{x}; \lambda \tau_i, \lambda \tau_j)$. 
    The results are reported for all blocks of finetuned ViT-B/32 under different settings, with $\lambda = 0.4$ and $\alpha=0.5$.
    }
    \label{fig:LLFC_appendix_linearity_add_vit_b_32}
    \vspace{-10pt}
  \end{center}
\end{figure*}

\begin{figure*}[tb!]
  \begin{center}
  \includegraphics[width=0.7416\textwidth]{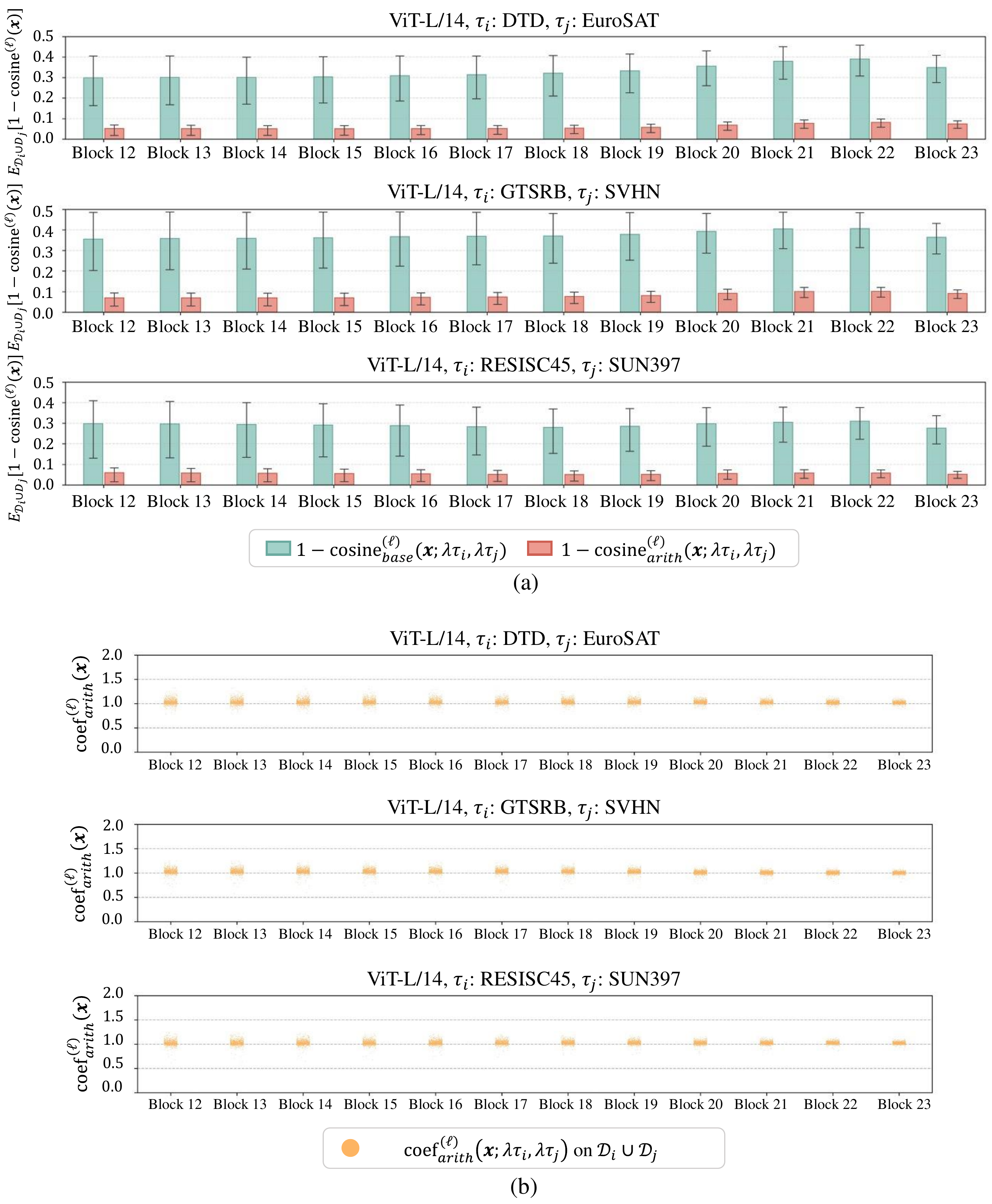}
    \vspace{-5pt}
    \caption{
    Verification of \textbf{Learning via Addition} in task arithmetic. (a) Compare $\mathbb{E}_{\mathcal{D}}[1-{\rm cosine}_{arith}^{(\ell)}(\boldsymbol{x}; \lambda \tau_i, \lambda \tau_j)]$ with $\mathbb{E}_{\mathcal{D}}[1-{\rm cosine}_{base}^{(\ell)}(\boldsymbol{x}; \lambda \tau_i, \lambda \tau_j)]$. The bottom and top of the error bar represent the lower and upper quartile of the values across the dataset, respectively.
    (b) Distribution of ${\rm coef}_{arith}^{(\ell)}(\boldsymbol{x}; \lambda \tau_i, \lambda \tau_j)$. 
    The results are reported for the last $12$ blocks of finetuned ViT-L/14 under different settings, with $\lambda = 0.4$ and $\alpha=0.5$.
    }
    \label{fig:LLFC_appendix_linearity_add_vit_l_14}
    \vspace{-10pt}
  \end{center}
\end{figure*}

\begin{figure*}[tb!]
  \begin{center}
  \includegraphics[width=0.7416\textwidth]{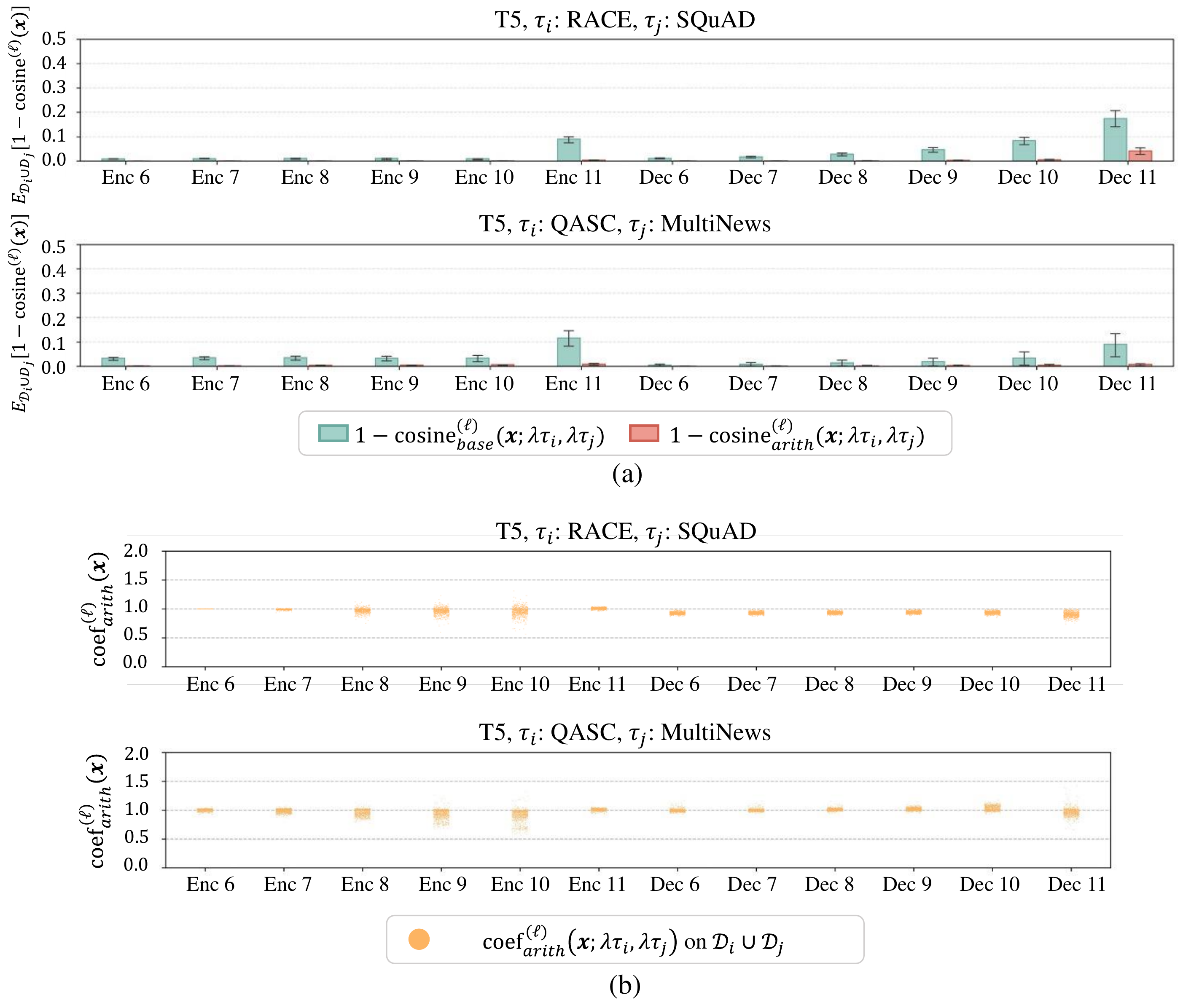}
    \vspace{-5pt}
    \caption{
    Verification of \textbf{Learning via Addition} in task arithmetic. (a) Compare $\mathbb{E}_{\mathcal{D}}[1-{\rm cosine}_{arith}^{(\ell)}(\boldsymbol{x}; \lambda \tau_i, \lambda \tau_j)]$ with $\mathbb{E}_{\mathcal{D}}[1-{\rm cosine}_{base}^{(\ell)}(\boldsymbol{x}; \lambda \tau_i, \lambda \tau_j)]$. The bottom and top of the error bar represent the lower and upper quartile of the values across the dataset, respectively.
    (b) Distribution of ${\rm coef}_{arith}^{(\ell)}(\boldsymbol{x}; \lambda \tau_i, \lambda \tau_j)$. 
    The results are reported for the last $6$ encoder blocks and the last $6$ decoder blocks of finetuned T5 under different settings, with $\lambda = 0.4$ and $\alpha=0.5$.
    }
    \label{fig:LLFC_appendix_linearity_add_t5}
    \vspace{-10pt}
  \end{center}
\end{figure*}

\begin{figure*}[tb!]
  \begin{center}
  \includegraphics[width=0.7416\textwidth]{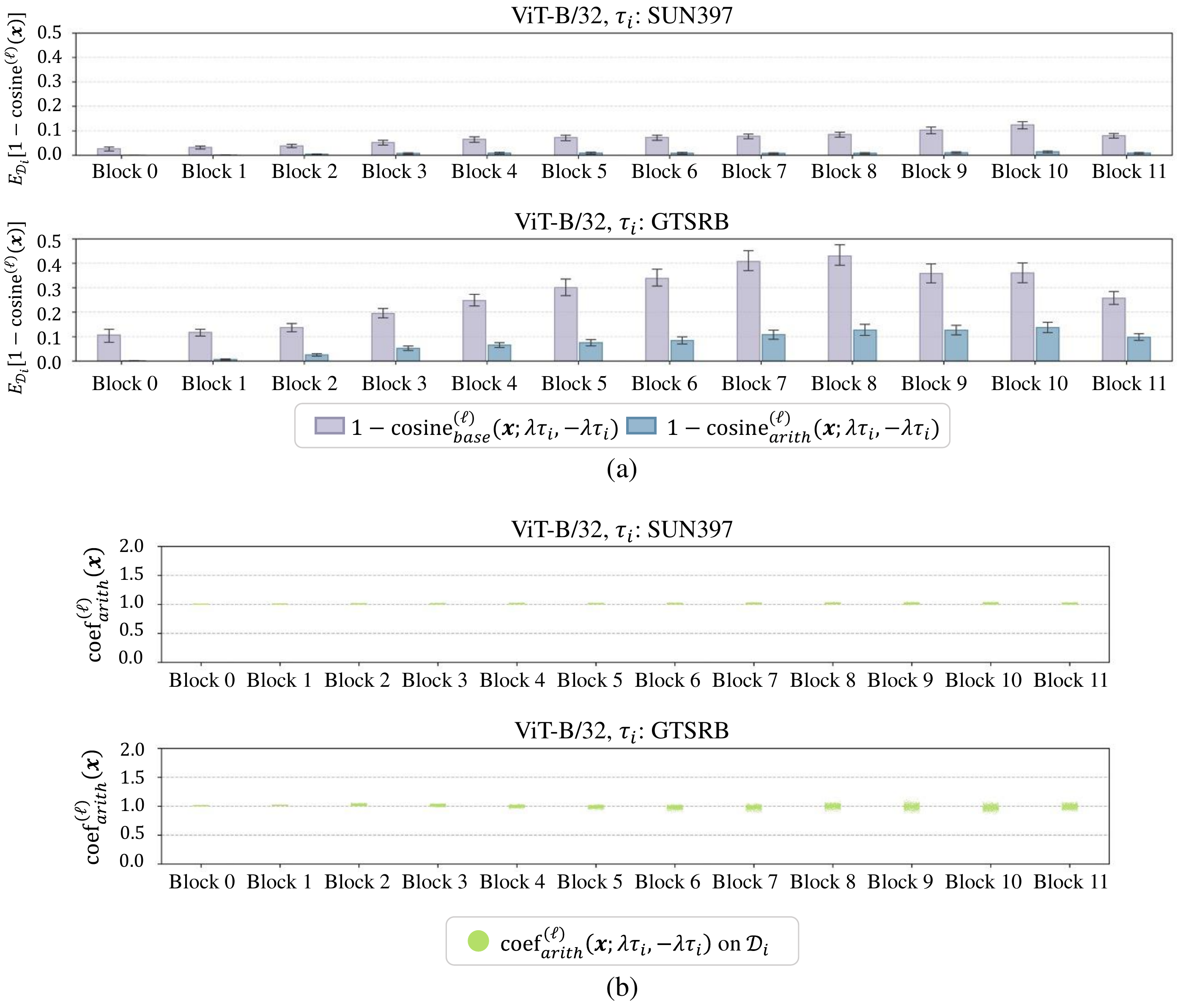}
    \vspace{-5pt}
    \caption{
    Verification of \textbf{Learning via Negation} in task arithmetic. (a) Compare $\mathbb{E}_{\mathcal{D}}[1-{\rm cosine}_{arith}^{(\ell)}(\boldsymbol{x}; \lambda \tau_i, -\lambda \tau_i)]$ with $\mathbb{E}_{\mathcal{D}}[1-{\rm cosine}_{base}^{(\ell)}(\boldsymbol{x}; \lambda \tau_i, -\lambda \tau_i)]$. The bottom and top of the error bar represent the lower and upper quartile of the values across the dataset, respectively.
    (b) Distribution of ${\rm coef}_{arith}^{(\ell)}(\boldsymbol{x}; \lambda \tau_i, -\lambda \tau_i)$. 
    The results are reported for all blocks of finetuned ViT-B/32 under different settings, with $\lambda = 0.4$ and $\alpha=0.5$.
    }
    \label{fig:LLFC_appendix_linearity_neg_vit_b_32}
    \vspace{-10pt}
  \end{center}
\end{figure*}

\begin{figure*}[tb!]
  \begin{center}
  \includegraphics[width=0.7416\textwidth]{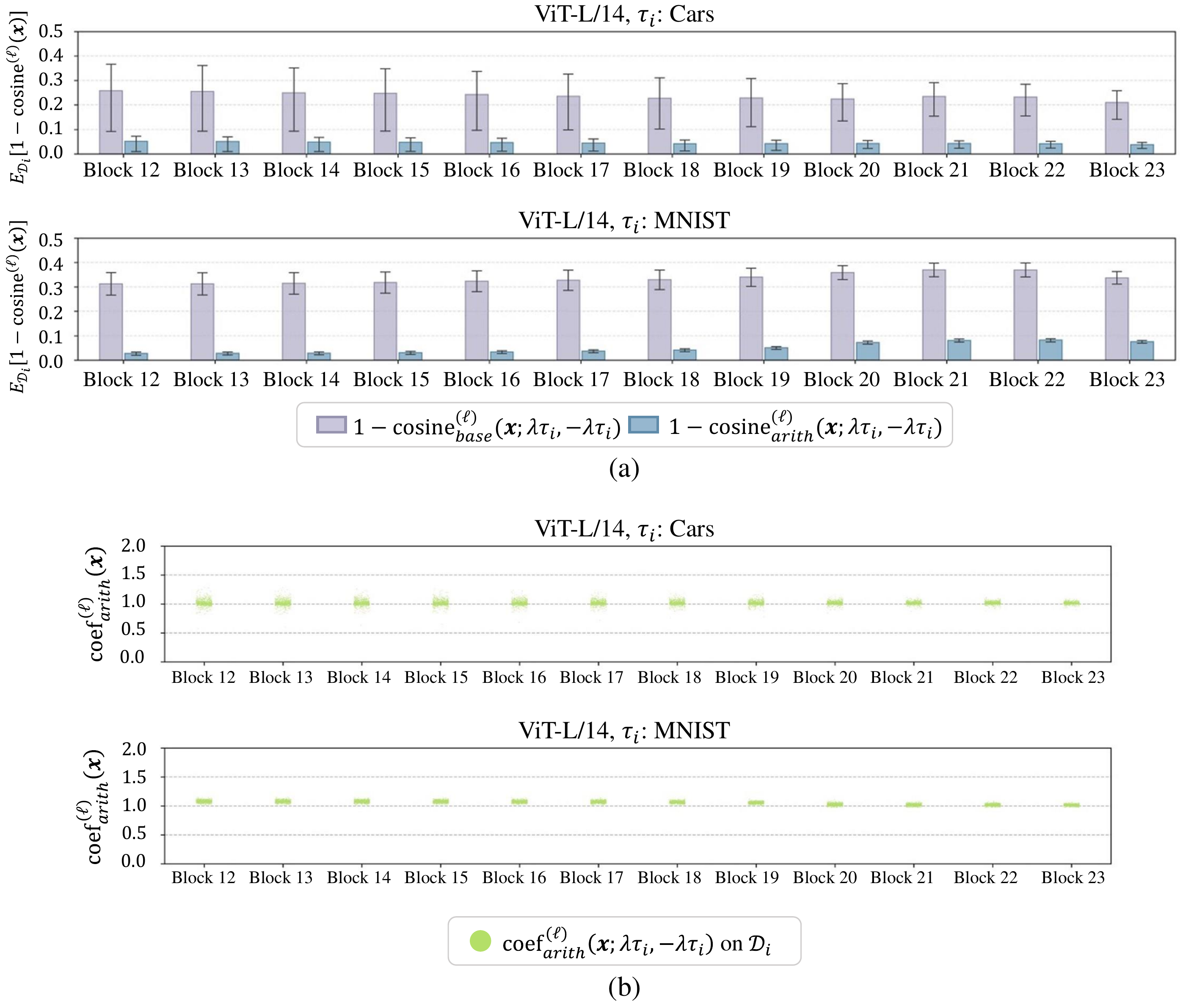}
    \vspace{-5pt}
    \caption{
    Verification of \textbf{Learning via Negation} in task arithmetic. (a) Compare $\mathbb{E}_{\mathcal{D}}[1-{\rm cosine}_{arith}^{(\ell)}(\boldsymbol{x}; \lambda \tau_i, -\lambda \tau_i)]$ with $\mathbb{E}_{\mathcal{D}}[1-{\rm cosine}_{base}^{(\ell)}(\boldsymbol{x}; \lambda \tau_i, -\lambda \tau_i)]$. The bottom and top of the error bar represent the lower and upper quartile of the values across the dataset, respectively.
    (b) Distribution of ${\rm coef}_{arith}^{(\ell)}(\boldsymbol{x}; \lambda \tau_i, -\lambda \tau_i)$. 
    The results are reported for the last $12$ blocks of finetuned ViT-L/14 under different settings, with $\lambda = 0.4$ and $\alpha=0.5$.
    }
    \label{fig:LLFC_appendix_linearity_neg_vit_l_14}
    \vspace{-10pt}
  \end{center}
\end{figure*}

\begin{figure*}[tb!]
  \begin{center}
  \includegraphics[width=0.7416\textwidth]{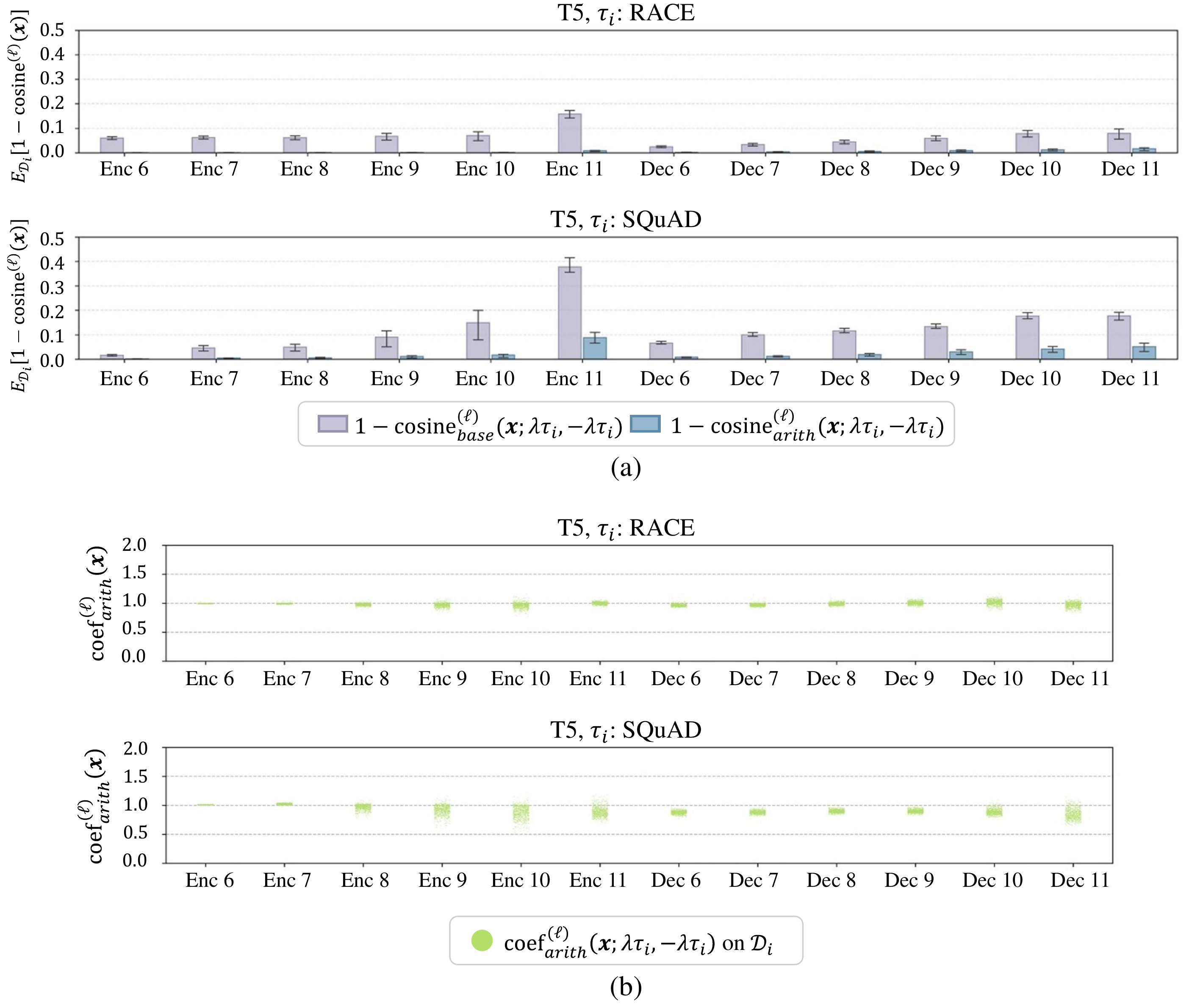}
    \vspace{-5pt}
    \caption{
    Verification of \textbf{Learning via Negation} in task arithmetic. (a) Compare $\mathbb{E}_{\mathcal{D}}[1-{\rm cosine}_{arith}^{(\ell)}(\boldsymbol{x}; \lambda \tau_i, -\lambda \tau_i)]$ with $\mathbb{E}_{\mathcal{D}}[1-{\rm cosine}_{base}^{(\ell)}(\boldsymbol{x}; \lambda \tau_i, -\lambda \tau_i)]$. The bottom and top of the error bar represent the lower and upper quartile of the values across the dataset, respectively.
    (b) Distribution of ${\rm coef}_{arith}^{(\ell)}(\boldsymbol{x}; \lambda \tau_i, -\lambda \tau_i)$. 
    The results are reported for the last $6$ encoder blocks and the last $6$ decoder blocks of finetuned T5 under different settings, with $\lambda = 0.4$ and $\alpha=0.5$.
    }
    \label{fig:LLFC_appendix_linearity_neg_t5}
    \vspace{-10pt}
  \end{center}
\end{figure*}


\end{document}